\newif\ifappendixshow  %
\appendixshowtrue

\newif\ifdraft     %
\draftfalse

\newif\ifexternal  %
\externaltrue

\newif\ifpaper     %
\paperfalse

\documentclass[preprint,12pt]{elsarticle}

\journal{AIJ}

\newcommand{\frontpagedeclarations}{%
\title{Shapley Value Computation in Ontology-Mediated Query Answering}

\author[labri]{Meghyn Bienvenu} 
\ead{meghyn.bienvenu@cnrs.fr}

\author[labri]{Diego Figueira} 
\ead{diego.figueira@cnrs.fr}

\author[labri]{Pierre Lafourcade} 
\ead{pierre.lafourcade@u-bordeaux.fr}

\affiliation[labri]{organization={Univ. Bordeaux, CNRS,  Bordeaux INP, LaBRI, UMR 5800, F-33400},%
            city={Talence},
            country={France}}}

\PassOptionsToPackage{fleqn}{amsmath}  %
  \usepackage{amsmath}

\ifappendixshow%
\usepackage[bibliography=common,appendix=append]{apxproof} %
\else
\usepackage[bibliography=common,appendix=strip]{apxproof}
\fi

\NewCommandCopy{\proofqedsymbol}{\qedsymbol}%
\newcommand{\remarkqedsymbol}{{$\triangle$}}%
\AtBeginEnvironment{proof}{\renewcommand{\qedsymbol}{\proofqedsymbol}}
\AtBeginEnvironment{nestedproof}{\renewcommand{\qedsymbol}{\nestedproofqedsymbol}}
\AtBeginEnvironment{example}{%
  \pushQED{\qed}\renewcommand{\qedsymbol}{\exampleqedsymbol}%
}
\AtEndEnvironment{example}{\popQED\endexample}
\AtBeginEnvironment{remark}{%
  \pushQED{\qed}\renewcommand{\qedsymbol}{\remarkqedsymbol}%
}
\AtEndEnvironment{remark}{\popQED\endexample}

\newif\ifappendix  %
\appendixfalse

\usepackage{algpseudocodex} %
\usepackage{scrhack} %
\usepackage{listings} %
\usepackage{tabularx} %
\usepackage{subfig}

\usepackage{graphicx} %

\usepackage{tikz}
\usepackage{tikzit} %
\usetikzlibrary{decorations.pathreplacing}
\usetikzlibrary{decorations.pathmorphing}
\usetikzlibrary{positioning}
\usetikzlibrary{patterns.meta}
\usetikzlibrary {arrows.meta}

\PassOptionsToPackage{british}{babel} %
    \usepackage{babel}
\usepackage{bbding} %
\usepackage{csquotes} %
\usepackage[inline,shortlabels]{enumitem} %
\usepackage{lineno} %
\usepackage{xspace} %

\usepackage{hyperref}
\usepackage{cleveref} %
   \crefformat{part}{Part #2\MakeUppercase{#1}#3}
\usepackage{crossreftools} %

\usepackage{amssymb} %
\usepackage{amsthm} %
\usepackage{mathcommand} %
\usepackage{mathtools} %
\usepackage{nicefrac} %
\usepackage{stmaryrd} %

\ifdraft %
\usepackage[backgroundcolor=orange!20, textcolor={Dark Ruby Red}, textsize=tiny]{todonotes}
\else
\usepackage[backgroundcolor=orange!20, textcolor={Dark Ruby Red}, textsize=tiny,disable]{todonotes}
\fi
\definecolor{green}{RGB}{0,120,0}
\definecolor{hlyellow}{RGB}{250, 250, 190}
\definecolor{diegoeditcolor}{RGB}{210,210,255}
\definecolor{meghyneditcolor}{RGB}{210,255,210}
\definecolor{pierreeditcolor}{RGB}{255, 225, 186}
\newcommand{\sidediego}[1]{}
\newcommand{\sidemeghyn}[1]{}
\newcommand{\sidepierre}[1]{}
\newcommand{\meghyn}[1]{}
\newcommand{\pierre}[1]{}

\newcommand{\diego}[1]{}

\newcommand{\aka}{a.k.a.}

\newcommand{\eg}{e.g.}

\newcommand{\ie}{i.e.,}

\newcommand{\st}{s.t.}
\newcommand{\wrt}{w.r.t.}

   \newcommand{\ldb}{\llbracket} 
   \newcommand{\rdb}{\rrbracket}

\DeclarePairedDelimiter\set{\{}{\}} %

\ifdraft %

\else

\fi

\LoopCommands\lettersUppercase[l#1] %
   {\newmathcommand#2{\mathbb{#1}}}

\LoopCommands{ABCDEFGIJKMNQRTUVWXYZ}[#1] %
   {\newmathcommand#2{\mathcal{#1}}}
\LoopCommands{HLOPS}[#1] %
   {\renewmathcommand#2{\mathcal{#1}}}

\newcommand{\eps}{\varepsilon} %

\DeclareMathOperator*{\bigO}{O}

\newcommand{\la}{\leftarrow}
\newcommand{\La}{\Leftarrow}
\newcommand{\Ra}{\Rightarrow}
\newcommand{\LRa}{\Leftrightarrow}
\newcommand{\defeq}{\vcentcolon=}
\newcommand{\eqdef}{=\vcentcolon}
\renewcommand{\le}{\leqslant}
\renewcommand{\ge}{\geqslant}

\newcommand{\inc}{\subseteq} %

\newcommand{\ic}{\sqsubseteq} %

\definecolor{light-gray}{gray}{0.9} %
\newcommand{\proofcase}[1]{\noindent\colorbox{light-gray}{#1}~~}

\newenvironment{equation-inline}{ %
\refstepcounter{equation}}{\hfill(\theequation)\\}

\theoremstyle{plain}

\newtheoremrep{theorem}{Theorem}

\newtheoremrep{claim}[theorem]{Claim}
   \AddToHook{env/claim/begin}{\crefalias{theorem}{claim}} 
\newtheoremrep{corollary}[theorem]{Corollary}
   \AddToHook{env/corollary/begin}{\crefalias{theorem}{corollary}} 
\newtheoremrep{example}[theorem]{Example}
   \AddToHook{env/example/begin}{\crefalias{theorem}{example}} 
\newtheoremrep{lemma}[theorem]{Lemma}
   \AddToHook{env/lemma/begin}{\crefalias{theorem}{lemma}} 
\newtheoremrep{proposition}[theorem]{Proposition}
   \AddToHook{env/proposition/begin}{\crefalias{theorem}{proposition}} 

\theoremstyle{remark}

   \AddToHook{env/open/begin}{\crefalias{theorem}{open}} 
\newtheoremrep{remark}[theorem]{Remark}
   \AddToHook{env/remark/begin}{\crefalias{theorem}{remark}} 

\crefname{theorem}{Theorem}{Theorems}
\crefname{claim}{Claim}{Claims}
\crefname{corollary}{Corollary}{Corollaries}
\crefname{example}{Example}{Examples}
\crefname{lemma}{Lemma}{Lemmas}
\crefname{proposition}{Proposition}{Propositions}
\crefname{remark}{Remark}{Remarks}

\tikzset{
	subtree/.pic={
		\coordinate (-a) at (0,0);
		\coordinate (-b) at (-1,-2);
		\coordinate (-c) at (1,-2);
		\coordinate (-d) at (-0.35,-0.35);
		\coordinate (-e) at (0.35,-0.35);
		\coordinate (-h) at (-1.26,-2.16);
		\coordinate (-i) at (1.26,-2.16);
		\coordinate (-j) at (0.35,0.35);
		\coordinate (-k) at (-0.35,0.35);
		\coordinate (-west) at (-0.4,0);
		\coordinate (-east) at (0.4,0);
		
		\draw (-a) -- (-b) -- (-c) -- cycle;
		\draw[fill=white] (-a) circle (0.4);
	}
}

\tikzset{
	subtree_small/.pic={
		\coordinate (-a) at (0,0);
		\coordinate (-b) at (-.75,-1);
		\coordinate (-c) at (.75,-1);
		\coordinate (-d) at (-0.35,-0.35);
		\coordinate (-e) at (0.35,-0.35);
		\coordinate (-h) at (-0.92,-1.06);
		\coordinate (-i) at (0.92,-1.06);
		\coordinate (-j) at (0.35,0.35);
		\coordinate (-k) at (-0.35,0.35);
		\coordinate (-west) at (-0.4,0);
		\coordinate (-east) at (0.4,0);
		
		\draw (-a) -- (-b) -- (-c) -- cycle;
		\draw[fill=white] (-a) circle (0.4);
	}
}

\tikzset{
	subtrinvisible/.pic={
		\coordinate (-a) at (0,0);
		\coordinate (-b) at (-1,-2);
		\coordinate (-c) at (1,-2);
		\coordinate (-d) at (-0.35,-0.35);
		\coordinate (-e) at (0.35,-0.35);
		\coordinate (-h) at (-1.26,-2.16);
		\coordinate (-i) at (1.26,-2.16);
		\coordinate (-j) at (0.35,0.35);
		\coordinate (-k) at (-0.35,0.35);
		\coordinate (-west) at (-0.4,0);
		\coordinate (-east) at (0.4,0);
	}
}

\tikzset{
	graphbox/.pic={
		\node [draw, circle, minimum height=17pt] (-0) at (-1, 0.75) {};
		\node [draw, circle, minimum height=17pt] (-1) at (-1, -0.25) {};
		\node [draw, circle, minimum height=17pt] (-2) at (-1, -1.25) {};
		\node [draw, circle, minimum height=17pt] (-3) at (1, -1.25) {};
		\node [draw, circle, minimum height=17pt] (-4) at (1, -0.25) {};
		\node [draw, circle, minimum height=17pt] (-5) at (1, 0.75) {};
		\node [] (-6) at (0, 1.25) {\Huge $G$};
		\coordinate (-7) at (-1.5, -1.75) {};
		\coordinate (-8) at (1.5, 1.75) {};
		\coordinate (-north) at (0, 1.75);
		\coordinate (-west) at (-1.5, 0);
		\coordinate (-east) at (1.5, 0);
		\coordinate (-south) at (0, -1.75);
		
		\draw (-8) rectangle (-7);
		\draw (-0) to (-5);
		\draw (-1) to (-3);
		\draw (-2) to (-4);
	}
}

\tikzset{
	graphbox-ns/.pic={
		\coordinate (-7) at (-1.5, -1.75) {};
		\coordinate (-8) at (1.5, 1.75) {};
		\draw[fill=white] (-8) rectangle (-7);
		
		\node [draw, circle, minimum height=17pt] (-0) at (-1, 0.75) {};
		\node [draw, circle, minimum height=17pt] (-1) at (-1, -0.25) {};
		\node [draw, circle, minimum height=17pt] (-2) at (-1, -1.25) {};
		\node [draw, circle, minimum height=17pt] (-3) at (1, -1.25) {};
		\node [draw, circle, minimum height=17pt] (-4) at (1, -0.25) {};
		\node [draw, circle, minimum height=17pt] (-5) at (1, 0.75) {};
		\node [] (-6) at (0, 1.25) {\Huge $G$};
		\coordinate (-north) at (0, 1.75);
		\coordinate (-west) at (-1.5, 0);
		\coordinate (-east) at (1.5, 0);
		\coordinate (-south) at (0, -1.75);
		
		\draw[very thick,double distance=3pt] (-north) -- ++(0,0.7);
		\draw[very thick,double distance=3pt,arrows = {-Implies[]}] (-south) -- ++(0,-1.6);
		\draw (-8) rectangle (-7);
		\draw (-0) to (-5);
		\draw (-1) to (-3);
		\draw (-2) to (-4);
	}
}

\tikzset{
	graphbox-ew/.pic={
		\node [draw, circle, minimum height=17pt] (-0) at (-1, 0.75) {};
		\node [draw, circle, minimum height=17pt] (-1) at (-1, -0.25) {};
		\node [draw, circle, minimum height=17pt] (-2) at (-1, -1.25) {};
		\node [draw, circle, minimum height=17pt] (-3) at (1, -1.25) {};
		\node [draw, circle, minimum height=17pt] (-4) at (1, -0.25) {};
		\node [draw, circle, minimum height=17pt] (-5) at (1, 0.75) {};
		\node [] (-6) at (0, 1.25) {\Huge $G$};
		\coordinate (-7) at (-1.5, -1.75) {};
		\coordinate (-8) at (1.5, 1.75) {};
		\coordinate (-north) at (0, 1.75);
		\coordinate (-west) at (-1.5, 0);
		\coordinate (-east) at (1.5, 0);
		\coordinate (-south) at (0, -1.75);
		
		\draw[thick,double distance=3pt] (-west) -- ++(-1,0);
		\draw[thick,double distance=3pt,arrows = {-Implies[]}] (-east) -- ++(1,0);
		\draw (-8) rectangle (-7);
		\draw (-0) to (-5);
		\draw (-1) to (-3);
		\draw (-2) to (-4);
	}
}
\ifdraft
\usepackage[xcolor, hyperref, cleveref, notion, quotation, composition]{knowledge}
\else\ifpaper
\usepackage[xcolor, hyperref, cleveref, notion, quotation, paper]{knowledge}
\else
\usepackage[xcolor, hyperref, cleveref, notion, quotation, electronic]{knowledge}
\fi\fi

\usepackage{mathcommand}
\knowledgeconfigure{quotation, protect quotation={tikzcd}}
\knowledgeconfigure{diagnose line=true, diagnose bar=true}

\definecolor{Dark Ruby Red}{HTML}{5d1416}
\definecolor{Dark Blue Sapphire}{HTML}{003c47} %
\definecolor{Dark Gamboge}{HTML}{be7c00}

\IfKnowledgePaperModeTF{
    \knowledgestyle{notion}{color=black}
    \hypersetup{}
}{
    \knowledgestyle{intro notion}{color={Dark Ruby Red}, emphasize}
    \knowledgestyle{notion}{color={Dark Blue Sapphire}}
    \hypersetup{ %
        colorlinks=true,
        breaklinks=true,
        linkcolor={Dark Blue Sapphire}, %
        citecolor={Dark Blue Sapphire}, %
        filecolor={Dark Blue Sapphire}, %
        urlcolor={Dark Blue Sapphire},
    }
    \IfKnowledgeElectronicModeTF{
    }{
        \knowledgeconfigure{anchor point color={Dark Ruby Red}, anchor point shape=corner}
        \knowledgestyle{intro unknown}{color={Dark Gamboge}, emphasize}
        \knowledgestyle{intro unknown cont}{color={Dark Gamboge}, emphasize}
        \knowledgestyle{kl unknown}{color={Dark Gamboge}}
        \knowledgestyle{kl unknown cont}{color={Dark Gamboge}}
    }
}

\makeindex

\newcommand{\APintrorep}[1]{%
	\ifappendix%
	\kl{#1}%
	\else%
	\AP\intro{#1}%
	\fi%
}

   \knowledgenewrobustcmd{\ann}{\cmdkl{\nu}}
\knowledgenewrobustcmd{\oneann}{\cmdkl{\mathbf{1}}}
\knowledgenewrobustcmd{\adom}{\cmdkl{\textit{adom}}} %

\knowledgenewrobustcmd{\Minsups}[1]{\cmdkl{\mathsf{MS}_{#1}}} %

\knowledgenewrobustcmd{\BPP}{\cmdkl{\mathsf{BPP}}}
\knowledgenewrobustcmd{\coNP}{\cmdkl{\mathsf{coNP}}}
\knowledgenewrobustcmd{\FP}{\cmdkl{\mathsf{FP}}} 
\knowledgenewrobustcmd{\FPsNP}{\cmdkl{\mathsf{FP}^{\mathsf{\#NP}}}} 
\knowledgenewrobustcmd{\FPsP}{\cmdkl{\mathsf{FP}^\mathsf{\#P}}}
\knowledgenewrobustcmd{\FPsPH}{\cmdkl{\mathsf{FP}^{\mathsf{\#PH}}}}
\knowledgenewrobustcmd{\NP}{\cmdkl{\mathsf{NP}}}
\knowledgenewrobustcmd{\PH}{\cmdkl{\mathsf{PH}}}
\knowledgenewrobustcmd{\PsP}{\cmdkl{\mathsf{P}^\mathsf{\#P}}}
\knowledgenewrobustcmd{\Ptime}{\cmdkl{\mathsf{P}}}
\knowledgenewrobustcmd{\sNP}{\cmdkl{\mathsf{\#NP}}}
\knowledgenewrobustcmd{\sP}{\cmdkl{\mathsf{\#P}}} 
\knowledgenewrobustcmd{\sPH}{\cmdkl{\mathsf{\#PH}}}

\knowledgenewmathcommand{\numBipVerCov}{\cmdkl{\mathsf{\#BipVerCov}}}
\knowledgenewrobustcmd{\numBipIndep}{\cmdkl{\mathsf{\#BipIndepSet}}} %
\knowledgenewrobustcmd{\numSTConn}{\cmdkl{\mathsf{\#stConnect}}} %

\knowledgenewmathcommandPIE{\countMS}{%
   \cmdkl{\#}^{\cmdkl{\textsf{ms}}}#1#2#3}
\knowledgenewmathcommandPIE{\countFMS}{%
   \cmdkl{\#}^{\cmdkl{\textsf{fms}}}#1#2#3}
\knowledgenewmathcommandPIE{\countAns}{%
   \cmdkl{\#}^{\cmdkl{\textsf{hom}}}#1#2#3}

\knowledgenewmathcommandPIE{\evalCountMS}{%
   \cmdkl{\textsc{eval-}\#}^{\cmdkl{\textsf{ms}}}#1#2#3}
\knowledgenewmathcommandPIE{\evalCountFMS}{%
   \cmdkl{\textsc{eval-}\#}^{\cmdkl{\textsf{fms}}}#1#2#3}
\knowledgenewmathcommandPIE{\evalCountAns}{%
   \cmdkl{\textsc{eval-}\#}^{\cmdkl{\textsf{hom}}}#1#2#3}

\knowledgenewmathcommandPIE{\GIMC}{%
   \cmdkl{\mathsf{GIMC}#1#2#3}}
\knowledgenewmathcommandPIE{\FGIMC}{%
   \cmdkl{\mathsf{FGIMC}#1#2#3}}

\knowledgenewmathcommandPIE{\PQE}{%
   \cmdkl{\mathsf{PQE}#1#2#3}}
\knowledgenewmathcommandPIE{\SPQE}{%
   \cmdkl{\mathsf{SPQE}#1#2#3}}
\knowledgenewmathcommandPIE{\SPPQE}{%
   \cmdkl{\mathsf{SPPQE}#1#2#3}}
\knowledgenewmathcommandPIE{\PQEPhalf}{%
   \cmdkl{\mathsf{PQE}#1#2#3\!\left(\nicefrac 1 2\right)}}
\knowledgenewmathcommandPIE{\PQEPhalfOne}{%
   \cmdkl{\mathsf{PQE}#1#2#3\!\left(\nicefrac 1 2 ; 1\right)}}
\knowledgenewmathcommandPIE{\CPQE}{%
   \cmdkl{\mathsf{CPQE}#1#2#3}}
\knowledgenewmathcommandPIE{\CSPQE}{
   \cmdkl{\mathsf{CSPQE}#1#2#3}}
\knowledgenewmathcommandPIE{\CSPPQE}{
   \cmdkl{\mathsf{CSPPQE}#1#2#3}}
\knowledgenewmathcommandPIE{\CPQEPhalf}{
   \cmdkl{\mathsf{CPQE}#1#2#3\!\left(\nicefrac 1 2\right)}}
\knowledgenewmathcommandPIE{\CPQEPhalfOne}{
   \cmdkl{\mathsf{CPQE}#1#2#3\!\left(\nicefrac 1 2 ; 1\right)}}

\knowledgenewmathcommandPIE{\stShapley}{%
   \cmdkl{\mathsf{SVC}}^{\cmdkl{\star}}#1#2#3}

\knowledgenewmathcommandPIE{\dShapley}{%
   \cmdkl{\mathsf{SVC}}^{\cmdkl{\textsf{dr}}}#1#2#3}
\newmathcommandPIE{\dShapleyNoKL}{%
   \mathsf{SVC}^{\textsf{dr}}#1#2#3}
\knowledgenewmathcommandPIE{\dnShapley}{%
   \cmdkl{\mathsf{nSVC}}^{\cmdkl{\textsf{dr}}}#1#2#3}
\knowledgenewmathcommandPIE{\dcShapley}{%
   \cmdkl{\mathsf{CSVC}}^{\cmdkl{\textsf{dr}}}#1#2#3}
\knowledgenewmathcommandPIE{\dcnShapley}{%
   \cmdkl{\mathsf{nCSVC}}^{\cmdkl{\textsf{dr}}}#1#2#3}

\knowledgenewmathcommandPIE{\mcShapley}{%
   \cmdkl{\mathsf{SVC}}^{\cmdkl{\textsf{MC}}}#1#2#3}
\knowledgenewmathcommandPIE{\pShapley}{%
   \cmdkl{\mathsf{SVC}}^{\cmdkl{\textsf{P}}}#1#2#3}
\knowledgenewmathcommandPIE{\rShapley}{%
   \cmdkl{\mathsf{SVC}}^{\cmdkl{\textsf{R}}}#1#2#3}
\knowledgenewmathcommandPIE{\saShapley}{%
   \cmdkl{\mathsf{SVC}}^{\cmdkl{\textsf{hom}}}#1#2#3}

\newcommand{\minsupindex}{\textsf{ms}}
\knowledgenewmathcommandPIE{\msShapley}{%
   \cmdkl{\mathsf{SVC}}^{\cmdkl{\minsupindex}}#1#2#3}
\knowledgenewmathcommandPIE{\sShapley}{%
   \cmdkl{\mathsf{SVC}}^{\cmdkl{\textsf{s}}}#1#2#3}
\knowledgenewmathcommandPIE{\sharpShapley}{%
   \cmdkl{\mathsf{SVC}}^{\cmdkl{\textsf{\#}}}#1#2#3}
\knowledgenewmathcommandPIE{\wShapley}{%
   \cmdkl{\mathsf{SVC}}#1#2#3}

\newcommand{\subendo}{{\textup{\textsf{n}}}}
\newcommand{\subexo}{{\textup{\textsf{x}}}}
\knowledgenewrobustcmd{\Dn}[1][\D]{#1_{\cmdkl{\subendo}}}
\knowledgenewrobustcmd{\Dx}[1][\D]{#1_{\cmdkl{\subexo}}}

\knowledgenewrobustcmd{\constn}{\cmdkl{\textit{const}_{\cmdkl{\subendo}}}}
\knowledgenewrobustcmd{\constx}{\cmdkl{\textit{const}_{\cmdkl{\subendo}}}}

\knowledgenewmathcommandPIE{\games}{%
   \cmdkl{\mathcal{F}^{\emptyset \mapsto 0}#1#2#3}}
\knowledgenewmathcommandPIE{\sgames}{%
   \cmdkl{\mathcal{SF}^{\emptyset \mapsto 0}#1#2#3}}

\knowledgenewrobustcmd{\Sh}{\cmdkl{\mathrm{Sh}}} %
\knowledgenewrobustcmd{\Bz}{\cmdkl{\mathrm{Bz}}} %

\knowledgenewmathcommandPIE{\scorefun}{%
   \cmdkl{\xi}#1#2#3}
\knowledgenewmathcommandPIE{\STscorefun}{%
   \cmdkl{\Xi}^{\cmdkl{\mathsf{\star}}}#1#2#3}
\knowledgenewmathcommandPIE{\stscorefun}{%
   \cmdkl{\xi}^{\cmdkl{\mathsf{\star}}}#1#2#3}

\knowledgenewmathcommandPIE{\onems}{%
   \cmdkl{\xi}^{\cmdkl{\mathsf{1ms}}}#1#2#3}%
\knowledgenewmathcommandPIE{\SHAPscorefun}{%
   \cmdkl{\Xi}^{\cmdkl{\textsf{SHAP}}}#1#2#3}
\knowledgenewmathcommandPIE{\shapscorefun}{%
   \cmdkl{\xi}^{\cmdkl{\textsf{SHAP}}}#1#2#3}

\knowledgenewmathcommandPIE{\Dscorefun}{%
   \cmdkl{\Xi}^{\cmdkl{\textsf{dr}}}#1#2#3}
\knowledgenewmathcommandPIE{\dscorefun}{%
   \cmdkl{\xi}^{\cmdkl{\textsf{dr}}}#1#2#3}

\knowledgenewmathcommandPIE{\MCscorefun}{%
   \cmdkl{\Xi}^{\cmdkl{\textsf{MC}}}#1#2#3}
\knowledgenewmathcommandPIE{\mcscorefun}{%
   \cmdkl{\xi}^{\cmdkl{\textsf{MC}}}#1#2#3}
\knowledgenewmathcommandPIE{\Pscorefun}{%
   \cmdkl{\Xi}^{\cmdkl{\textsf{P}}}#1#2#3}
\knowledgenewmathcommandPIE{\pscorefun}{%
   \cmdkl{\xi}^{\cmdkl{\textsf{P}}}#1#2#3}
\knowledgenewmathcommandPIE{\Rscorefun}{%
   \cmdkl{\Xi}^{\cmdkl{\textsf{R}}}#1#2#3}
\knowledgenewmathcommandPIE{\rscorefun}{%
   \cmdkl{\xi}^{\cmdkl{\textsf{R}}}#1#2#3}
\knowledgenewmathcommandPIE{\SAscorefun}{%
   \cmdkl{\Xi}^{\cmdkl{\textsf{hom}}}#1#2#3}
\knowledgenewmathcommandPIE{\sascorefun}{%
   \cmdkl{\xi}^{\cmdkl{\textsf{hom}}}#1#2#3}

\knowledgenewmathcommandPIE{\MSscorefun}{%
   \cmdkl{\Xi}^{\cmdkl{\minsupindex}}#1#2#3}
\knowledgenewmathcommandPIE{\msscorefun}{%
   \cmdkl{\xi}^{\cmdkl{\minsupindex}}#1#2#3}
\knowledgenewmathcommandPIE{\Sscorefun}{%
   \cmdkl{\Xi}^{\cmdkl{\textsf{s}}}#1#2#3}
\knowledgenewmathcommandPIE{\sscorefun}{%
   \cmdkl{\xi}^{\cmdkl{\textsf{s}}}#1#2#3}
\knowledgenewmathcommandPIE{\SHARPscorefun}{%
   \cmdkl{\Xi}^{\cmdkl{\textsf{\#}}}#1#2#3}
\knowledgenewmathcommandPIE{\sharpscorefun}{%
   \cmdkl{\xi}^{\cmdkl{\textsf{\#}}}#1#2#3}
\knowledgenewmathcommandPIE{\Wscorefun}{%
   \cmdkl{\Xi}#1#2#3}
\knowledgenewmathcommandPIE{\wscorefun}{%
   \cmdkl{\xi}#1#2#3}

\knowledgenewrobustcmd{\sigmaless}[1]{\cmdkl{\sigma}_{\!\cmdkl{<}#1}}
\knowledgenewrobustcmd{\sigmaleq}[1]{\cmdkl{\sigma}_{\!\cmdkl{\le}#1}}

\knowledgenewrobustcmd{\bse}[1]{\cmdkl{X_{#1}}} %

\knowledgenewrobustcmd{\atoms}{\cmdkl{\textit{atoms}}}
\knowledgenewrobustcmd{\arity}{\cmdkl{\mathrm{arity}}}
\knowledgenewrobustcmd{\const}{\cmdkl{\textit{const}}}
\knowledgenewrobustcmd{\Const}{\cmdkl{\mathsf{Const}}} %
\knowledgenewrobustcmd{\dom}{\cmdkl{\mathrm{dom}}} %
\knowledgenewrobustcmd{\mterms}{\cmdkl{\textit{term}}}
\knowledgenewrobustcmd{\vars}{\cmdkl{\textit{vars}}}
\knowledgenewrobustcmd{\Var}{\cmdkl{\mathsf{Var}}}

\knowledgenewrobustcmd{\Unif}[1][q]{\cmdkl{\mathbf{M}}_{#1}}

\knowledgenewmathcommandPIE{\IsubA}{%
   \cmdkl{\I#1#2#3}}

\knowledgenewrobustcmd{\dnames}{\cmdkl{\mathsf{N_D}}}
\knowledgenewrobustcmd{\cnames}{\cmdkl{\mathsf{N_C}}}
\knowledgenewrobustcmd{\rnames}{\cmdkl{\mathsf{N_R}}}
\knowledgenewrobustcmd{\inames}{\cmdkl{\mathsf{N_I}}}
\knowledgenewrobustcmd{\nulls}{\cmdkl{\mathsf{N_U}}}
\knowledgenewrobustcmd{\vnames}{\cmdkl{\mathsf{N_V}}}
\knowledgenewrobustcmd{\irnames}{\cmdkl{\mathsf{N^{\pm}_R}}}
\knowledgenewrobustcmd{\terms}{\cmdkl{\mathsf{terms}}}
\knowledgenewrobustcmd{\mods}{\cmdkl{\mathsf{Mod}}}

\knowledgenewrobustcmd{\withT}[1]{(\T,#1)} %

\knowledgenewrobustcmd{\omqsat}{\mathrel{\cmdkl{\models}}} %

\knowledgenewrobustcmd{\dllitec}{ %
\cmdkl{\ensuremath{\mathsf{DL\text{-}Lite}_{\mathsf{core}}}}}
\knowledgenewrobustcmd{\dllitech}{ %
\cmdkl{\ensuremath{\mathsf{DL\text{-}Lite}_{\mathsf{core}}^{\mathcal{H}}}}}
\knowledgenewrobustcmd{\dlliteh}{ %
\cmdkl{\ensuremath{\mathsf{DL\text{-}Lite}_{\mathsf{Horn}}}}}
\knowledgenewrobustcmd{\dlliter}{ %
\cmdkl{\ensuremath{\mathsf{DL\text{-}Lite}_\mathcal{R}}}}%
\knowledgenewrobustcmd{\EL}{\cmdkl{\mathcal{EL}}}
\knowledgenewrobustcmd{\elhibot}{\cmdkl{\mathcal{ELHI}_{\bot}}}
\newcommand{\horndl}{\elhibot}
\knowledgenewrobustcmd{\Lmin}{\cmdkl{\L_{\min}}} %

\knowledgenewrobustcmd{\partsof}[1]{\cmdkl{\mathcal{P}}(#1)} %
\knowledgenewrobustcmd{\fpartsof}[1]{\cmdkl{\mathcal{P}_{\mathsf{f}}}(#1)} %
\knowledgenewrobustcmd{\Sym}{\cmdkl{\mathfrak{S}}} %
\knowledgenewrobustcmd{\Totord}{\cmdkl{\mathfrak{O}}} %

\knowledgenewrobustcmd{\lb}{\cmdkl{\mathrm{lb}}}%

\DeclareMathOperator{\Ima}{Im} %

\knowledgenewrobustcmd{\Esp}{\mathbf{E}} %
\knowledgenewrobustcmd{\Prob}{\mathbf{P}} %
\knowledgenewrobustcmd{\Vari}{\mathbf{V}} %

\knowledgenewrobustcmd{\dcup}{\mathbin{\cmdkl{\uplus}}} %
\knowledgenewrobustcmd{\bigdcup}{\mathop{\cmdkl{\biguplus}}} %

\knowledgenewrobustcmd{\homto}[1][]{\mathrel{\cmdkl{\xrightarrow{\smash{\textit{\tiny #1 \!hom}}}}}} %
\knowledgenewrobustcmd{\Chomto}[1][]{\mathrel{\cmdkl{\xrightarrow{\smash{\textit{\tiny #1 \!$\C$-hom}}}}}} %
\knowledgenewrobustcmd{\Pichomto}[1][]{\mathrel{\cmdkl{\xrightarrow{\smash{\textit{\tiny #1 \!$\Pic$-h}}}}}} %
\knowledgenewrobustcmd{\polyrx}{ %
   \mathrel{\cmdkl{\le_{\mathsf{P}}}}}
\knowledgenewrobustcmd{\polyeq}{ %
   \mathrel{\cmdkl{\equiv_{\mathsf{P}}}}}

\knowledgenewrobustcmd{\class}{\mathcal{C}} %

\knowledgenewmathcommand{\ACQ}{\cmdkl{\mathsf{ACQ}}} %
\knowledgenewmathcommand{\CQ}{\cmdkl{\mathsf{CQ}}} %
\knowledgenewmathcommand{\CQeq}{\cmdkl{\mathsf{CQ}^{=}}} %
\knowledgenewmathcommand{\CQeqneq}{\cmdkl{\mathsf{CQ}^{\neq,=}}} %
\knowledgenewmathcommand{\CQneq}{\cmdkl{\mathsf{CQ}^{\neq}}} %
\knowledgenewmathcommand{\CRPQ}{\cmdkl{\mathsf{CRPQ}}} %
\knowledgenewmathcommand{\IQ}{\cmdkl{\mathsf{IQ}}} %
\knowledgenewmathcommand{\RPQ}{\cmdkl{\mathsf{RPQ}}} %
\knowledgenewmathcommand{\sjfACQ}{\cmdkl{\mathsf{sjf\text{-}ACQ}}} %
\knowledgenewmathcommand{\sjfCQ}{\cmdkl{\mathsf{sjf\text{-}CQ}}} %
\knowledgenewmathcommand{\UCQ}{\cmdkl{\mathsf{UCQ}}} %
\knowledgenewtextcommand{\UCQneg}{\cmdkl{UCQ$^\lnot$}} %
\knowledgenewmathcommand{\UCQneg}{\cmdkl{\mathsf{UCQ}^\lnot}} %
\knowledgenewmathcommand{\UCRPQ}{\cmdkl{\mathsf{UCRPQ}}} %

\knowledgenewrobustcmd{\aC}{\cmdkl{C}} %
\knowledgenewrobustcmd{\aCij}[1][i,j]{\cmdkl{C_{#1}}} %
\knowledgenewrobustcmd{\SNij}[1][i,j]{\cmdkl{M^N_{#1}}} %
\knowledgenewrobustcmd{\SNijtld}[1][i,j]{\cmdkl{\tilde{M}^N_{#1}}} %
\knowledgenewrobustcmd{\Ak}[1][k]{\cmdkl{\A^{#1}}} %
\knowledgenewrobustcmd{\Akchi}[2][k]{\cmdkl{\A_{#2}^{#1}}}
\knowledgenewrobustcmd{\Ao}[1][]{\cmdkl{\A^{\circ}_{#1}}} %

\knowledgenewmathcommandPIE{\AG}[1][G]{%
   \cmdkl{\A}^{\cmdkl{#1}}#2#3#4}
\newcommand{\AGn}{\AG_{\subendo}}
\newcommand{\AGx}{\AG_{\subexo}}
\knowledgenewrobustcmd{\PG}[1][]{\cmdkl{\P^{G}_{#1}}} %
\newcommand{\PGn}{\PG[\subendo]}
\newcommand{\PGx}{\PG[\subexo]}
\knowledgenewrobustcmd{\Axto}{\cmdkl{\A^{\to}_{\subexo}}} %
\knowledgenewrobustcmd{\Axla}{\cmdkl{\A^{\la}_{\subexo}}} %
\knowledgenewrobustcmd{\Below}{\cmdkl{\mathbf{B}}} %
\knowledgenewrobustcmd{\Left}{\cmdkl{\mathbf{L}}} %
\knowledgenewrobustcmd{\Right}{\cmdkl{\mathbf{R}}} %

\knowledgenewrobustcmd{\qreach}[1][s,t]{\cmdkl{r^*(#1)}}
\knowledgenewrobustcmd{\Reach}[1]{\cmdkl{\textsf{CC}(}#1\cmdkl{)}}
\knowledgenewrobustcmd{\rhoA}[1][\A]{\cmdkl{\rho_{#1}}}
\knowledgenewrobustcmd{\reachmodscorefun}{\cmdkl{\xi^N}} %

\knowledgenewmathcommandPIE{\psiN}{%
   \cmdkl{\psi}^{\cmdkl{N}}#1#2#3}

\knowledgenewmathcommand{\etardx}{\cmdkl{\eta}}%
   \input{kl/knowledges.kl} %

 \makeatletter
 \def\ps@pprintTitle{%
 \let\@oddhead\@empty
 \let\@evenhead\@empty
 \def\@oddfoot{}%
 \let\@evenfoot\@oddfoot}
 \makeatother

\begin{document}
\selectlanguage{british}  %

\begin{frontmatter}

\frontpagedeclarations

\begin{abstract}
   
The Shapley value was originally introduced in cooperative game theory as a wealth distribution mechanism.
It has since found use in knowledge representation and databases for the purpose of assigning scores to formulas and 
database tuples based upon their contribution to obtaining a query result or inconsistency. 
The application of the Shapley value outside of its original setting relies upon defining a numeric wealth function 
that captures the phenomenon of interest. In the case of database queries, recent work has focused on the 
so-called \emph{drastic Shapley value}, obtained by translating a Boolean query into a 0/1 function based upon whether 
the query is satisfied or not. 
The present paper explores the use of the drastic Shapley value in the context of ontology-mediated query answering (OMQA). 
We present a detailed complexity analysis of the drastic Shapley value computation ($\dShapleyNoKL{}$) problem in the OMQA setting. 
In particular, we establish a dichotomy result that shows that for every ontology-mediated query $(\T,q)$ composed of an ontology $\T$ formulated in the description logic $\mathcal{ELHI_{\bot}}$ and a connected constant-free homomorphism-closed query~$q$ the corresponding $\dShapleyNoKL{}$ problem is either tractable (in $\mathsf{FP}$) 
or $\mathsf{\#P}$-hard. 
 We further show how %
the $\mathsf{\#P}$-hardness side of the dichotomy can be strengthened to cover possibly disconnected queries with constants. 
Our results exploit recently discovered connections between $\dShapleyNoKL{}$ and probabilistic query evaluation and allow us to generalize existing results on probabilistic OMQA. 

\end{abstract}

\begin{keyword}
Shapley value computation \sep  ontology-mediated query answering \sep quantitative explanations \sep responsibility measures \sep description logic \sep probabilistic query evaluation \sep complexity analysis

\end{keyword}

\end{frontmatter}


\ifpaper
\else

\noindent
\raisebox{-.4ex}{\HandRight}\ \ This pdf contains internal links: clicking on a "notion@@notice" leads to its \AP ""definition@@notice"".%

\fi
\noindent\raisebox{-.4ex}{\HandRight}\ This article is an extended version of the KR 2024 paper \cite{BienvenuFL24} (see page \pageref{par:changes} for differences).

\section{Introduction}

The Shapley value was originally proposed in the context of cooperative game theory 
as a method for fairly distributing the wealth of a coalition of players based upon their 
respective contributions. It has appealing theoretical properties, having been shown to 
be the unique wealth distribution mechanism that satisfies a set of desirable axioms. 
Since its proposal in \cite{shapley:book1952}, 
it has found applications in numerous domains, including various areas of computer science. 
In artificial intelligence, the Shapley value has been utilized for defining inconsistency measures of propositional 
\cite{grantMeasuringInconsistencyKnowledgebases2006,hunterMeasureConflictsShapley2010} 
and description logic knowledge bases \cite{dengMeasuringInconsistenciesOntologies2007},
and more recently for defining explanations of machine learning models \cite{lundbergUnifiedApproachInterpreting2017}. 
The Shapley value has also gained attention recently in database research \cite{DBLP:journals/sigmod/BertossiKLM23}, 
where it has been employed both for defining inconsistency values of databases \cite{livshitsShapleyValueInconsistency2022} 
and also for providing quantitative explanations of query answers \cite{livshitsShapleyValueTuples2021} by assigning 
numeric scores to the database facts based upon their respective contributions to making a query answer hold. 

Importantly, in order to employ the Shapley value for explaining query answers, one must first translate the query
into a numerical wealth function. One simple way to do so is to represent the query (with free variables instantiated with the considered answer tuple) 
as a 0/1 function that returns 1 if the query is satisfied on the given input database, and 0 otherwise. 
This approach, which has been termed the \emph{drastic Shapley value} in \cite{ourpods25} 
because it is closely related to the inconsistency measure of the same name \cite[§5]{livshitsShapleyValueInconsistency2022}, 
has thus far garnered the most attention in the database literature and will be the focus of the present paper. 
It should however be noted that other translations of queries into wealth functions are possible. Indeed, 
the recent work \cite{ourpods25} defines and explores a different Shapley-based measure in the context of database queries, 
based on weighted sums of minimal supports (WSMS). 
Moreover, it is also possible to define responsibility measures for queries without appealing to the Shapley value at all, 
and in fact measures such as \emph{causal responsibility} \cite{meliouComplexityCausalityResponsibility2010} and the \emph{(drastic) Banzhaf power index} (\aka\ \emph{causal effect} \cite{salimiQuantifyingCausalEffects2016}) predate work on Shapley-based measures. 
We direct readers to  \cite{livshitsShapleyValueTuples2021,DBLP:journals/pacmmod/AbramovichDF0O24,ourpods25}
for more details on alternative measures and how they relate to the drastic Shapley value.

In general, the drastic Shapley value computation problem is known to be computationally challenging, being 
$\mathsf{\#P}$-hard in data complexity for common classes of queries, such as conjunctive queries.  
This has motivated non-uniform complexity studies aimed at pinpointing which %
queries admit tractable drastic Shapley value computation \cite{livshitsShapleyValueTuples2021,reshefImpactNegationComplexity2020,khalilComplexityShapleyValue2023}, 
in particular, by establishing fruitful 
connections with probabilistic query evaluation and variants of model counting \cite{deutchComputingShapleyValue2022a,karaShapleyValueModel2023,ourpods24}. 
An %
$\FP$/$\sP$-hard dichotomy in data complexity was established in \cite{livshitsShapleyValueTuples2021} for self-join free conjunctive queries (i.e.\ not having two atoms %
with the same relation name),  and it was shown in particular that the queries which enjoy tractable ($\FP$) drastic Shapley value computation are the same as those admitting efficient evaluation over probabilistic databases. This result has subsequently been extended to cover other classes of queries; in particular, the dichotomy was proven in \cite{ourpods24} to hold for connected homomorphism-closed queries without constants over binary signatures. 

In the present paper, we %
explore the use of the drastic Shapley value in the ontology setting, 
building upon these recent advances in the database area. We shall mostly focus on 
how the drastic Shapley value can be employed for explaining answers in the context of 
ontology-mediated query answering (OMQA). We recall that OMQA is used
to improve access to incomplete and possibly heterogeneous data through the addition of ontology layer, 
which provides a user-friendly vocabulary for query formulation as well as domain knowledge that is taken
into account when computing the query answers. 
Over the past fifteen years, OMQA has grown into a vibrant research topic within 
both the KR and database communities \cite{poggiLinkingDataOntologies2008,DBLP:journals/sigmod/CaliGLP11,mugnierIntroductionOntologyBasedQuery2014,bienvenuOntologyMediatedQueryAnswering2015,xiaoOntologyBasedDataAccess2018}. 
With the increasing maturity and deployment of OMQA techniques, 
there is an acknowledged need to help users understand the query results. 
Various notions of explanations with different levels of detail can be considered
for OMQA, ranging from providing proofs of how an answer can be derived \cite{DBLP:conf/otm/BorgidaCR08,DBLP:conf/ruleml/AlrabbaaBKK22}
to generating minimal subsets of the KB that suffice to obtain the answer or   
identifying the assertions and/or axioms that 
are relevant in the sense that they belong to such a minimal subset \cite{DBLP:journals/jair/BienvenuBG19,DBLP:conf/ijcai/CeylanLMV19,DBLP:conf/ecai/CeylanLMV20,DBLP:conf/kr/PenalozaS10}.
Quantitative responsibility measures such as the drastic Shapley value %
offer a more nuanced, quantitative version 
of the latter approach, by assigning the relevant assertions and axioms scores based upon 
their level of responsibility or importance in obtaining the considered query answer (or entailment).

For our study of drastic Shapley value computation, %
we will work with description logic (DL) knowledge bases (KB), 
consisting of an ABox (dataset) and TBox (ontology). We introduce some natural ways of defining the 
drastic Shapley value computation ($\dShapleyNoKL$) problem in the DL setting, by varying what is to be explained (entailment of a TBox axiom, ABox assertion, or query answer),
which parts of the KB are assigned values, and how the complexity is measured. 
To begin our study, we establish the $\mathsf{\#P}$-hardness of the drastic Shapley value computation of a simple graph reachability query, 
which we then employ to show $\mathsf{\#P}$-hardness of several variants of the $\dShapleyNoKL$ problem, identifying several sources of complexity. 
In particular, we show that $\mathsf{\#P}$-hardness 
applies as soon as we consider problems in which TBox axioms can be assigned scores, 
or consider logics which allow for existential restrictions ($\exists r. C$) or concept conjunction ($\sqcap$). 
While these results establish intractability even in fairly restricted settings, 
we also obtain a tractability result for ontology-mediated queries in which the TBox is formulated in common DL-Lite dialects \cite{DBLP:journals/jar/CalvaneseGLLR07} like $\dlliter$ (which underlies OWL 2 QL profile \cite{profiles})  
and the query consists of a single atom. 
With a view towards identifying further tractable cases, 
we embark on a non-uniform complexity analysis, whose aim is to classify the data complexity 
of the drastic Shapley value computation problems $\dShapleyNoKL_{Q}$ associated with each ontology-mediated query (OMQ) $Q=(\T,q)$.
By transferring recent results from the database setting, we establish a $\mathsf{FP}$/$\mathsf{\#P}$-hard dichotomy result of the drastic Shapley value
computation problem $\dShapleyNoKL_{Q}$ for OMQs $Q=(\T,q)$ where the TBox $\T$ is formulated in the Horn DL $\mathcal{ELHI_{\bot}}$ 
and $q$ is a constant-free connected homomorphism-closed query. 
Moreover, when restricted to the case where $q$ is a conjunctive query, the dichotomy is \AP ""effective"", i.e. we can decide whether
$\dShapleyNoKL_{Q}$ is $\mathsf{FP}$ or $\mathsf{\#P}$-hard. 
Our final and most technically challenging result is to show that the $\mathsf{\#P}$-hardness part of the 
dichotomy can be strengthened to cover OMQs based upon a wider range of queries $q$. 
Specifically, we show that for any OMQ $Q=(\T,q)$ based upon a $\mathcal{ELHI_{\bot}}$ TBox and a UCQ $q$ (which may be disconnected and/or contain constants),
non-FO-rewritability of $Q$ implies $\mathsf{\#P}$-hardness of $\dShapleyNoKL_{Q}$. Due to the tight connections holding between 
drastic Shapley value computation and probabilistic query evaluation, the proof of this result can be further used to 
obtain an $\mathsf{FP}$/$\mathsf{\#P}$-hard dichotomy for probabilistic ontology-mediated queries from $(\mathcal{ELHI_{\bot}},\mathsf{UCQ})$, 
substantially generalizing existing results. 

The paper is structured as follows. \Cref{sec:prelims} introduces key notions from databases and description logics,
and \Cref{sec:shapley} defines Shapley values and recalls useful results about their applications in databases. We also prove a new hardness result
for graph reachability queries, which we apply in \Cref{sec:reach} to show hardness of the drastic Shapley value computation
in various ontology settings. In \Cref{sec:rw}, we present our $\mathsf{FP}$/$\mathsf{\#P}$-hard dichotomy result for OMQs in the Horn DL $\mathcal{ELHI_{\bot}}$,
and in \Cref{sec:non-rw},  we strengthen the $\mathsf{\#P}$-hardness result to cover a wider range of queries, obtaining as a by-product an 
improved dichotomy result for probabilistic ontology-mediated query ansering.  
We conclude the paper in \Cref{sec:discuss} 
with a summary of our contributions and a discussion of future work.

\subsection*{Changes with respect to the conference paper} %
\label{par:changes}
The main results of this work were first presented in a conference paper at KR 2024 \cite{BienvenuFL24}, but with few proof details. 
In addition to including detailed proofs of all results,
the present paper also contains some new material beyond the conference version.  
First, several of our hardness results (specifically, \Cref{prop:graph,prop:kb1,prop:kb2}) 
have been improved to work even in the absence of exogenous facts.
Second, \Cref{prop:dllite-atomic} on the tractability of Shapley value computation for 
 instance queries over exogenous DL-Lite ontologies has been extended to dialects more general than $\dllitec$.
Third, we have added a new hardness result (\Cref{prop:kb3}) covering DLs admitting axioms with concept conjunction. 
Fourth, we now provide formal statements and proofs of \Cref{prop:addfpras} and \Cref{prop:hardapprox} about approximating Shapely values.  
Finally, we have added new discussion in \Cref{sec:discuss} about recently introduced alternatives to the "drastic Shapley value".

\color{black}

\section{Preliminaries}\label{sec:prelims}
\begin{toappendix}
\label{app:prelims}
\end{toappendix}
We recall some important notions %
related to description logics (DLs), databases,  queries, and complexity, 
directing readers to 
\cite{DBLP:books/daglib/0041477} for a detailed introduction to DLs. 
Note that our presentation of DLs and databases slightly differs from the ``usual'' ones
so that we may employ some definitions and notations in both settings.

\subsection{Databases} 
\AP A ""database"" $\D$ is a finite set of relational facts $P(\vec{a})$, where $P$ is a $k$-ary symbol
drawn from a countably infinite set of relation symbols $\intro*\dnames$
and $\vec{a}$ is a $k$-ary tuple of (""individual"") ""constants"" drawn from a countably infinite set~$\intro*\inames$.
We shall also consider \AP ""extended databases"" which may contain infinitely many facts $P(\vec{a})$,
and where the elements of $\vec{a}$ are drawn from $\inames$ and from a countably infinite
set $\intro*\nulls$ of unnamed elements. 
The \AP ""domain"" $\dom(\D)$ of $\D$ contains all constants and unnamed elements occurring in $\D$,
and we use $\intro*\const(\D)$ for the constants in $\D$. When $\D$ is a database, $\dom(D)=\const(\D)$. 

A \AP ""homomorphism"" from an "extended database" 
$\D$ to an "extended database" $D'$ is a function 
$h: \dom(\D) \rightarrow \dom(\D')$ such that 
$P(h(\vec{a})) \in D'$ for every $P(\vec{a}) \in \D$. 
We write $\D \AP\intro*\homto \D'$ to indicate the existence of such an $h$. 
If additionally $h(c)=c$ for every $c \in C  \cap \const(\D)$, with $C \subseteq \inames$,
then we call $h$ a "$C$-homomorphism" and write $\D \reintro*\homto[C-] \D'$. 

We say that a (possibly extended) database $\D$ is ""connected@@db"" if so is 
the underlying undirected graph with vertices $\dom(\D) \cup \D$ 
and edges $\{(a_i, P(\vec{a})) \mid 1 \leq i \leq n, P(a_1, \ldots, a_n) \in \D\}$.
The ""connected components@@db"" of $\D$ are the maximal subsets of $\D$
that are connected in the underlying graph.%

\subsection{Queries}
\AP In the most general sense, a \emph{$k$-ary} ""query"" ($k \geq 0$) can be defined as
a function $q$ that maps every extended database $\D$
to a set of $k$-tuples of constants from $\const(\D)$ (the \emph{answers} to $q$). 
Queries of arity $0$ are called \AP ""Boolean"". 
When $q$ is a Boolean query, each $\D$ is mapped either to $\{()\}$ or $\{\}$. 
In the former case, we say that $\D$ ""satisfies@@q"" $q$ and write $\D \models q$. 
If additionally $\D' \not \models q$
for every $\D' \subsetneq \D$, 
then we shall call $\D$ a \AP""minimal support for $q$"".

\AP A Boolean query $q$ is said to be
""monotone"" if $\D\models q$ and $\D \inc \D'$ implies $\D'\models q$, and
""closed under homomorphisms"", or \reintro{hom-closed},
if $\D \models q$ and $\D \homto[] \D'$ implies that $\D' \models q$. The notion of ""$C$-hom-closed""
is defined analogously using $\homto[C-]$ instead of $\homto[]$.
When $q$ is ($C$-)hom-closed
or "monotone" in general%
, $\D \models q$ iff $\D$ contains some "minimal support" for $q$; we say that $q$ is ""connected@@q"" if all its "minimal supports" are connected.%

So far we have considered an abstract notion of query, but in practice, queries 
are often specified  
in concrete query languages.
First-order (FO) queries are given by formulas in first-order predicate logic with equality, 
whose relational atoms are built from predicates from $\dnames$ and terms drawn from 
$\inames \cup \vnames$, with $\intro*\vnames$ a countably infinite set of variables,
equipped with standard FO logic semantics (i.e. $\D \models q$ if $\D$, viewed as a first-order
structure, satisfies the FO sentence $q$).
Two prominent classes of FO queries are 
\AP ""conjunctive queries"" (\reintro{CQs}\phantomintro{\CQ}), which are finite conjunctions of relational atoms whose variables may be existentially quantified,  %
and \AP ""unions of conjunctive queries"" (\reintro{UCQs}\phantomintro{\UCQ}) which 
are finite disjunctions of "CQs" having the same free variables. 
We remark that Boolean (U)CQs without constants are "hom-closed", and 
Boolean (U)CQs with constants in $\intro*\aC$ are $\aC$-"hom-closed". 
Other well-known examples of ($\aC$)-hom-closed queries include 
Datalog queries and regular path queries (RPQs). %

\subsection{Description Logic Knowledge Bases}\label{prelims-dl}\AP A "DL" ""knowledge base"" (\reintro{KB}) $\K = (\A,\T)$ comprises an "ABox" (dataset) $\A$ and a "TBox" ("ontology") $\T$,
which are built from countably infinite sets $\intro*\cnames$ of ""concept names"" (unary predicates) and $\intro*\rnames$ of ""role names"" (binary predicates)
with $\cnames \cup \rnames \subseteq \dnames$,
and the individual constants from $\inames$. 
\AP
An ""ABox"" is a database with relations drawn from $\cnames \cup \rnames$ and thus contains two kinds of facts:\phantomintro{assertion}
 ""concept assertions"" $A(c)$ ($A \in \cnames, c \in \inames$) and ""role assertions"" $r(c,d)$ ($r \in \rnames, c,d \in \inames$).
\AP
A ""TBox"" is a finite set of ""axioms"", whose form is dictated by the ""DL"" in question. We will use 
\emph{$\mathcal{L}$ TBox} to refer to a TBox formulated in the DL $\mathcal{L}$. %

The most expressive DL considered in this paper is \AP$\intro*\horndl$, in which 
""complex concepts"" are constructed as follows:
$$C := \top \mid A \mid C \sqcap C' \mid \exists R. C  \qquad A \in \cnames, R \in \irnames$$
where $\intro*\irnames = \rnames \cup \{r^- \mid r \in \rnames\}$. 
In $\horndl$, "TBoxes" consist of ""concept inclusions"" $C \sqsubseteq D$ (with $C,D$ complex concepts)
and ""role inclusions"" $R \sqsubseteq S$ with $R,S \in \irnames$. 
The "DL" $\horndl$ contains several well-known tractable DLs as fragments, 
such as the logic $\intro*\EL$ \cite{DBLP:conf/ijcai/BaaderBL05}, which is obtained from $\horndl$ by disallowing inverse roles 
and role inclusions, as well as some prominent dialects of the DL-Lite family \cite{DBLP:journals/jar/CalvaneseGLLR07,DBLP:journals/jair/ArtaleCKZ09}. 
In the simplest DL-Lite dialect, called $\intro*\dllitec$, "TBoxes" composed of "concept inclusions" of the form 
$$ B_1 \sqsubseteq (\neg) B_2 \qquad B_i := A \mid \exists R. \top \quad A \in \cnames, R \in \irnames$$
Two other prominent DL-Lite dialects are $\intro*\dlliteh$, obtained by extending $\dllitec$ by allowing for conjunctions ($\sqcap$) of concepts on the left-hand-side 
of concept inclusions, and $\intro*\dlliter$, which extends $\dllitec$ by allowing 
role inclusions $R_1 \sqsubseteq (\neg) R_2$ ($R_i \in \irnames$).%

\AP The semantics of DL "KB"s is defined using ""interpretations"" $\I= (\Delta^\I, \cdot^\I)$,
where $\Delta^\I \subseteq \inames \cup \nulls$ is a non-empty %
set 
and $\cdot^\I$ a function\footnote{To simplify the comparison with the database setting, we
make the \emph{standard names assumption}, interpreting constants as themselves, 
but our results also hold under the weaker \emph{unique names assumption}. %
Moreover, to allow for finite interpretation domains, we do not require all constants to be interpreted. } that maps every 
$A \in \cnames$ to a set $A^\I \subseteq \Delta^\I$ and  every $r \in \rnames$ to a binary relation
$r^\I \subseteq \Delta^\I \times \Delta^\I$. %
The %
function~$\cdot^\I$ is straightforwardly extended to interpret complex concepts and roles: %
$\top^\I= \Delta^\I$, $(C \sqcap D)^\I = C^\I \cap D^\I$,  $(\exists R. C)^\I= \{d \mid \exists\, e \in C^\I \text{ s.t. } (d,e) \in R^\I %
\}$, $(\neg G)^\I = \Delta^\I \setminus G^\I$, and  $(r^-)^\I=\{(e,d) \mid (d,e) \in r^\I\}$.
Note that by requiring that $\Delta^\I \subseteq \inames \cup \nulls$, 
we ensure that every interpretation $\I$ can be viewed as an extended database $\D_\I=\{A(e) \mid e \in A^\I\} \cup \{r(d,e) \mid (d,e) \in r^\I\}$,
and we shall view $\I$ as an extended database when convenient. %

\AP An interpretation $\I$ ""satisfies a (concept or role) inclusion""\phantomintro{\omqsat} $G \sqsubseteq H$ if $G^\I \subseteq H^\I$,
and it \reintro{satisfies an assertion} $A(c)$ (resp.\ $r(c,d)$) if $c \in A^\I$ (resp.\ $(c,d) \in r^\I$). 
We call $\I$ a ""model of a TBox"" $\T$ if it satisfies every axiom in $\T$,
a ""model of an ABox"" $\A$ if it satisfies every "assertion" in $\A$, %
and %
a ""model of a KB"" $(\A,\T)$ if it is a model of both $\T$ and $\A$. 
We use $\intro*\mods(\K)$ for the set of "models of a KB" $\K$. 
A "KB" $\K$ is ""consistent"" if $\mods(\K)\neq \emptyset$ (else it is \reintro{inconsistent}). 
An "ABox" $\A$ is "$\T$-consistent" when the "KB" 
$(\A, \T)$ is "consistent". 
An axiom $\alpha$ is ""entailed from a TBox"" $\T$, written $\T \models \alpha$,
if every "model@model of a TBox" of $\T$ "satisfies@@omq" $\alpha$, and %
an axiom or "assertion" $\alpha$
is ""entailed from a KB"" $\K$, 
written $\K \models \alpha$, if every "model@model of the KB" of $\K$ "satisfies@@omq"~$\alpha$. 

Later in the paper, it will be convenient to assume that $\horndl$
"TBoxes" only use $\bot$ in axioms of the form $C \sqsubseteq \bot$, with $C$
a possibly complex concept that does not mention $\bot$. We will say that 
such a "TBox" is in \AP""$\bot$-normal form"". Note that this can be assumed w.l.o.g.\ since 
any "TBox" can be transformed 
into an equivalent "TBox" in 
"$\bot$-normal form" in polynomial time: first remove all inclusions $C_\bot \sqsubseteq D$ 
in which $C_\bot$ contains $\bot$, then replace any inclusion $C \sqsubseteq D_\bot$
in which $D_\bot$ contains $\bot$ by $C \sqsubseteq \bot$.

\subsection{Querying Description Logic KBs} 

We say that %
a Boolean query $q$ is "entailed from a DL KB"~$\K$, 
written $\K \models q$, if $\D_\I \models q$ for every $\I \in \mods(\K)$.
\AP The ""certain answers"" to a non-"Boolean" $k$-ary query $q(\vec{x})$ \wrt\  a "KB" $\K = (\A, \T)$ are the 
$k$-tuples $\vec{a}$ of "constants" from $\const(\A)$ %
such that $\K \models q(\vec{a})$, with $q(\vec{a})$ the Boolean query 
obtained by substituting $\vec{a}$ for the free variables $\vec{x}$. 
Note that when the KB %
 is "inconsistent", every "Boolean" query is trivially entailed,  
so every 
possible tuple $\vec{a}$
of "ABox" "constants" counts as 
a certain answer. 

While it is traditional to view queries as being posed to the KB, 
it is sometimes more convenient to adopt a database perspective and 
treat $\T$ and $q$ together as constituting a composite \AP ""ontology-mediated query"" (\reintro{OMQ}) $Q=(\T, q)$, which 
is posed to the "ABox" $\A$. %
When we adopt this perspective, we will write
$\A \models (\T, q)$ or $\A \models Q$ to mean $(\A, \T) \models q$. 
We will use the notation $(\mathcal{L}, \mathcal{Q})$
to designate the class of all OMQs $(\T, q)$ such that $\T$ is formulated in the DL $\mathcal{L}$
and $q$ is a query from the class of queries $\mathcal{Q}$.
Most work on OMQA assumes that the input query takes either the form of a CQ or has the more restricted form of an \AP""instance query"" (\reintro{IQ}\phantomintro{\IQ})\footnote{Sometimes "instance queries" are defined more generally as $C(x)$, where $C$ may be a complex concept. However, it is well known that by introducing a fresh concept name $A_C$ and adding inclusions $A_C \sqsubseteq C$ and $C \sqsubseteq A_C$ to the TBox, we may focus w.l.o.g on IQs involving only concept names. }, consisting of a single concept atom $A(x)$ with $A \in \cnames$ and $x$ an answer variable (or $A(c)$, with $c\in \inames$, if we focus on Boolean instance queries). 
We shall thus present results for OMQ languages $(\mathcal{L}, \mathcal{Q})$ where $\mathcal{Q}$ is either $\CQ$, $\UCQ$ or $\IQ$. 

A prominent technique for computing certain answers (or checking query entailment)
is to rewrite an "OMQ" into another query that can be directly evaluated using a database system. 
Formally, we call a query $q^*(\vec{x})$ a \AP ""rewriting of an OMQ"" $(\T,q)$ if 
for every "ABox" $\A$ and candidate answer %
 $\vec{a}$: 
$$\A \models (\T,q(\vec{a})) \quad \text{ iff }  \quad \A \models q^*(\vec{a})$$
If we modify the above definition to only quantify over "$\T$-consistent" "ABoxes",
then we speak instead of a \AP ""rewriting w.r.t.\ consistent ABoxes"". 
When %
 $q^*$ is a first-order query, we call it a \AP ""first-order (FO) rewriting"".
If an "OMQ" $Q$ possesses an "FO-rewriting", we say that $Q$ is \reintro{FO-rewritable},
else it is called \reintro{non-FO-rewritable}. 

It will also prove useful to be able to switch back and forth between the consistency 
and query entailment tasks. 
The following lemma recalls a standard technique for reducing inconsistency checking to 
query answering, in which we replace $\bot$ by a fresh concept that is used to represent inconsistency. 
It further shows how we can reduce query answering w.r.t.\ arbitrary ABoxes (in which case the query may be trivially entailed due to inconsistency)
to query answering w.r.t.\ a $\bot$-free TBox (without need to consider consistency).
The proof is routine but is provided in \ref{app:prelims} for completeness.

\begin{lemmarep}\label{incons-bot}
Let $\T$ be an $\horndl$ TBox in "$\bot$-normal form",
and let $\T'$ be the TBox obtained from $\T$ by replacing every occurrence of $\bot$ with $A_\bot$, where 
$A_\bot$ is a fresh concept name. For every ABox $\A$ that does not contain $A_\bot$: 
\begin{itemize}
\item $(\A,\T)$ is inconsistent iff $(\A, \T') \models \exists x. A_\bot(x)$
\item $(\A, \T) \models q$ iff $(\A, \T') \models q \vee \exists x. A_\bot(x)$ 
 \end{itemize}
\end{lemmarep}
\begin{proof}
Consider an ABox $\A$ that does not mention $A_\bot$. To establish the first statement, 
let us suppose that $\A$ is $\T$-inconsistent. 
Then this means that for every model $\I'$ of $(\A, \T')$, 
there must exist $C \sqsubseteq \bot \in \T$ such that 
 $C^\I \neq \emptyset$. It follows that $A_\bot^\I \neq \emptyset$,
from which we conclude $(\A, \T') \models \exists x. A_\bot(x)$. 
If instead $\A$ is $\T$-consistent,
then there exists a model $\I$ of $(\A, \T)$.
It follows that for every inclusion $C \sqsubseteq \bot \in \T$, 
we must have  $C^\I = \emptyset$. By setting $A_\bot^{\I'} = \emptyset$,
we obtain a model $\I'$ of $(\A, \T')$ witnessing that $(\A, \T') \not \models \exists x. A_\bot(x)$. 

For the second statement, suppose that $(\A, \T) \models q$. 
If $(\A,\T)$ is inconsistent, then the first statement tells us that $(\A, \T') \models \exists x. A_\bot(x)$,
hence $(\A, \T') \models q \vee \exists x. A_\bot(x)$. 
Next suppose $(\A,\T)$ is consistent, and let $\I'$ be a model of $(\A, \T')$. 
Either $A_\bot^{\I'} \neq \emptyset$, in which case $\exists x. A_\bot(x)$ holds,
or $A_\bot^{\I'} = \emptyset$, in which case $\I'$ is a model of $(\A, \T)$
and hence must satisfy $q$. We thus obtain $(\A, \T') \models q \vee \exists x. A_\bot(x)$, as required. 
\end{proof}

We shall also make use of another useful property of query entailment in $\horndl$, namely, that 
when a "CQ" is entailed, we can always extract a ``tree-like'' ABox that is sufficient to entail the query 
and can be homomorphically embedded in the given ABox. Before we can formalize this property, 
we recall some necessary definitions, taken from \cite{bienvenuFirstOrderrewritabilityContainment2016a}. 

\AP
An "ABox" $\A$ is ""tree-shaped"" if the undirected graph with nodes $\const(\A)$ and edges $\{\{a,b\}\mid r(a,b)\in \A\}$ is acyclic and connected and $r(a,b)\in\A$ implies that
\begin{enumerate*}[(i)]
\item $s(a,b)\notin \A$ for all $s\neq r$ and
\item $s(b,a)\notin \A$ for all "role names"~$s$.
\end{enumerate*}
For "tree-shaped" "ABoxes" $\A$, we often distinguish an "individual" used as the root, denoted with $\intro*\rhoA$.
\AP An "ABox" $\A$ is a ""pseudo tree"" if it is the union of "ABoxes" $\A_0,\dots,\A_k$ that satisfy the following conditions:
\begin{enumerate}
\item $\A_1,\dots,\A_k$ are "tree-shaped";
\item $k\le |\const(\A_0)|$;
\item $\A_i\cap\A_0=\{\rhoA[\A_i]\}$ and $\const(\A_i)\cap\const(\A_j)=\emptyset$, for $1\leq i < j \leq k$.
\end{enumerate}
\AP We call $\A_0$ the ""core@@ptree"" of $\A$ and $\A_1,\dots,\A_k$ the ""trees@@ptree"" of $\A$. The ""width@@ptree"" of $\A$ is $|\const(\A_0)|$, its ""depth@@ptree"" is the depth of the deepest "tree@@ptree" of $\A$, and its ""outdegree@@ptree"" is the maximum outdegree of the "ABoxes" $\A_1,\dots,\A_k$. For a "pseudo tree" "ABox" $\A$ and $l\ge 0$, we write $\A|_{\le l}$ to denote the restriction of $\A$ to the "individuals" whose minimal distance from a "core@@ptree" individual is at most $l$, and analogously for $\A|_{>l}$.
We can now state the desired property, whose proof relies upon first `unravelling' the given ABox and then employing compactness. We refer readers to 
\cite{bienvenuFirstOrderRewritabilityContainment2020} for details. 

\begin{proposition}[{\cite[Proposition 23]{bienvenuFirstOrderRewritabilityContainment2020}}]\label{prop:unraveling}%
\hspace{-1ex}\footnote{We have slightly modified the formulation to suit our setting, dropping mention of ABox signatures and functional roles, not considered in this work, and rephrasing in terms of Boolean queries with constants, rather than queries with answer variables. }
Let $Q=\withT{q}$ be a "Boolean" "OMQ" from $(\mathcal{ELIH_{\bot}},\CQ)$ with set of constants $\aC$, and let $\A$ be an "ABox" that is "$\T$-consistent" such that $\A \omqsat Q$. Then there is a "pseudo tree" "ABox" $\A^*$ that is "$\T$-consistent", of "width@@ptree" at most $|q|$, of "depth@@ptree" bounded by $|\T|$ and such that all constants from $C$ occur in the "core@@ptree" of $\A^*$ and the following conditions are satisfied:
\begin{enumerate}[(a)]
\item\label{prop:unraveling-a} $\A^*\omqsat Q$;
\item\label{prop:unraveling-b} $\A^* \homto[C-] \A$.
\end{enumerate}
\end{proposition}

\subsection{Canonical Models}\label{canmoddef}
\def\canmod{\ensuremath{{\I_{\A,\T}}}\xspace}
\def\canmodk{\ensuremath{{\I_{\K}}}\xspace}
\def\IK{\canmodk}

In Horn DLs, like $\horndl$ and $\dllitec$, every "consistent" "KB" admits a so-called canonical (or universal) model, which embeds homomorphically into every other model. As one of our proofs will rely upon the structure of the canonical model, we shall recall its construction  for  the DL $\horndl$, adopting the definition given in \cite{bienvenuOntologyMediatedQueryAnswering2015}. 

Given a satisfiable $\horndl$ KB $\K=(\A,\T)$, we consider
the interpretation $\canmod$ 
defined as follows. 
The domain $\Delta^{\I_{\A,\T}}$ of $\canmod$  consists of sequences of the form $a R_{1} M_{1}
\ldots R_{n} M_{n}$ ($n \geq 0$) such that $a \in \const(\A)$, 
each $R_i$ is a role from $\irnames$,  each $M_i$ is a conjunction of concepts from $\cnames \cup \{\top\}$, and the following conditions are satisfied (note that here and later, we shall sometimes abuse notation by treating conjunctions of concepts as sets):
\begin{itemize}
   \item If $n \geq 1$, then  $\K \models  \exists R_1. M_1(a)$ and there is no $M_1' \supsetneq M_1$ such that  
$\K \models  \exists R_1. M_1'(a)$ %
\item For every $1 \leq i < n$, $\T \models M_{i} \sqsubseteq \exists R_{i+1}. M_{i+1} $ 
and there is no $M_{i+1}' \supsetneq M_{i+1}$ such that  $\T \models M_{i} \sqsubseteq \exists R_{i+1}. M'_{i+1} $ 
\end{itemize}
Note that we will suppose that all elements in $\Delta^{\I_{\A,\T}} \setminus \const(\A)$
belong to~$\nulls$.
Concept names and role names are interpreted as follows:
\begin{align*}
A^{\I_{\A,\T}} \,\, = &\,\, \{ a \in \const(\A) \mid  \K \models A(a)
\} \, \cup \\ &\,\,   \{ e \in
\Delta^{\I_{\A,\T}}\setminus \const(\A) \mid e = e' R M \text{ and }
A \text{ a conjunct of } M \}   \\
r^{\I_{\A,\T}} \,\,= &\,\, \{ (a,b) \mid \K \models r(a,b)\} \, \cup \\
&\,\, \{
(e_{1},e_{2})
\mid
e_{2} = e_{1} S  \, M  \text{ and } \T \models S \sqsubseteq r \} \,\cup   \\
& \,\, \{ (e_{2},e_{1})
\mid 
e_{2} = e_{1} S \, M  \text{ and } \T \models S \sqsubseteq r^{-}  \} 
\end{align*}
Intuitively, $\canmod$ completes the "ABox" by adding tree-shaped structures
using unnamed elements in order to satisfy the "TBox" axioms in the least constrained way possible. We call $\I_{\A,\T}$ the \AP ""canonical model"" of $\K=(\A, \T)$. 

By construction, $\canmod$ is a model of the KB $\K=(\A, \T)$. Moreover, it embeds homomorphically into every other model of $\K$: 
$\D_{\I_{\A,\T}} \homto[$C$-] \D_\I$ for every $\I \in \mods(\K)$ and every $C \subseteq \inames$. 
Importantly, this means that for "$C$-hom-closed" queries such as CQs and UCQs, 
we can obtain the certain answers (which are defined w.r.t.\ all models of the KB) by evaluating the query in the canonical model. 
Formally:

\begin{theorem}[{\cite[Theorem 13]{bienvenuOntologyMediatedQueryAnswering2015}}]\!\!\footnote{We cite the paper in which the particular construction we use was presented, but similar results were proven earlier. See \cite[Proposition 2]{DBLP:conf/aaai/EiterOSTX12} for a superficially different but closely related canonical model construction for Horn-$\mathcal{SHIQ}$, which extends $\horndl$.} \label{canmodprops}%
If $\K=(\A, \T)$ is a satisfiable $\horndl$ KB, 
then for every "$C$-hom-closed" query~$q$ (with $C \subseteq \inames$): %
$$\K \models q(\vec{a}) \quad \text{ iff } \quad \D_{\I_{\A,\T}} \models q(\vec{a})$$
\end{theorem}

We shall also make use of the following lemma, which shows that if there is a homomorphism between two ABoxes, 
then there is a corresponding homomorphism between (the extended databases of) their canonical models. 
The proof is not difficult but a bit tedious, so we relegate it to \ref{app:prelims}.

\begin{lemmarep}\label{aboxhomcanmod}
Let $\T$ be an $\horndl$ TBox, and let $\A$ and $\B$ be $\T$-consistent ABoxes. 
If $\A \homto[C-] \B$, then $\D_{\I_{\A,\T}} \homto[C-] \D_{\I_{\B,\T}}$.
\end{lemmarep}
\begin{proof}
Suppose that $\T$ is an $\horndl$ TBox, and $\A$ and $\B$ are $\T$-consistent ABoxes
such that $\A \homto[C-] \B$. Let $h$ be a homomorphism witnessing that $\A \homto[C-] \B$.
We can use $h$ to define a homomorphism $h'$ from $\D_{\I_{\A,\T}}$ to $\D_{\I_{\B,\T}}$
inductively as follows:
\begin{itemize}
\item for every $a \in \const(\A)$, set $h'(a)=h(a)$
\item for every $e S M \in \Delta^{\I_{\A,\T}}$ (with $e \in \Delta^{\I_{\A,\T}}$, $S \in \irnames$, $M \subseteq \cnames$), set $h'(e S  M) = h'(e) S  M'$ %
for some $M'$ 
s.t.\ $M \subseteq M'$ and $h'(e) S  M' \in \Delta^{\I_{\B,\T}}$ %
\end{itemize}
It is easy to see from the canonical model construction that every element of $\Delta^{\I_{\A,\T}}$ has one of the preceding two forms. 
In what follows, we shall prove by induction that $h'$ is well defined (that is, it is always possible to find $M'$ satisfying the conditions of the second item) and that is 
a homomorphism. 
Note that since $h'$ is equal to $h$ on $\const(\A)$, and $h(c) =c$ for all $c \in \const(\A) \cap C$, 
this suffices to show that $h'$ is a $C$-homomorphism from $\D_{\I_{\A,\T}}$ to $\D_{\I_{\B,\T}}$. 

The induction will proceed according to the depth of elements in $\Delta^{\I_{\A,\T}}$, 
where constants in $\const(\A)$ have depth $0$ and an element $a R_{1} M_{1}
\ldots R_{n} M_{n}$ has depth $n$. 
For the base case, 
we show that $h'$ satisfies the required conditions 
when restricted to depth-0 elements, i.e.\ the constants in $\const(\A)$. 
To this end, first suppose that $r(a,b) \in \D_{\I_{\A,\T}}$ for $a,b \in \const(\A)$. 
Then there must exist a (possibly inverse) role $S$ such that $S(a,b) \in \A$ and $\T \models S \sqsubseteq r$.
Since $h$ is a homomorphism from $\A$ to $\B$, it follows that $S(h(a), h(b)) = S(h'(a),h'(b)) \in \B$, 
hence $\B,\T \models r(h'(a),h'(b))$ and $r(h'(a),h'(b)) \in \D_{\I_{\B,\T}}$, as required. 
Next take some $a \in \const(\A)$ such that $D(a) \in \D_{\I_{\A,\T}}$, 
and suppose for a contradiction that 
$D(h'(a)) \not \in \D_{\I_{\B,\T}}$. 
Let us define an interpretation $\J$ with domain $\const(\A) \cup \Delta^{\I_{\B,\T}}$ as follows:
\begin{align*}
A^{\J} \,\, = &\,\, A^{\I_{\B,\T}} \, \cup \, \{ a \mid  a \in \const(\A) \text{ and } h'(a) \in A^{\I_{\B,\T}}\}  \\
r^{\J} \,\,= &\,\, r^{\I_{\B,\T}} \, \cup \, \{ (a,b) \mid  (h'(a),h'(b)) \in r^{\I_{\B,\T}}\} \, \cup  \\
&\,\, \{
(a,e) \mid a \in \const(\A) \text{ and } (h'(a),e) \in r^{\I_{\B,\T}}\} \, \cup \\
& \,\, \{ (e,a) 
\mid 
a \in \const(\A) \text{ and } (e,h'(a)) \in r^{\I_{\B,\T}}\}
\end{align*}
Observe that $\J$ is a model of $\A$. Indeed, if $E(c) \in \A$, then $E(h(c)) \in \B$ (due to $h$ being a 
homomorphism), hence $E(h'(c)) \in \B$ (since $h'(c)=h(c)$), yielding $c \in E^{\J}$. Similarly, if $s(c,d) \in \A$, 
then $s(h(c), h(d)) = s(h'(c),h'(d)) \in \B$, yielding $(c,d) \in s^{\J}$. We can further show that $\J$
is a model of $\T$. This can be done by first proving, by a straightforward induction on the size of concepts, 
that for every (potentially complex) concept $C$, we have $c \in C^{\J}$ iff $h'(c) \in C^{\I_{\B,\T}}$
for every $c \in \const(\A)$, and $e \in C^{\J}$ iff $e \in C^{\I_{\B,\T}}$ for every $e \in \Delta^{\I_{\B,\T}}$. %
Next consider an axiom $C_1 \sqsubseteq C_2 \in \T$. If $c \in C_1^{\J}$ for $c \in \const(\A)$, 
then $h'(c) \in C_1^{\I_{\B,\T}}\}$, hence $h'(c) \in C_2^{\I_{\B,\T}}$ (since $\I_{\B,\T}$ is a model of $\T$),
which yields $c \in C_2^{\J}$. Similarly, if $e \in C_1^{\J}$ for $e \in \Delta^{\I_{\B,\T}}$,
then $e \in C_1^{\I_{\B,\T}}$, hence $e \in C_2^{\I_{\B,\T}}$, so we get $e \in C_2^{\J}$.
It is also easily seen that all role inclusions in $\T$ also hold in $\J$ due to how roles are interpreted in $\J$, following $\I_{\B,\T}$. 
It only remains to observe that $a \not \in A^{\J}$ due to the definition of $\J$ and assumption that $D(h'(a)) \not \in \D_{\I_{\B,\T}}$.
It follows that $\A,\T \not \models D(a)$, which contradicts $D(a) \in \D_{\I_{\A,\T}}$. We can thus conclude that 
$h'$ is a homomorphism of $\D_{\I_{\A,\T}}$ restricted to $\const(\A)$ into $\D_{\I_{\B,\T}}$.

For the induction step, let us suppose that $h'$ is a homomorphism when restricted to elements of $\Delta^{\I_{\A,\T}}$
of depth at most $k$, and take some  $e  S  M \in \Delta^{\I_{\A,\T}}$ of depth $k+1$. 
Since $e  S  M \in \Delta^{\I_{\A,\T}}$, we know from the definition of $\I_{\A,\T}$  that %
$M_e = \{A \in \cnames \mid e \in A^{\I_{\A,\T}}\}$ is such that $\T \models M_e \sqsubseteq \exists S. M$. 
Due to the induction hypothesis,  for every $A \in M_e$, we have $A(h'(e)) \in \D_{\I_{\B,\T}}$, and hence $h'(e) \in (\exists S. M)^{\I_{\B,\T}}$. 
It follows from the definition of $\I_{\B,\T}$ that there must exist $M'$ such that $M \subseteq M'$ and $h'(e) S  M' \in \Delta^{\I_{\B,\T}}$.
This means that the definition of $h'$ is well defined and so $h'(e S  M) = h'(e) S  M'$ for some such set $M'$. 
Note that for every $A \in \cnames$, if $A(e  S  M) \in \D_{\I_{\A,\T}}$, then $A \in M$, hence $A(h'(e)  S  M') \in \D_{\I_{\B,\T}}$, as required. 
Moreover, due to the structure of $\D_{\I_{\A,\T}}$, the only role assertions involving $e  S  M$ and elements of depth at most $k$
are role assertions between $e S   M$ and its `predecessor' $e$, 
which are either of the form $r(e, e  S   M)$ where $\T \models S \sqsubseteq r$
or $r'(e  S   M, e)$ where $\T \models S^- \sqsubseteq r'$. 
Let us consider some such assertion $r(e, e  S  M)  \in \D_{\I_{\A,\T}}$ (the argument for assertions of the form $r'(e  S   M, e)$  is analogous). 
We have already seen that $h'(e) S   M' \in \Delta^{\I_{\B,\T}}$, hence 
we will have $(h'(e), h'(e)  S   M') \in S^{\I_{\B,\T}}$ and thus $r(h'(e), h'(e)  S   M') \in \D_{\I_{\B,\T}}$, as required. 
This completes the argument for the induction step, allowing us to conclude that $h'$ is a homomorphism.
\end{proof}

\subsection{Complexity}\nointro{\NP}%
The Shapley value computation tasks studied in this paper are \emph{function problems}, which return numbers as outputs. Our analysis will reference the following well-known complexity classes for function problems:  $\intro*\FP$, the set of function problems solvable in deterministic polynomial time, and $\intro*\sP$, which comprises those function problems which correspond to the number of accepting runs of some nondeterministic polynomial-time Turing machine.
\AP
We will work with ""polynomial-time Turing reductions"" between computational tasks, and we write $P_1 \intro*\polyeq P_2$ to denote that there are polynomial-time algorithms to compute $P_i$ using unit-cost calls to $P_{3-i}$, for both $i \in \set{1,2}$.

\section{Shapley Value: Definition \& Basic Results}\label{sec:shapley}
In this section, we formally define the "Shapley value", 
recall relevant existing results, and prove a 
new intractability result for computing "Shapley values" in reachability games. 

\subsection{Definition of Shapley Value}
The "Shapley value" \cite{shapley:book1952} was introduced as a means to fairly distribute wealth amongst players in %
a "cooperative game" based upon their respective contributions. A \AP""cooperative game"" is defined as a set of players 
together with a "wealth function" which assigns a numeric value to every subset of the players (called a coalition). 
Note however that for the set of "players", one can utilise any set of elements equipped with a numeric function over its subsets. 
The Shapley value then serves to transform such a numeric function over subsets into a numeric function over the elements themselves.

\AP Formally, a ""wealth function"" is a function $\intro*\scorefun : \partsof{\bse{\scorefun}} \to \lQ$ such that $\scorefun(\emptyset) = 0$, 
 where the notation $\intro*\partsof{S}$ denotes the sets of all subsets of the input set~$S$. We shall refer to the underlying set $\intro*\bse{\scorefun}$ neutrally as the ""base set""  of $\scorefun$ (rather than speaking of players).  
The ""Shapley value"" is a function %
$\AP\intro*\Sh$ that takes a "wealth function" $\scorefun$ as input\footnote{Traditionally, the Shapley value would take a cooperative game $(\bse{\scorefun},\scorefun)$ as input,  but it suffices to take $\scorefun$ as input since the "base set" $\bse{\scorefun}$ is fully determined by the domain of $\scorefun$.} and outputs a function $\Sh_{\scorefun} : \bse{\scorefun} \to \lQ$. %
It is defined as the only such function $\psi$ that satisfies the following axioms:

\begin{enumerate}[leftmargin=\widthof{(wSym)}+\labelsep]
   \item[{\crtcrossreflabel{(wSym)}[Sh:1]}] \textit{Weak Symmetry:}
      if
      $\scorefun(S\cup\{\alpha\}) = \scorefun(S\cup\{\beta\})$ for all $S\subseteq \bse{\scorefun}\setminus\{\alpha,\beta\}$, then $\psi_{\scorefun}(\alpha) = \psi_{\scorefun}(\beta)$;
   \item[{\crtcrossreflabel{(Null)}[Sh:2]}] \AP\textit{Null element:} any (so-called ""null@@player"") "base element" that does not contribute to increasing the wealth (\ie\ such that $\scorefun(S \cup \set \alpha)= \scorefun(S)$ for all $S$) must obtain 0 as contribution;
   \item[{\crtcrossreflabel{(Lin)}[Sh:4]}] \textit{Linearity:} the value of the sum of two "wealth functions" $\scorefun_1 + \scorefun_2$ over the same "base set" is just the sum of values over the separate wealth functions: $\psi_{\scorefun_1+\scorefun_2}(\alpha)=\psi_{\scorefun_1}(\alpha)+\psi_{\scorefun_2}(\alpha)$;
   \item[{\crtcrossreflabel{(Eff)}[Sh:3]}] \textit{Efficiency:} the sum
      $\sum_{\alpha\in \bse{\scorefun}} \psi_{\scorefun}(\alpha)$
      of all contributions equals the total wealth $\scorefun(\bse{\scorefun})$ of the "base set".
\end{enumerate}

\begin{remark}
The precise axioms have seen some variations over the years. In the seminal paper \cite{shapley:book1952}, Shapley presents only 3 axioms: \ref{Sh:2} is omitted and a different stronger alternative axiom for \ref{Sh:3} is used, which implies both \ref{Sh:2} and \ref{Sh:3} as stated here. 

Moreover, the original symmetry axiom is slightly stronger than \ref{Sh:1}, with the drawback of being somewhat more tedious to state. It has recently been shown, however, that the above weak symmetry is sufficient to ensure the unicity of the "Shapley value" \cite{ourpods25arxiv}.
\end{remark}

The "Shapley value" can be equivalently defined in terms of the following %
closed-form formula: %
\AP
\begin{equation}\label{formul:sh1}
	\Sh_{\scorefun}(\alpha) = \frac{1}{|\bse{\scorefun}|!}\sum_{\sigma\in \Totord(\bse{\scorefun})} \!\!\! \left(\scorefun(\sigmaleq{\alpha}) - \scorefun(\sigmaless{\alpha})\right)
\end{equation}
where $\intro*\Totord(\bse{\scorefun})$ denotes the set of all total orderings of $\bse{\scorefun}$ and $\intro*\sigmaless{\alpha}$ (resp.~$\intro*\sigmaleq{\alpha}$) the set of elements that smaller than $\alpha$ (resp.\ smaller than or equal to $\alpha$) by the order $\sigma$. 
Intuitively, Eq.~\eqref{formul:sh1} takes the average marginal contribution $\scorefun(\sigmaleq{\alpha}) - \scorefun(\sigmaless{\alpha}$) of $\alpha$, across all possible orderings $\sigma$ of $\bse{\scorefun}$. 

The following alternative, but equivalent, formula is more computationally efficient because it is a sum over subsets of $\bse{\scorefun}$, rather than orderings, %
by grouping the orderings that have the same $\sigmaless{p}$ together: 
\begin{equation}\label{formul:sh}
   \Sh_{\scorefun}(\alpha) = \sum_{S\subseteq \bse{\scorefun} \setminus \set \alpha } \!\!\!  \frac{|S|!(|\bse{\scorefun}| - |S| -1)!}{|\bse{\scorefun}|!}\left(\scorefun(S \cup \{\alpha\}) - \scorefun(S)\right)
\end{equation}

\subsection{Applying the Shapley Value to Databases}\label{ssec:drshap}
There has been significant interest lately in using the "Shapley value"
to quantify the contribution of database facts to a given query answer.
\AP
The formal setting is as follows: 
the "database" $\D$ is ""partitioned@partitioned database""
into ""endogenous"" and ""exogenous"" ""facts"", \AP$\D = \intro*{\Dn} \dcup \intro*{\Dx}$. %
Note that, in the present paper, we use $A \AP \intro*\dcup B\phantomintro*{\bigdcup}$ to denote the union $A \cup B$ of two \emph{disjoint} sets $A,B$, the symbol $\dcup$ serving to emphasize that the sets are disjoint. %
The "exogenous facts" (in $\Dx$) are treated as obvious or always present, and therefore should not be granted any responsibility. What we want is thus a \AP ""quantitative responsibility measure"" $\phi_{q,\D} : \Dn \to \lQ$ that assigns a numeric responsibility to every "endogenous fact", based on its contribution to making the input Boolean query hold %
 in the (partitioned) "database". Note that we may focus w.l.o.g.\ on responsibility measures for Boolean queries, since we can define the responsibility of a database fact to obtaining the answer $\vec{a}$ to query $q(\vec{x})$ by simply computing the responsibility of the fact w.r.t.\ the associated Boolean query $q(\vec{a})$.

\subsubsection*{Shapley-Based Responsibility Measures}
To obtain such a "measure", it suffices to define a \AP ""wealth function family"" $\intro*\STscorefun \defeq (\intro*\stscorefun_{q,\D})_{q,\D}$ that associates, with every "Boolean query" $q$ and "partitioned database" $\D = \Dn \dcup \Dx$, a "wealth function" $\stscorefun_{q,\D}$ of "base set" $\bse{\stscorefun_{q,\D}} \defeq \Dn$ that reflects the satisfaction of $q$. Intuitively, for every $S\inc \Dn$, the wealth $\stscorefun_{q,\Dx}(S)$ should give a quantitative measure of ``how much does $S\dcup\Dx$ satisfy $q$''. Once a "wealth function family" has been chosen, the "Shapley value" can be applied as a black box to define the "responsibility measure" $\phi_{q,\D} : \alpha\in \Dn \mapsto \Sh_{\stscorefun_{q,\D}}(\alpha)$.

The arguably simplest solution, which will be the focus of this paper, is to employ the so-called \AP ""drastic wealth function""
$\intro*\Dscorefun \defeq (\dscorefun_{q,\D})_{q,\D}$ defined as 
$$\intro*\dscorefun_{q,\D}: S\inc \Dn \mapsto v_{S} - v_{\subexo}$$
where $v_{S} = 1$ (resp.\ $v_{\subexo} = 1$) if $S \dcup \Dx \models q$ (resp.\ if $\Dx \models q$), and $0$ otherwise. This defines the \AP ""drastic Shapley value"" of an "endogenous fact" $\alpha \in \Dn$ as $\Sh_{\dscorefun_{q,\D}}(\alpha)$.
It should be noted that, since $\Dscorefun$ was the only "wealth function family" considered until very recently, most of the literature (including the conference version of the present article) refers to the "drastic Shapley value" simply as ``the Shapley value'' \cite{livshitsShapleyValueTuples2021}. 
However, other 
Shapley-based responsibility measures can be defined by choosing another way to associate wealth functions to queries. 
We refer readers to Section \ref{sec:discuss} for discussion and a comparison with the recently defined WSMS family of measures.

\subsubsection*{Relevance and Desirable Properties}
Several axiomatic properties have been defined and discussed in \cite[]{ourpods25,ourpods25arxiv} to justify that the "drastic Shapley value" is a ``good'' responsibility measure:

\begin{itemize}[leftmargin=\widthof{(wSym-db)}+\labelsep]
   \item[{\crtcrossreflabel{(wSym-db)}[Shdb:1]}] If $q,q'$ define the same Boolean function and $S\cup\{\alpha\}\models q \LRa S\cup\{\beta\}\models q'$ for all $S\subseteq \D\setminus\{\alpha,\beta\}$, then $\phi_{q,\D}(\alpha) = \phi_{q',\D}(\beta)$;\footnote{This axiom was simply denoted (Sym-db) in the proceedings paper \cite{ourKR25}, but the authors recommend (wSym-db) instead, where the w stands for `weak', for reasons detailed in the extended version \cite{ourKR25arxiv}.} 
      \item[{\crtcrossreflabel{(Null-db)}[Shdb:2]}] If a "fact" $\alpha\in\Dn$ is \AP""irrelevant"" (\ie\ $\alpha$ appears in no "minimal support" for $q$ in $\D$, or $\Dx\models q$) then $\phi_{q,\D}(\alpha) = 0$; otherwise $\phi_{q,\D}(\alpha) > 0$.%
   \item[{\crtcrossreflabel{(MS1)}[MS:1]}] All other things being equal, a fact that appears in smaller "minimal supports" should have higher responsibility.
   \item[{\crtcrossreflabel{(MS2)}[MS:2]}] All other things being equal, a fact that appears in more "minimal supports" should have higher responsibility.
\end{itemize}

The formal statements of \ref{MS:1} and \ref{MS:2} are somewhat technical and beyond the scope of this paper, but they can be found in \cite[§B.1]{ourpods25}. Here we shall focus on \ref{Shdb:2}, which is arguably the most critical property, since it intuitively states that the "facts" with responsibility 0 should exactly be those whose addition to a subdatabase  never changes the outcome.
The "drastic Shapley value" satisfies this axiom (along with the three others) for all "Boolean" "monotone" queries \cite{ourpods25arxiv},
which includes the "$C$-hom-closed" queries we study here.
Of course, observe that the "drastic Shapley value" is more informative
than mere "relevance", given \ref{MS:1} and \ref{MS:2}. Indeed, in the forthcoming \cref{ex:shap}, we shall exhibit a case in which all %
considered "assertions" are "relevant", but have different "drastic Shapley values".

\subsection{Drastic Shapley Value Computation in the Database Setting}
Extensive work has been directed at characterizing the complexity of computing the "drastic Shapley value". Given a class $\C$ of "queries", the ""drastic Shapley value computation"" problem on $\C$, denoted $\AP\intro*\dShapley_{\C}$, is the problem of computing, given as input a "query" $q\in \C$, a "partitioned database" $\D$ and an "endogenous fact" $\alpha\in \Dn$, the "drastic Shapley value" $\Sh_{\dscorefun_{q,\D}}(\alpha)$.
We will often consider classes consisting of a single query ($\C = \{q\}$), in which case we simply write $\dShapley_q$.

\subsubsection*{Related Problems in Probabilistic Query Evaluation}

We will exploit known connections between $\dShapley$ and probabilistic query evaluation.
\AP
A ""tuple-independent probabilistic database"" is a pair $\D = (X,\pi)$ where $X$ is a "database" and $\pi: X \to (0,1]$ is a probability assignment.
For a "Boolean query"~$q$, $\AP\intro*\Prob(\D \models q)$ is the probability of $q$ being true, where each "assertion" $\alpha$ has independent probability $\pi(\alpha)$ of being in the "database". 
The problem of computing, given a "tuple-independent probabilistic database" $\D$,  the probability $\Prob(\D \models q)$ is known as the ""probabilistic query evaluation"" problem, or $\intro*\PQE_{q}$.
We consider three restrictions of $\PQE_{q}$, by limiting the probabilities that appear in the image $\Ima(\pi)$ of the probability assignment of the  input "probabilistic database": %
\begin{itemize}
   \item $\intro*\PQEPhalf_{q}$: input $(X,\pi)$ is such that $\Ima(\pi) = \set{\nicefrac 1 2}$;
	\item $\intro*\PQEPhalfOne_q$: input $(X,\pi)$ is such that $\Ima(\pi) \inc \set{\nicefrac 1 2, 1}$;
	\item ""single proper probability query evaluation"" ($\intro*\SPPQE_{q}$): 
	input $(X,\pi)$ is such that $\Ima(\pi) \inc \set{p,1}$ for some $p \in (0,1]$.
\end{itemize}
These restricted versions of $\PQE$ can be also found in the literature under the names of their counting problem counterparts: $\PQEPhalf$ is also known as the ``model counting'' \cite{ourpods24} or ``uniform reliability'' \cite{Amarilli23} problems, and $\PQEPhalfOne$ as the ``generalized model counting'' \cite{kenigDichotomyGeneralizedModel2021} problem.

\subsubsection*{Known Results on Drastic Shapely Value Computation}

\AP
In \cite[Theorem 4.1]{livshitsShapleyValueTuples2021}, a $\FP$/$\sP$-hard dichotomy was established for ""self-join free"" "CQs" (i.e.\ not having two atoms %
with the same relation name). The dichotomy coincides with the $\FP$/$\sP$-hard dichotomy for $\PQE$ \cite[Theorem 5.2]{dalviEfficientQueryEvaluation2004}, and the tractable queries admit a syntactical characterization, known as ""hierarchical"" queries.
\AP
In fact, the $\PQE$ dichotomy extends to the more general class of "UCQs" \cite[Theorem 4.21]{dalviDichotomyProbabilisticInference2012}, where 
the queries for which $\PQE$ is tractable are known as ""safe@@q"" "UCQs" (hence, in particular "hierarchical" "CQs" are "safe@@q").
However, it is an open problem whether "UCQs" (or even "CQs" with "self-joins") also enjoy a dichotomy for $\dShapley$.
Concretely, it is unknown if $\dShapley_q$ is $\sP$-hard when $q$ is an arbitrary "unsafe@@q" "UCQ" (or "CQ").

Recent work has clarified the relation between the two dichotomies by reducing $\dShapley$ to $\PQE$ \cite[Proposition 3.1]{deutchComputingShapleyValue2022a} and reproving the hardness of $\dShapley$ \cite{karaShapleyValueModel2023} by "reduction" from the same model counting problem for \emph{Boolean functions} that had been used to show hardness of $\PQE$ for non-"hierarchical" "self-join free" "CQs" \cite{dalviEfficientQueryEvaluation2004}. 
Further, $\SPPQE$ and $\dShapley$ have been shown to be polynomial-time inter-reducible for many fragments of "hom-closed" queries \cite[Lemmas 4.1, 4.3 and 4.4]{ourpods24}, in particular for "connected@@q" queries without "constants" \cite[Corollary 4.1 (and Proposition 3.3)]{ourpods24}.
\begin{theorem}[{\cite[Corollaries 4.1 and 4.2]{ourpods24}}]\label{thm:homclosedconn-pods-dichotomy} 
	For every "connected@@q" "hom-closed" "Boolean" query $q$, $\SPPQE_{q} \polyeq \dShapley_q$; 
	further, on "graph databases", $\dShapley_q$ is in $\FP$ if $q$ is equivalent to a "safe@@q" "UCQ" and $\sP$-hard otherwise.
\end{theorem}
\AP
In the context of the previous statement, a ""graph database"" is a "database" restricted to relations of arity 1 or 2 (hence an "ABox" can be seen as a "graph database").
The result above relies crucially on the $\sP$-hardness of $\PQEPhalf$ (and hence of $\SPPQE$) for "non-FO-rewritable" (\aka\ ""unbounded"")
"hom-closed" queries on "graph databases" \cite[Theorem 1.3]{Amarilli23}, and the $\FP$/$\sP$-hard dichotomy of $\PQEPhalfOne$ for "UCQs" \cite[Theorem 2.2]{kenigDichotomyGeneralizedModel2021}.

\subsubsection*{Purely Endogenous Setting}
While the "endogenous"/"exogenous" partition is a useful mechanism, it is also quite natural to consider scenarios in which no such distinction is made
and all database facts are treated as "endogenous". 
\AP
Formally, we will write  $\AP\intro*\dnShapley_{\C}$ (or $\dnShapley_{q}$ when $\C = \{q\}$) to refer to $\dShapley_{\C}$ restricted to ""purely endogenous"" databases, \ie\ "partitioned databases" with only "endogenous" "assertions" (equivalently, having $\Dx = \emptyset$).

While \cite[§6.1]{ourpods24} defines $\dnShapley$ and provides a couple of results, the complexity of the problem remains mostly open. For instance, to the best of our knowledge, it is currently not known if $\dnShapley_q$ is hard for any single "CQ" (or even "UCQ") $q$. In fact, the same can be said for the related probabilistic problem $\PQEPhalf$.
On the other hand, we cannot either point to any query $q$ for which $\dnShapley_q$ is known to be tractable but $\dShapley_q$ isn’t.
As we shall soon see, however, the hardness of $\dnShapley$ can easily be obtained for some other queries, such as those expressing the notion of reachability, and we shall phrase our reductions for the "purely endogenous" problem whenever possible.%

\subsection{Hardness Result for Reachability}
\label{sec:reach-hardness}

This subsection shows the $\sP$-hardness of the "drastic Shapley value computation" of \AP""graph reachability"", %
which, as we shall see in \Cref{sec:reach}, implies the $\sP$-hardness of several problems in the setting of ontologies.

\AP
Consider the "Boolean" query which asks whether there is a directed path from a vertex $s$ to a vertex $t$ in a (directed) graph. We can represent a graph as a "graph database" using a single binary relation $r$, with $s,t$ "constants", in which case the $s$-$t$-reachability query can be expressed as $\intro*\qreach$ in standard regular path query (\AP""RPQ"") notation. 
As we show next, computing the "drastic Shapley value" for this simple query is already $\sP$-hard, even in the restricted case of "purely endogenous" databases.

\begin{proposition}\label{prop:graph}
	$\dnShapley_{\qreach}$ is $\sP$-hard. %
\end{proposition}
\begin{proof}
The result follows via a "reduction" akin to the one in \cite[Proposition 4.6]{livshitsShapleyValueTuples2021}, but from a different $\sP$-hard task, namely $\intro*\numSTConn$,
which is the task of, given a graph $G$ and vertices $s,t$ thereof, counting the number of subgraphs of $G$ 
which contain a path from $s$ to $t$ 
\cite[problem 11]{valiantComplexityEnumerationReliability1979}. This "reduction" follows a technique that will be used again in later proofs, which consists in producing several related variants of an instance, so that when applying \Cref{formul:sh} we obtain a system of linear combinations of the desired values. It then suffices to show that this system is invertible in order to obtain the desired values in polynomial time.

\begin{figure}[tb]\centering
\begin{tikzpicture}%
	\small
	\begin{pgfonlayer}{nodelayer}
		\node [draw, cloud, minimum height=7mm, minimum width=12mm] (0) at (-.5, 0) {$G$};
		\node [draw, circle,fill=white] (1) at (-1, 0) {$s$};
		\node [draw, circle,fill=white] (2) at (0, 0) {$t$};
		\node [draw, circle] (3) at (-2, 0) {\phantom{$s$}};
		\node [] at (3) {$s_i$};
		\node [] (4) at (-3, 0) {$\cdots$};
		\node [draw, circle] (5) at (-4, 0) {\phantom{$s$}};
		\node [] at (5) {$s_2$};
		\node [draw, circle] (6) at (-5, 0) {\phantom{$s$}};
		\node [] at (6) {$s_1$};
		\node [] (7) at (-6, 0) {$G_i$:};
	\end{pgfonlayer}
	\begin{pgfonlayer}{edgelayer}
		\draw [->, >=stealth] (6) to (5);
		\draw [->, >=stealth] (5) to (4);
		\draw [->, >=stealth] (4) to (3);
		\draw [->, >=stealth] (3) to (1);
		\draw [->, >=stealth,bend left=15,color=blue] (6.north east) to node[midway, below, sloped,color=blue] {$\mu$} (2.north west);
	\end{pgfonlayer}
\end{tikzpicture}
\caption{Illustration of the graph $G_i$.}
\label{fig:Gi}
\end{figure}

Let $G=(V,E)$ a directed graph (formally, we see $E$ as a database over the single relation $r$). 
Denote $m=|E|$ and build for every $i\in [m]$ the graph $G_i = (V_i,E_i)$ by adding a set of edges $C_i$ forming a path $s_1 \to s_2 \to \dots \to s_i \to s$ and an extra edge $\mu$ from $s_1$ to $t$ whose Shapley value we will measure, for the reachability from $s_1$ to $t$ (see \Cref{fig:Gi}). Recall \Cref{formul:sh}:
\begin{align*}
   \Sh_{\dscorefun_{\qreach[s_1,t],E_i}}(\mu) = \sum_{\X\inc E_i\setminus\{\mu\}} \gamma_\X \cdot
   (\dscorefun_{\qreach[s_1,t],E_i}(\X\cup\{\mu\}) - \dscorefun_{\qreach[s_1,t],E_i}(\X))
\end{align*}
where $\gamma_\X\defeq \frac{|\X|!(|E_i| - |\X| - 1)!}{|E_i|!}$. Note that for simplicity we write $\dscorefun_{\qreach[s_1,t],E_i}$ instead of $\dscorefun_{\qreach[s_1,t],(E_i,\emptyset)}$ since $\dnShapley_{\qreach}$ requires a purely "endogenous" input.

Here by definition $\dscorefun_{\qreach[s_1,t],E_i}(\X)=1$ if there exists a path from $s_1$ to $t$ in $(V_i,\X)$ and $0$ otherwise. Therefore, $\dscorefun_{\qreach[s_1,t],E_i}(\X\cup\{\mu\}) - \dscorefun_{\qreach[s_1,t],E_i}(\X) = 0$ if $\X$ already connects $s_1$ to $t$, \ie\ if $\X\cap E$ connects $s$ to $t$ (denoted $s\xrightarrow{\X\cap E}t$) and the other edges of $\X$ exactly form the set $C_i$. Otherwise, $\dscorefun_{\qreach[s_1,t],E_i}(\X\cup\{\mu\}) - \dscorefun_{\qreach[s_1,t],E_i}(\X) = 1$ since $\mu$ alone fulfils the connectedness condition. Hence, we can isolate the contribution of all terms such that $\dscorefun_{\qreach[s_1,t],E_i}(\X\cup\{\mu\}) - \dscorefun_{\qreach[s_1,t],E_i}(\X) = 0$:

\begin{align*}
   \Sh_{\dscorefun_{\qreach[s_1,t],E_i}}(\mu) -\sum_{\X\inc E_i\setminus\{\mu\}} \gamma_\X = - \sum_{\substack{\X\inc E_i\setminus\{\mu\}\\\X\setminus E = C_i\\s\xrightarrow{\X\cap E}t}} \gamma_\X
\end{align*}
The left-hand side of this equation can be simplified as $\Sh_{\dscorefun_{\qreach[s_1,t],E_i}}(\mu) -1$, and the right-hand side can be expressed as follows in terms of the numbers $cs_j$ of subgraphs of $G$ of size $j$ that connect $s$ to $t$, by grouping the terms in the sum by $|\X\cap E| \eqdef j$.
\[
   \Sh_{\dscorefun_{\qreach[s_1,t],E_i}}(\mu) -1 = -\sum_{j=0}^{m} \frac{(i+j)!(m-j)!}{(m+i+1)!}cs_j
\]
In other words, we obtain a linear system of equations $A \cdot \mathbf{x}=\mathbf{y}$, where 
$\mathbf{y}$ is the $m$-vector of the different values of $\Sh_{\dscorefun_{\qreach[s_1,t],E_i}}(\mu) -1$ and $\mathbf{x}$ is the $m$-vector of variables corresponding to the values $cs_j$. If the system can be solved, we obtain ---with an oracle for computing the $\Sh$-values plus polynomial-time arithmetic computations--- the values of all the $cs_j$ values, whose sum is precisely the solution to the $\numSTConn$ instance we wish to solve.

Finally, we can show that the matrix $A$ is invertible by studying its determinant: multiplying every row by $(m+i+1)!$ and dividing every column by $(m-j)!$ reduces $A$ to the matrix of general term $(i+j)!$, which is known to be invertible  \cite[proof of Theorem 1.1]{bacherDeterminantsMatricesRelated2002}.
\end{proof}

\section{Drastic Shapley Value Computation in the Ontology Setting: Definitions and First Intractability Results}\label{sec:reach}

Now that we that have seen how the "Shapley value" has been applied in the "database" setting, 
we can adapt the definitions and techniques to provide quantitative explanations in the ontology setting.

\subsection{Axiom Entailment}\label{ssec:axiomentail}
Before considering queries, we first consider an even simpler application of the "drastic Shapley value" to ontologies, in which we focus solely on the "TBox" and determine which "axioms" are most responsible for a given "TBox" entailment.
Formally, we consider a \AP""partitioned TBox"" $\T = \T_{\subendo}\dcup\T_{\subexo}$ with "endogenous" and "exogenous" elements like "partitioned databases" and a "concept inclusion" $A\ic B$ whose entailment we wish to explain. We can then define a "wealth function" $\dscorefun_{A\ic B,\T}$ whose "base set" is $\T_{\subendo}$ as follows: for every $S \inc \T_{\subendo}$, $\dscorefun_{A\ic B,\T}(S) \defeq v_{S} - v_{\subexo}$ where $v_{S} = 1$ (resp.\ $v_{\subexo} = 1$)  if $S\cup\T_{\subexo}\models A\ic B$ (resp.\ $\T_{\subexo}\models A\ic B$), and $0$ otherwise. The "wealth function" $\dscorefun_{A\ic B,\T}$ can then be fed into the "Shapley value" $\Sh$ to provide a quantitative responsibility measure for axiom entailment. 

\begin{example}\label{ex:tbox-shap}
Consider the "TBox" $\T$ depicted in the top half of \Cref{fig:ex-kb}, which we shall use as a running example.
It contains axioms about ingredients of recipes: for instance, the axiom $\alpha = \exists \mathsf{hasIngr}. \mathsf{FishBased} \ic \mathsf{FishBased}$ intuitively translates as `anything that has a fish-based ingredient is fish-based itself'.
A confused biologist might want to explain why this ontology entails the "concept inclusion" $\mathsf{Crustacean} \ic \mathsf{FishBased}$. She could decide to set the whole of $\T$ as "endogenous" since she might not have any reason to discriminate certain "axioms" from the others.
In this instance, the measure $\Sh_{\dscorefun_{\alpha%
,\T}}$%
assigns the same score of $\frac{1}{2}$ to the "axioms" $\mathsf{Crustacean}\ic \mathsf{Seafood}$ and $\mathsf{Seafood}\ic \mathsf{FishBased}$, and 0 to all the others.
\end{example}

\begin{figure}[tb]
	\small\centering
	\[\hspace{-1mm}\begin{array}{ll}
		\mathsf{FishBased}\sqcap\mathsf{MeatBased}\ic \mathsf{LandSea}&
			\mathsf{Fish}\ic \mathsf{FishBased}\\
		\exists \mathsf{hasIngr}. \mathsf{FishBased} \ic \mathsf{FishBased}&
			\mathsf{Seafood}\ic \mathsf{FishBased}\\
		\exists \mathsf{hasIngr}. \mathsf{MeatBased} \ic \mathsf{MeatBased}&
			\mathsf{Crustacean}\ic \mathsf{Seafood}\\
		\mathsf{hasSauce}\ic \mathsf{hasIngr}&
			\mathsf{Meat}\ic \mathsf{MeatBased}\\[.3em]
	  \hline
	  \multicolumn{2}{c}{\begin{tikzpicture}[]
	\coordinate (00) at (-2, 1);
	\coordinate (01) at (2, 1);
	\coordinate (02) at (-2, -0.5);
	\coordinate (03) at (2, -0.5);
	\coordinate (04) at (2, 1.55);
	\coordinate (05) at (2, 0.05);
	\begin{pgfonlayer}{nodelayer}
		\node [draw,minimum height=4ex,rounded corners] (0) at (00) {\textit{poulardeNantua}};
		\node [draw,minimum height=4ex,rounded corners] (1) at (01) {\textit{chicken}};
		\node [draw,minimum height=4ex,rounded corners] (2) at (02) {\textit{nantuaSauce}};
		\node [draw,minimum height=4ex,rounded corners] (3) at (03) {\textit{crayfish}};
		\node [] (4) at (04) {$\mathsf{Meat}$};
		\node [] (5) at (05) {$\mathsf{Crustacean}$};
	\end{pgfonlayer}
	\begin{pgfonlayer}{edgelayer}
		\draw [->, >=stealth] (0) to node[midway, above=-.1,sloped] {$\mathsf{hasIngr}$} node[midway,below,sloped] {$e_1$} (1);
		\draw [->, >=stealth] (0) to node[midway, left] {$\mathsf{hasSauce}$} node[midway,right] {$e_3$} (2);
		\draw [->, >=stealth] (0) to node[midway, above=-.1,sloped] {$\mathsf{hasIngr}$} node[midway,below,sloped] {$e_2$} (3.north west);
		\draw [->, >=stealth] (2) to node[midway, above=-.1,sloped] {$\mathsf{hasIngr}$} node[midway,below,sloped] {$e_4$} (3);
	\end{pgfonlayer}
\end{tikzpicture}}
	\end{array}\]
	\caption{An example "KB", with data and knowledge about a recipe from \protect\cite{escoffierGuideCulinaireAidememoire1903}. The arrows represent "role assertions" and labels on top of boxes (\eg\ {\sf Meat}) represent "concept assertions".}
	\label{fig:ex-kb}
\end{figure}

In %
\Cref{ex:tbox-shap}, the considered entailment 
$\mathsf{Crustacean} \ic \mathsf{FishBased}$ holds %
due to a %
chain of two "axioms" $\mathsf{Crustacean} \ic \mathsf{Seafood}$ and $\mathsf{Seafood}  \ic \mathsf{FishBased}$. As such a chain can obviously be made arbitrarily long, 
we can use axiom entailment to encode graph reachability, even if we work with the simplest possible "DL" $\AP\intro*\Lmin$ containing only "concept name" "inclusions@@concept".
In light of \Cref{prop:graph}, the computation of $\Sh_{\dscorefun_{A\ic B,\T}}$ will therefore be $\sP$-hard.

\begin{proposition}\label{prop:tbox}
The problem of computing, given a $\Lmin$ "partitioned TBox" $\T$, "concept names" $A,B$, and an "axiom" $\mu\in\T$, the value $\Sh_{\dscorefun_{A\ic B,\T}}(\mu)$ is $\sP$-hard, even if $\T$ is assumed to be purely "endogenous".
\end{proposition}
\begin{proof}
   To reduce from $\dnShapley_{\qreach}$, let $\D$ be an input database with only $r$ relations and two distinguished "constants" $s,t$. Consider the "TBox" $\T_{\D}\defeq \{A_c \ic A_d \mid r(c,d)\in \D\}$ (which consists of only "concept name" "inclusions@@concept") and the "concept inclusion" $A_s \ic A_t$. By construction, $t$ being reachable from $s$ is equivalent to $A_s \ic A_t$ being entailed,
hence the "wealth functions" $\dscorefun_{A_s\ic A_t,\T_{\D}}$ and $\dscorefun_{\qreach,{\D}}$ are isomorphic.
\end{proof}

\subsection{Query Entailment}
We next consider the application of the "drastic Shapley value" to explaining query entailment \wrt\ a "DL" KB. Formally, we extend "partitioned databases" into \AP""partitioned KBs"" $\K=(\A_{\subendo}\dcup\A_{\subexo},\T_{\subendo}\dcup\T_{\subexo})$ and the task is to compute the "drastic Shapley value" of the statements in $\A_{\subendo}\cup\T_{\subendo}$, in order to quantify their responsibility in $\K$ entailing a Boolean query $q$. %
Recall that the restriction to Boolean queries is w.l.o.g.\ since we can define the responsibility to obtaining the answer $\vec{a}$ to query $q(\vec{x})$ as the responsibility to the associated Boolean query $q(\vec{a})$.
From there, we define a "wealth function family" $(\dscorefun_{q,\K})_{q,\K}$ with "base set" $X_{\dscorefun_{q,\K}} \defeq \K_{\subendo}$ by setting, 
for every $S_a \inc \A_{\subendo}$ and $S_t \inc \T_{\subendo}$, $\dscorefun_{q,\K}(S_a \cup S_t) \defeq v_{S} - v_{\subexo}$, where $v_{S} = 1$ (resp.\ $v_{\subexo} = 1$) if $(S_a\dcup\A_{\subexo},S_t\dcup\T_{\subexo}) \models q$ (resp.\ $(\A_{\subexo},\T_{\subexo})\models q$), and 0 otherwise.
This measure can then be applied as showcased by \Cref{ex:shap}.

\begin{example}\label{ex:shap}
Now consider the full $\horndl$ "KB" defined in \Cref{fig:ex-kb},
which extends the "TBox" seen in \Cref{ex:tbox-shap} (top half) with an "ABox" (bottom half) that contains information on some ingredients and recipes.

A user of this "KB" might obtain \textit{poulardeNantua} as an answer to the query $\mathsf{LandSea}(x)$, and wonder which ingredients are the most responsible for this fact.
She can thus set everything but the "role assertions" (which specify ingredients) as "exogenous" and compute the "drastic Shapley values" for the "Boolean" query $\mathsf{LandSea}(poulardeNantua)$.
The modeling choice to set only "role assertions" as "endogenous" corresponds to considering the background knowledge provided by the "TBox" and the "concept assertions" as being external to responsibility attribution, since they are not part of recipes.\footnote{It is especially important in this scenario to exclude the "TBox" "axioms" because while the "axiom" $\mathsf{Meat}\ic \mathsf{MeatBased}$ is explicitly part of the KB, the "inclusion@@concept" $\mathsf{Crustacean}\ic \mathsf{FishBased}$ is only indirectly inferred. This difference would lower the scores of the fish-based ingredients relative to the meat-based ones if the "TBox" "axioms" were also set as  endogeneous. } 

We can compute the values via Eq.~\eqref{formul:sh1}. There are $4!=24$ possible orderings of the $4$ "endogenous" "role assertions" $\{e_1,e_2,e_3,e_4\}$. 
14 out of %
24 orderings are s.t.\ $\dscorefun_{q,\K}(\sigmaleq{e_1}) - \dscorefun_{q,\K}(\sigmaless{e_1})=1$, and similarly 6 for $e_2$, 2 for $e_3$, and 2 for $e_4$, making the respective "drastic Shapley values": \nicefrac{14}{24}, \nicefrac{6}{24}, \nicefrac{2}{24} and \nicefrac{2}{24}.

As expected, $e_1$ has the highest responsibility because it is necessary to satisfy the query, then comes $e_2$ that only needs to be combined with $e_1$ and finally $e_3$ and $e_4$ that must be used together in addition to $e_1$.
\end{example}

Of course this setting is almost a direct extension of the one considered in \Cref{ssec:axiomentail}, hence the drastic Shapley value computation will be $\sP$-hard, even for the simplest ontologies: 

\begin{proposition}\label{prop:kb1}
The problem of computing "drastic Shapley values" for Boolean "CQ"s %
over "partitioned KBs" on $\Lmin$ is $\sP$-hard. Hardness holds even 
for queries given as "ABox" "assertions" and "purely endogenous" KBs. 
\end{proposition}
\begin{proof}
   We build a "TBox" to reduce from $\dnShapley_{\qreach}$:
   let $\D$ be an input database with only $r$ relations and two distinguished "constants" $s,t$, from which we define $\T_{\D}\defeq \{A_c \ic A_d \mid r(c,d)\in \D\}$, $\A\defeq \{A_s(o)\}$ and the "ABox" "assertion" query $q \defeq A_t(o)$, where $o$ is a "constant".
If we had chosen to accept "exogenous" elements, we could have set the "ABox" as "exogenous" and the "TBox" as "endogenous": in that scenario a subset $\X\inc\T_{\D}$ is such that $(\A,\X)\models q$ iff $\X\models A_s\ic A_t$ iff $\X$ defines a subset of ${\D}$ which admits a path from $s$ to $t$. The "wealth function" $\dscorefun_{q,(\T_{\D},\A)}$ %
is therefore isomorphic to $\dscorefun_{\qreach,\D}$. %

   Since we actually also need to make the "ABox" "endogenous", we will additionally need to have the assertion $A_s(c)$ in $\X$ for $q$ to be entailed. The "wealth function" will be slightly different to the one for $\dnShapley_{\qreach}$:
   a subset $\X\inc\A\cup\T_{\D}$ is such that $\X\models q$ iff $\A\subseteq \X$ and $\X\models A_s\ic A_t$ iff $\A\subseteq \X$ and $\X\cap \T_{\D}$ defines a subset of ${\D}$ which admits a path from $s$ to $t$.
   To conclude, we use \Cref{lem:reachmodified} stated thereafter: the corresponding "wealth function" %
   for $(\T_{\D},\A)$ is isomorphic to  $\reachmodscorefun$ with $|N|=1$ on $\D\uplus N$, by mapping the single element of $\A$ to the single element of $N$ and every "axiom" $A_c \ic A_d$ in $\T_{\D}$ to the corresponding $r(c,d)\in \D$.
\end{proof}

\begin{toappendix}
\label{app:reach}
\end{toappendix}
\begin{lemmarep}\label{lem:reachmodified}
   Consider the "base set" $X \defeq E\uplus N$ consisting of the edges of a directed graph $G = (V,E)$ and a (non-empty) 
   finite set $N$ of `necessary elements'. Define the "wealth function" $\intro*\reachmodscorefun$ on every $S\inc X$ by $\reachmodscorefun(S) = 1$ if $N\subseteq S$ and the graph $(V,E\cap S)$ contains a path between $s$ and $t$, which are two distinguished vertices in $V$. Then the problem of computing the Shapley value of $\reachmodscorefun$ for any graph $G$ and distinguished pair $(s,t)$ is $\sP$-hard.
\end{lemmarep}
\begin{proofsketch}
The proof is provided in \ref{app:reach},
but only for the sake of completeness since it is almost identical to the one of \Cref{prop:graph}.
The only major difference is that the final matrix will have general terms $(i+j+|N|)!$ instead of $(i+j)!$. This poses no issue because \cite[proof of Theorem 1.1]{bacherDeterminantsMatricesRelated2002} proves that any matrix of general term $(i+j+k)!$ with $k$ a non-negative integer constant is invertible. 
\end{proofsketch}
\begin{proof}
The proof is by reduction from the $\numSTConn$ task, which is known to be $\sP$-hard  \cite[Problem~11]{valiantComplexityEnumerationReliability1979}.  Let $G=(V,E)$ a directed graph. 
Let $m=|E|$ and build for every $i\in [m]$ the graph $G_i = (V_i,E_i)$ by adding a set of edges $C_i$ forming a path $s_1 \to s_2 \to \dots \to s_i \to s$ and an extra edge $\mu$ from $s_1$ to $t$ whose Shapley value we will measure, for the reachability from $s_1$ to $t$ (see \Cref{fig:Gi} on page \pageref{fig:Gi}), with the added set $N$ of necessary elements. Recall \Cref{formul:sh}:
\begin{align*}
\Sh(E_i,\scorefun,\mu) = \sum_{\X\inc N\cup E_i\setminus\{\mu\}} \gamma_\X \cdot
(\scorefun(\X\cup\{\mu\}) - \scorefun(\X))
\end{align*}
where $\gamma_\X\defeq \frac{|\X|!(|E_i| - |\X| - 1)!}{|E_i|!}$.

Here by definition $\scorefun(\X)=1$ if $N\inc \X$ and there exists a path from $s_1$ to $t$ in $(V_i,\X)$ and $0$ otherwise. Therefore $\scorefun(\X\cup\{\mu\}) - \scorefun(\X) = 0$ in the followin two mutually exclusive cases:
\begin{enumerate}[(i)]
   \item\label{cas:reachmod1} $\X$ does not contain all necessary elements, \ie\ $N\not\inc\X$;
   \item\label{cas:reachmod2} $N\inc \X$ and $\X$ already connects $s_1$ to $t$, \ie\ if $\X\cap E$ connects $s$ to $t$ (denoted $s\xrightarrow{\X\cap E}t$) and the other elements of $\X$ exactly form the set $C_i\cup N$.
\end{enumerate}

In all other cases, $\scorefun(\X\cup\{\mu\}) - \scorefun(\X) = 1$ since $N\cup\{\mu\}$ suffices to fulfil the connectedness condition. Hence, we can isolate on the right-hand side the contribution of all terms such that $\scorefun(\X\cup\{\mu\}) - \scorefun(\X) = 0$:

\begin{equation}\label{eq:reachmod}
   \Sh(E_i,\scorefun,\mu) -\sum_{\X\inc E_i\setminus\{\mu\}} \gamma_\X = - \sum_{\substack{\X\inc N \cup E_i\setminus\{\mu\}\\ N\not\inc\X}} \gamma_\X - \sum_{\substack{\X\inc N \cup E_i\setminus\{\mu\}\\\X\setminus E = C_i \cup N\\s\xrightarrow{\X\cap E}t}} \gamma_\X
\end{equation}
The left-hand side of \Cref{eq:reachmod} can be simplified as $\Sh(E_i,\scorefun,\mu) -1$. The first sum of the right-hand side, which we denote by $\Sigma_i$, represents the case \ref{cas:reachmod1},  and can be expressed as follows:

\begin{align*}
 \Sigma_i \defeq \sum_{\substack{\X\inc N \cup E_i\setminus\{\mu\}\\ N\not\inc\X}} \gamma_\X
   &= \sum_{j=0}^{|N|+|E_i|-1} \sum_{\substack{\X\inc N \cup E_i\setminus\{\mu\}\\ N\not\inc\X \\ |\X| = j}}\frac{(i+j+|N|)!(m-j)!}{(m+i+|N|+1)!}
\end{align*}

To compute the above formula, it therefore suffices to count all subsets of $N \cup E_i$ of size $j$ that do not contain the whole of $N$. That is, the $\binom{|N|+m}{j}$ total subsets of size $j$ minus the $\binom{m}{j-|N|}$ that contain the whole of $N$. This sum $\Sigma_i$ is therefore computable in polynomial time.\medskip

Next observe that the second sum of the right-hand side on \Cref{eq:reachmod} (which represents the case \ref{cas:reachmod2}) can be expressed as follows in terms of the numbers $cs_j$ of subgraphs of $G$ of size $j$ that connect $s$ to $t$, by grouping the terms in the sum by $|\X\cap E|=j$.
\[
\Sh(E_i,\scorefun,\mu) -1 +\Sigma_i = -\sum_{j=0}^{m} \frac{(i+j+|N|)!(m-j)!}{(m+i+|N|+1)!}cs_j
\]

In other words, we obtain a linear system of equations $A \cdot \mathbf{x}=\mathbf{y}$, where 
$\mathbf{y}$ is the $m$-vector of the different values of $\Sh(E_i,\scorefun,\mu) -1+\Sigma_i$ and $\mathbf{x}$ is the $m$-vector of variables corresponding to the values $cs_j$. If the system can be solved, we obtain ---with an oracle for computing $\Sh(\cdot,\reachmodscorefun,\cdot)$ plus polynomial-time arithmetic computations--- the values of all the $cs_j$ values, whose sum is precisely the solution to the $\numSTConn$ instance we wish to solve.

Finally, we can show that the matrix $A$ is invertible by studying its determinant: multiplying every row by $(m+i+|N|+1)!$ and dividing every column by $(m-j)!$ reduces $A$ to the matrix of general term $(i+j+|N|)!$ where $|N|$ is a constant integer, which is known to be invertible  \cite[proof of Theorem 1.1]{bacherDeterminantsMatricesRelated2002}.
\end{proof}

\subsection{Exogenous Ontologies}

Since \Cref{prop:tbox,prop:kb1} show that computing "drastic Shapley values" of "axioms" is inevitably intractable,
a natural restriction to consider in search of a tractable fragment consists in treating all "TBox" "axioms" as "exogenous", like in \Cref{ex:shap}.
Looking back at this example, the entailment of $\mathsf{FishBased}(\textit{poulardeNantua})$, which eventually leads to the entailment of the desired $\mathsf{LandSea}(\textit{poulardeNantua})$, requires one of two chains of ingredients: $\textit{poulardeNantua} \xrightarrow{\mathsf{hasIngr}} \textit{crayfish}$ or $\textit{poulardeNantua} \xrightarrow{\mathsf{hasSauce}} \textit{nantuaSauce} \xrightarrow{\mathsf{hasIngr}} \textit{crayfish}$. Intuitively, we find reachability again, this time expressed within the data itself.
This behaviour will naturally make  "drastic Shapley value" computation $\sP$-hard, and it is caused by the "axiom" $\exists \mathsf{hasIngr}. \mathsf{FishBased} \ic \mathsf{FishBased}$, 
which can be found in any "DL" %
at least as expressive as $\EL$.

Recall that, according to our definitions, $(\L,\IQ)$ where $\L$ is a "DL" denotes the class of all "OMQ"s of the form $(\T,q)$, where $\T$ in an $\L$ "TBox" %
and $q$ is a "CQ" of the form $A(c)$, for $A\in\cnames$ and $c\in\inames$. In particular, $(\L,\IQ)$ is a class of "queries" and $\dShapley_{(\L,\IQ)}$ denotes the problem of computing "drastic Shapley values" for any such "OMQ" over any "partitioned ABox" ($\T$ is implicitly treated as purely "exogenous").

\begin{proposition}\label{prop:kb2}
   $\dnShapley_{(\EL,\IQ)}$ is $\sP$-hard, even when restricted to ontologies with only one axiom.
\end{proposition}
\begin{proof}
   Let $\D$ be an input database with only $r$ relations and two distinguished "constants" $s,t$, from which we define the "ABox" $\A\defeq\{A(c_t)\}\cup\{r(c_x,c_y)\mid (x,y)\in E\}$.
   Then define the "instance query" $q \defeq A(c_s)$ and %
   $\EL$ "TBox" $\T\defeq\{\exists r. A \ic A\}$.
   A subset $\X\inc\A$ is such that $(\X,\T)\models q$ iff it contains the fact $A(c_t)$ and the rest defines a subset of $\D$ where there is a path from $s$ to $t$.
   If we take $N$ to be a singleton that corresponds to the "fact" $A(c_t)$, the resulting "wealth function" will be precisely isomorphic to $\reachmodscorefun$ from \Cref{lem:reachmodified} on $\D\uplus N$, which yields the desired hardness result.
\end{proof}

In fact, we can further establish $\sP$-hardness for any DL that allows for conjunction of concept names, which notably includes $\dlliteh$. This requires adopting a different proof strategy (not relying upon reachability) and giving up the restriction to bounded-size TBoxes. %

\begin{proposition}\label{prop:kb3}
   Let $\L_{\sqcap}$ be the "DL" in which TBox axioms take the form of concept inclusions built solely from concept names and conjunction ($\sqcap$). 
   Then $\dnShapley_{(\L_{\sqcap},\IQ)}$ is $\sP$-hard.\end{proposition}
\begin{proof}
  We reduce from the %
  problem of counting the number of vertex covers (\ie\ sets of vertices that touch all edges) in a bipartite graph. 
  This problem is known to be $\sP$-hard \cite[Problem 1]{provanComplexityCountingCuts1983}.
Given a (bipartite) graph $G=(V,E)$, we define the instance:
\begin{align*}
   \T_{G} \defeq   &
   \{A_u  \ic B_{(u,v)}, A_v  \ic B_{(u,v)} \mid (u,v) \in E\} \cup \{\bigsqcap_{e\in E} B_e \ic C\} \\
   \A_G \defeq & \left\{A_u(c) \mid u\in V\right\} \\
   q_G \defeq & \, C(c)
\end{align*}
It is easy to see that a subset  $\X \inc \A_G$ satisfies the "OMQ" $(\T_G,q_G)$ iff $\{u \mid A_u(c) \in \X\}$ is a vertex cover of $G$.
From there we can open the playbook of the proof of \Cref{prop:graph} and build variants $(\A_i,\T_i)$ of $(\A_G,\T_G)$ by replacing the axiom $\sqcap_{e\in E} B_e \ic C$ with $\bigsqcap_{e\in E} B_e \sqcap \bigsqcap_{k=1}^i D_k \ic C$ and adding to $\A_G$ the assertions $D_k(c)$ for every $k\in [i]$ and $\mu \defeq C(c)$. These modifications have the same effect as the ones used for \Cref{prop:graph} and depicted in \Cref{fig:Gi}
in the following sense: a subset $\X \inc \A_i$ satisfies $(\T_i,q_i)$ if $\mu \in \X$ or if $\X = \{D_k(c) \mid k\in [i]\}\cup \X'$ where $\X'$ is a satisfying subset of the original construction. We therefore obtain the desired reduction using the exact same reasoning. 
\end{proof}

Interestingly, we can show that the Shapley value computation problem $\dShapley_{(\L,\IQ)}$ %
 is tractable 
if we consider prominent DL-Lite dialects which do not allow for concept conjunction. 
We shall phrase the next result for $\dlliter$, but it holds equally well for other DL-Lite dialects which enjoy tractable instance checking and which satisfy the following \emph{singleton support property}: every minimal support of an "IQ" consists of a single assertion. These conditions are satisfied in particular by
$\mathsf{DL\text{-}Lite}_\mathcal{F}$ and $\mathsf{DL\text{-}Lite}_\mathcal{A}$ \cite{DBLP:journals/jar/CalvaneseGLLR07}, as is apparent when examining the shape of the rewritings of "IQ"s in these logics. %

\begin{proposition}\label{prop:dllite-atomic}
   $\dShapley_{(\dlliter,\IQ)} \in \FP$
\end{proposition}
\begin{proof}%
Suppose %
that our input is the "IQ" $q=A(c)$, $\dlliter$ TBox $\T$, and ABox $\A$ partitioned into $\A_{\subendo}\dcup\A_{\subexo}$. 
We consider each assertion $\alpha \in \A$ in turn, and we test whether $\{\alpha\} \models (\T,A(c))$, 
letting $\{\alpha_1, \ldots, \alpha_m\} \subseteq \A$ be the set of all assertions for which the test succeeds. 
Due to the singleton support property of \dlliter, 
we know that for every $\A' \subseteq \A$, we have $\A' \models (\T,A(c))$ iff $\{\alpha_1, \ldots, \alpha_m\} \cap \A' \neq \emptyset$ iff $\A' \models \alpha_1 \vee \dotsb \vee  \alpha_m$. 

We can thus completely disregard $\T$ (since it is exogenous) and 
compute the drastic Shapley values for $q^*=\alpha_1 \vee \dotsb \vee  \alpha_m$ and the partitioned ABox $\A_{\subendo}\dcup\A_{\subexo}$. 
We observe that if $\A_{\subexo}$ contains some $\alpha_i$, then all assertions in $\A_{\subendo}$ have value zero. Otherwise, only the assertions $\alpha_i$ will have a non-zero drastic Shapley value, and their values can be easily computed using Eq.~\eqref{formul:sh} and the observation that there are precisely ${|\A_{\subendo}|-m \choose k}$ subsets $S\subseteq \A_{\subendo} \setminus  \{\alpha_i\}$ of size $k$ such that $\scorefun_{q^*}(S \cup \{\alpha_i\}) - \scorefun_{q^*}(S)=1$. 
\end{proof}

Determining for which classes of \emph{non-atomic} "CQ"s the previous proposition holds w.r.t.\ DL-Lite ontologies is challenging, 
as it would require us to first establish a full complexity characterization for plain "CQ"s (without an ontology, on a binary signature), which remains an open question. 

\subsection{Approximation and Relevance}
\label{ssex:approximation}

In view of the hardness results of \Cref{prop:tbox,prop:kb1,prop:kb2,prop:kb3}, an alternative would be to give up on the precise "drastic Shapley value" and instead find an approximation, or at the very least distinguish between elements having a zero or non-zero "drastic Shapley value", because, as discussed at the end of \Cref{ssec:drshap}, this would mean identifying the "relevant" facts.

Previous work on approximating "drastic Shapley values" in the context of "databases" \cite{livshitsShapleyValueTuples2021,khalilComplexityShapleyValue2023} considered so-called Fully Polynomial Randomized Approximation Schemes (FPRAS), and more concretely the following two variants.
\AP An ""additive FPRAS"" for a function $f$ is a randomized algorithm $A(x,\eps,\delta)$ that takes an instance $x$ for $f$ as well as two parameters $\eps,\delta \in (0;1)$, whose runtime is polynomial in $x, \nicefrac{1}{\eps}, \log(\nicefrac{1}{\delta})$, and such that:
\[\Prob\left[  f(x)-\eps \le A(x,\eps,\delta) \le f(x) + \eps  \right] \ge 1-\delta\]
\AP A ""multiplicative FPRAS"" (often known as simply `FPRAS') is of the same form but must instead satisfy:
\[\Prob\left[  \frac{f(x)}{1+\eps} \le A(x,\eps,\delta) \le (1+\eps)f(x) \right] \ge 1-\delta\]

The former variant is fairly easy to obtain, as the following shows.

\begin{proposition}[Implicitly found in \protect{\cite[§4.3]{livshitsShapleyValueTuples2021}}]\label{prop:addfpras}
   Let $\scorefun$ be a "Boolean" and "monotone" "wealth function" that can be computed in polynomial time. Then there exists an "additive FPRAS" for $\Sh_{\scorefun}$.
\end{proposition}
\begin{proof}
   We approximate Eq.~\eqref{formul:sh1} by the average value of $\scorefun(\sigmaleq{\alpha}) - \scorefun(\sigmaless{\alpha})$ over $\bigO(\frac{\log(\nicefrac{1}{\delta})}{\eps^2})$ orderings $\sigma$ drawn uniformly at random.
   The result is an "additive FPRAS" by the Chernoff-Hoeffding bound \cite[Theorem 1 (2.3)]{hoeffdingProbabilityInequalitiesSums1963}.
\end{proof}

However, "additive FPRASes" have the key weakness of not allowing us to determine which facts are relevant, which has been argued to be the most fundamental requirement for a responsibility measure \cite[(Null-db) axiom]{ourpods25}. For this reason, "multiplicative FPRASes" are far more interesting, as they do decide relevance. Sadly, in the case of $\dnShapley_{\qreach}$ it has been shown in \cite[Theorem 5.1]{khalilComplexityShapleyValue2023} that no "multiplicative FPRAS" can be found unless\footnote{The conference proceedings version \cite{BienvenuFL24} of the present article says `$\mathsf{BPP} \inc \NP$' ---it is a typo.} $\NP\inc\mathsf{BPP}$, because merely deciding "relevance" is $\NP$-hard.

\begin{proposition}\label{prop:hardapprox}
   The relevance problems (that is, deciding if the "drastic Shapley value" is zero or not instead of giving the precise value) associated with the problems of \Cref{prop:graph,prop:tbox,prop:kb1,prop:kb2}
   are $\NP$-hard. As a consequence, no "multiplicative FPRAS" can be found unless $\NP\inc\mathsf{BPP}$.
   \end{proposition}
\begin{proof}
   Since all hardness proofs are derived from the hardness of $\dnShapley_{\qreach}$, the elements that represent graph edges are relevant iff their associated edge is part of a simple path from $s$ to $t$. The latter problem of deciding if a given edge lies on a simple path from $s$ to $t$  is known to be $\NP$-complete by a standard reduction from the $\NP$-complete subgraph homeomorphism problem \cite[Theorem 2]{fortuneDirectedSubgraphHomeomorphism1980}.%
   \footnote{It can also be seen as a direct consequence of the fact that checking whether there is a simple path with label in $r^*$ between two vertices of a "graph database" is $\NP$-hard by \cite[Theorem 2]{DBLP:journals/siamcomp/MendelzonW95} or, equivalently, whether the "RPQ" $\qreach$ holds in a "graph database" under the so-called `simple-path semantics'.}
\end{proof}

\section{A Dichotomy for OMQs in $\mathcal{ELHI_{\bot}}$}\label{sec:rw}
\label{sec:dichotomy-OMQ}
The results of \Cref{sec:reach} show that allowing "TBoxes" to be part of the input makes the "drastic Shapley value computation" problems $\sP$-hard even when restricted to   ontology languages of limited expressivity and with all TBox axioms set as exogenous. %
This suggests the interest of %
analyzing the complexity of "drastic Shapley value computation" 
at the level of individual ontology-mediated queries, with only the data ("ABox") as input.
Such a non-uniform approach to complexity analysis has previously been undertaken for several OMQA settings \cite{DBLP:conf/aaai/HernichLW17,DBLP:journals/lmcs/LutzSW19,DBLP:journals/ai/LutzS22}, 
in particular in the context of probabilistic OMQA \cite{jungOntologyBasedAccessProbabilistic2012}. 
Moreover, this perspective aligns nicely with the formulation of "drastic Shapley value computation" for database queries
and shall allow us to transfer results from the database setting. %

In the present%
, we will be interested in applying the existing dichotomy result for database queries (recalled in Theorem \ref{thm:homclosedconn-pods-dichotomy}) in a ``black-box'' fashion to obtain results for OMQs. To do so, we need  to
 identify a class of "OMQs" that is "hom-closed" and "connected@@q". 
This is the purpose of the following lemma, which shows that the addition of an $\horndl$ ontology preserves the "connectedness@@q" and "hom-closed" properties of queries.

\AP

\begin{lemma}\label{lem:cnx}
	Let $q$ be a "connected@@q" $\aC$-"hom-closed" query and $\T$ an $\horndl$ "ontology". Then the "OMQ" $Q\defeq(\T,q)$ is a "connected@@q" $\aC$-"hom-closed" query.
\end{lemma}

\begin{proof} 
Consider $q$ and $\T$ as in the lemma statement. To show that the OMQ $Q\defeq(\T,q)$ is "connected@@q", 
we need to show that every "minimal support" for $Q$ is "connected@@q".
Let us thus take some "minimal support" $\A$ of $Q=(\T,q)$. 

We first consider the case where $\A$ is "$\T$-consistent". 
The "knowledge base" $(\A,\T)$ thus 
admits a "canonical model" $\I_{\A,\T}$ relative to $\T$, constructed in the manner detailed in Section \ref{canmoddef}.  %
Due to the properties of the canonical model, recalled in Theorem \ref{canmodprops}, and the fact that $(\A,\T)\models q$ (since $\A$ is a "minimal support" of $Q$)
and $q$ is $\aC$-"hom-closed", it follows that $\I_{\A,\T}\models q$. 
In particular,  this means that the associated extended database
$\D_{\I_{\A,\T}}$ must contain a "minimal support" $S$  of $q$. Moreover, 
since $q$ is "connected@@q",  $S$ must be contained in a single "connected component@@q" of $\D_{\I_{\A,\T}}$. 
However, it can be easily seen from the definition of $\I_{\A,\T}$
that every "connected component@@db" of $\D_{\I_{\A,\T}}$ has the form $\D_{\I_{\A^*,\T}}$ for some "connected component@@db" $\A^*$ of $\A$. 
It follows that $S \inc \D_{\I_{\A^*,\T}}$ for some "connected component@@db" $\A^*$ of $\A$. 
Since $\D_{\I_{\A^*,\T}}$ contains $S$, and $S$ is a "minimal support" for $q$, we have $\I_{\A^*,\T}\models q$. Thus, $(\A^*,\T)\models q$, or equivalently, $\A^*\omqsat Q$. 
As $\A$ was assumed to be a "minimal support" of $Q$, we necessarily have $\A^* = \A$ and thus $\A$ is "connected@@db".

Now suppose instead that the considered "minimal support" $\A$ of $Q=(\T,q)$
is not "$\T$-consistent", i.e.\ $\A$ is a minimal $\T$-inconsistent subset. 
Without loss of generality, we may assume that 
$\T$ is in "$\bot$-normal form". 
Now let $\T'$
be the TBox obtained from $\T$ by replacing each $\bot$ by $A_\bot$, where
$A_\bot$ is a fresh concept name. By Lemma \ref{incons-bot},
 we know that for any ABox $\B$ that 
doesn't mention $A_\bot$, the following are equivalent: (i) $(\B,\T)$ is inconsistent,
and (ii) $(\B, \T') \models \exists x. A_\bot(x)$. It follows that $\A$ must also be a 
"minimal support" of the OMQ $(\T', \exists x. A_\bot(x))$. Moreover, $\A$ is 
$\T'$-consistent since $\T'$ does not contain $\bot$, and the atomic CQ $\exists x. A_\bot(x)$ is trivially "connected@@q". 
Thus, we can apply the same reasoning as in the preceding paragraph to show that $\A$ is "connected@@q".

It remains to show that $Q$ is $\aC$-"hom-closed". To this end, consider "ABoxes" $\A, \B$ such that 
$\A \models Q$ and $\A \homto[C-] \B$, and suppose for a contradiction that $\B \not \models Q$. 
Note that this means in particular that $\B$ is $\T$-consistent. 
First suppose that $\A$ is also $\T$-consistent. 
Then by Lemma \ref{aboxhomcanmod},
$\A \homto[C-] \B$ implies $\D_{\I_{\A,\T}} \homto[C-] \D_{\I_{\B,\T}}$. 
Since $\A \models Q$, we have $\D_{\I_{\A,\T}}\models q$. 
As $q$ is $\aC$-"hom-closed", 
this yields $\D_{\I_{\B,\T}}\models q$, hence $\B \models Q$ as required. 
Now suppose instead that $\A$ is $\T$-inconsistent, and 
let $\T'$ be defined as in the previous paragraph (replacing $\bot$ by fresh concept $A_\bot$). 
Then by Lemma \ref{incons-bot}, 
$(\A, \T') \models \exists x. A_\bot(x)$, so $\D_{\I_{\A,\T'}} \models \exists x. A_\bot(x)$. 
Again applying Lemma \ref{aboxhomcanmod}, we can use $\A \homto[C-] \B$
to infer $\D_{\I_{\A,\T'}} \homto[C-] \D_{\I_{\B,\T'}}$. It follows that $\I_{\B,\T'} \models \exists x. A_\bot(x)$,
so $(\B,\T') \models \exists x. A_\bot(x)$. With a further application of Lemma \ref{incons-bot},
we can deduce that $\B$ is $\T$-inconsistent, contrary to our assumption. 
This concludes our proof that $Q$ is $\aC$-"hom-closed". 
\end{proof}

With \Cref{lem:cnx} at hand, we can now directly apply \Cref{thm:homclosedconn-pods-dichotomy} (originally shown in \cite{ourpods24})
to the class of
 "OMQ"s whose base query is "constant-free" and "connected@@q" and whose ontology is expressed in $\horndl$. This yields an $\FP$/$\sP$-hard complexity dichotomy for the drastic Shapley value computation 
 task via the tight correspondence with probabilistic query evaluation.

\begin{theorem}\label{th:cnx-hom-closed}
   For every "connected@@q" ("constant-free") "hom-closed" query $q$ and $\horndl$ ontology $\T$, we have $\dShapley_{(\T,q)} \polyeq \SPPQE_{(\T,q)}$. Further, the problem is in $\FP$  if the "OMQ" ($\T,q$) can be "rewritten" into a "safe@@q" "UCQ" and %
	$\sP$-hard otherwise.
\end{theorem}

\begin{proof}
	For OMQs $(\T,q)$ from the considered class, %
	\Cref{lem:cnx} states that $(\T,q)$ is a "connected@@q" ($\emptyset$-)"hom-closed" query,  hence \Cref{thm:homclosedconn-pods-dichotomy} gives the desired results.
\end{proof}

While the preceding dichotomy result shows that $\dShapley_{(\T,q)}$ is either in $\FP$ or $\sP$-hard, 
one naturally would like to be able to check which complexity applies to a given OMQ. 
We can show that this is possible for some common classes of OMQs. Indeed, 
it follows from results on probabilistic query evaluation that it is decidable  %
whether or not a given "UCQ" is "safe@@q"
 \cite[implicit]{dalviDichotomyProbabilisticInference2012}.
Consequently,
the dichotomy given by \Cref{th:cnx-hom-closed} is effective whenever "first-order rewritability" is decidable for the considered class of "OMQ"s and a "first-order rewriting" can always be effectively computed when one exists. This is in particular true for OMQs formulated in $(\horndl,\CQ)$ \cite[Theorem 5]{bienvenuFirstOrderrewritabilityContainment2016a}.

Furthermore, \cite[Theorem 5]{jungOntologyBasedAccessProbabilistic2012} gives a syntactic characterisation of which "constant-free" "connected@@q" "OMQ"s in $(\dllitec,\CQ)$ are equivalent to "safe@@q" "UCQ"s. By \Cref{th:cnx-hom-closed}, it also characterizes "constant-free" "connected@@q" "OMQ"s $Q\in(\dllitec,\CQ)$ for which $\dShapley_{Q}\in\FP$.

\section{Strengthening the $\mathsf{\#P}$-Hardness Result}\label{sec:non-rw}
\label{sec:strengthening}
\begin{toappendix}
\label{app:non-rw}
\end{toappendix}

The dichotomy of \Cref{th:cnx-hom-closed} is limited in two respects:
\begin{itemize}
\item The result only covers the class of "connected@@q" "constant-free" queries,
because this is the scope of \Cref{thm:homclosedconn-pods-dichotomy}. However, if we wish to compute Shapley values of answers to non-Boolean queries,
we must first instantiate the free variables with the constants from the answer tuple, and then consider the resulting query with constants. 
\item When an "OMQ" is seen as an abstract query, the distinction between "consistent" and "inconsistent" "ABoxes" is lost. 
However, one might be interested in explaining answers to a query only over "consistent" "ABoxes".
\end{itemize}
When considering "OMQ"s in $(\horndl,\UCQ)$, both of these points can be improved upon by studying the properties of such "OMQ"s that are "non-FO-rewritable".

\begin{theoremrep}\label{th:non-rw}
Let $q$ be a "UCQ" and $\T$ a $\horndl$ "ontology". If the "OMQ" $Q\defeq\withT{q}$ is "non-FO-rewritable" "w.r.t. consistent ABoxes", then $\dShapley_{Q}$ and $\PQEPhalfOne_{Q}$ on "consistent" "ABoxes" are both $\sP$-hard.
\end{theoremrep}
\begin{proof}
Once we have proven the necessary claims, we can prove the main theorem
by building the same "reduction" from $\numBipIndep$ that was used to prove the $\sP$-hardness of $\dShapley_{q_{\mathsf{ARB}}}$ for the "CQ" $q_{\mathsf{ARB}}\defeq A(x),R(x,y),B(y)$ \cite[Proposition 4.6]{livshitsShapleyValueTuples2021}. 
Recall that we are trying to compute the number of independent sets in some input bipartite graph $G=(X,Y,E)$. For convenience, we also denote by $V\defeq X\cup Y$ the set of all vertices in $G$.

Since the possible number of "splittable" "interfaces" is bounded as a consequence of \Cref{clm:splittable}, we can choose a large enough value for $k$ so that $\P$ contains an "unsplittable" interface $(a_\chi,a_{\chi+1})$. From this we can build an $\AG$ which, by \Cref{clm:unsplittable}, is such that the subsets $\X\inc\AGn$ \st\ $\X\dcup\AGx\nvDash Q$ are in a size-preserving bijection with the independent sets of $G$. Also note that $\AG$ can be $\aC$-"homomorphically" mapped back to a "consistent" "minimal support" for $Q$. Since we only consider $\aC$-"hom-closed" "inconsistency" conditions, $\AG$ and all of its subsets are "consistent", hence the "reduction" is indeed limited to "consistent" data.\medskip

At this stage we have established that the number of independent subsets of $G$ equals the complement of the number of $\X\inc\AGn$ \st\ $\X\dcup\AGx\nvDash Q$. This number is $2^{|\AGn|}$ times the probability that $\AG\models q$ if the "endogenous" (resp. "exogenous") "assertions" all have probability \nicefrac{1}{2} (resp. 1), which proves the $\sP$-hardness of $\PQEPhalfOne_{Q}$, as desired. In order to conclude the proof, we now show the hardness of $\dShapley_{Q}$ as well. The remainder of the proof is formally identical to the arguments that can be found in the proof of \cite[Proposition 4.6]{livshitsShapleyValueTuples2021}, and it is only presented here for sake of completeness.\medskip

The general idea is the same as for the proof of \Cref{prop:graph}: since some carefully chosen "drastic Shapley values" can be expressed as a linear combination of, in this instance, the numbers $\mathsf{IS}(G,k)$ of independent sets of a given size $k$ in $G$, we can build a family of instances to obtain an invertible linear system from which the $\mathsf{IS}(G,k)$, whose sum is the value we desire, can be obtained in polynomial time. In detail, we build a family $(G_i)_{i\in\ldb0,|V|+1\rdb}$ of graphs by adding vertices and edges to $G$, as depicted in \Cref{fig:reduction-lbkm}. Formally we add a vertex $x_0$ to $X$, and $i$ vertices $y_1,\dots,y_i$ to $Y$. Then we add the edges $\{(x_0,y)\mid y\in Y\}$ if $i=0$ and the edges $\{(x_0,y_k) \mid k\in [i]\}$ otherwise.

\begin{figure}[tb]
\centering
\begin{tikzpicture}[scale=.6]
\coordinate (00) at (-8, 1);
\coordinate (01) at (-8, 0);
\coordinate (02) at (-8, -1);
\coordinate (03) at (-6, -1);
\coordinate (04) at (-6, 0);
\coordinate (05) at (-6, 1);
\coordinate (06) at (-7, 1.5);
\coordinate (07) at (-3, 1);
\coordinate (08) at (-3, 0);
\coordinate (09) at (-3, -1);
\coordinate (010) at (-1, -1);
\coordinate (011) at (-1, 0);
\coordinate (012) at (-1, 1);
\coordinate (013) at (-2, 1.5);
\coordinate (014) at (-3, -2);
\coordinate (015) at (1, 1);
\coordinate (016) at (1, 0);
\coordinate (017) at (1, -1);
\coordinate (018) at (3, -1);
\coordinate (019) at (3, 0);
\coordinate (020) at (3, 1);
\coordinate (021) at (2, 1.5);
\coordinate (022) at (1, -2);
\coordinate (023) at (3, -2);
\coordinate (024) at (4.5, 0);
\coordinate (025) at (6, 1);
\coordinate (026) at (6, 0);
\coordinate (027) at (6, -1);
\coordinate (028) at (8, -1);
\coordinate (029) at (8, 0);
\coordinate (030) at (8, 1);
\coordinate (031) at (7, 1.5);
\coordinate (032) at (6, -2);
\coordinate (033) at (8, -2);
\coordinate (034) at (8, -3);
\coordinate (035) at (8, -2.5);
\begin{pgfonlayer}{nodelayer}
	\node [draw, circle] (0) at (00) {};
	\node [draw, circle] (1) at (01) {};
	\node [draw, circle] (2) at (02) {};
	\node [draw, circle] (3) at (03) {};
	\node [draw, circle] (4) at (04) {};
	\node [draw, circle] (5) at (05) {};
	\node [] (6) at (06) {G};
	\node [draw, circle] (7) at (07) {};
	\node [draw, circle] (8) at (08) {};
	\node [draw, circle] (9) at (09) {};
	\node [draw, circle] (10) at (010) {};
	\node [draw, circle] (11) at (011) {};
	\node [draw, circle] (12) at (012) {};
	\node [] (13) at (013) {$G_0$};
	\node [draw, circle, thick, fill=lightgray] (14) at (014) {};
	\node at (014) {\scriptsize $x_0$};
	\node [draw, circle] (15) at (015) {};
	\node [draw, circle] (16) at (016) {};
	\node [draw, circle] (17) at (017) {};
	\node [draw, circle] (18) at (018) {};
	\node [draw, circle] (19) at (019) {};
	\node [draw, circle] (20) at (020) {};
	\node [] (21) at (021) {$G_1$};
	\node [draw, circle, thick, fill=lightgray] (22) at (022) {};
	\node at (022) {\scriptsize $x_0$};
	\node [draw, circle, thick, fill=lightgray] (23) at (023) {};
	\node at (023) {\scriptsize $y_1$};
	\node [] (24) at (024) {\dots};
	\node [draw, circle] (25) at (025) {};
	\node [draw, circle] (26) at (026) {};
	\node [draw, circle] (27) at (027) {};
	\node [draw, circle] (28) at (028) {};
	\node [draw, circle] (29) at (029) {};
	\node [draw, circle] (30) at (030) {};
	\node [] (31) at (031) {$G_j$};
	\node [draw, circle, thick, fill=lightgray] (32) at (032) {};
	\node at (032) {\scriptsize $x_0$};
	\node [draw, circle, thick, fill=lightgray] (33) at (033) {};
	\node at (033) {\scriptsize $y_1$};
	\node [draw, circle, thick, fill=lightgray] (34) at (034) {};
	\node at (034) {\scriptsize $y_j$};
	\node [] (35) at (035) {\dots};
\end{pgfonlayer}
\begin{pgfonlayer}{edgelayer}
	\draw (0) to (5);
	\draw (1) to (3);
	\draw (2) to (4);
	\draw (7) to (12);
	\draw (8) to (10);
	\draw (9) to (11);
	\draw (15) to (20);
	\draw (16) to (18);
	\draw (17) to (19);
	\draw (14) to (12);
	\draw (14) to (11);
	\draw (14) to (10);
	\draw (22) to (23);
	\draw (25) to (30);
	\draw (26) to (28);
	\draw (27) to (29);
	\draw (32) to (33);
	\draw (32) to (34);
\end{pgfonlayer}
\end{tikzpicture}
\caption[Graph variants for the reduction]{Variants of the graph $G$ for the "reduction". Adapted from \cite[Figure 4]{livshitsShapleyValueTuples2021}.}
\label{fig:reduction-lbkm}
\end{figure}

Recall that there exists a bijection $\etardx$ between the "endogenous facts" of $\AG$ and the vertices in $G$. We shall pay special attention to the "fact" $\mu \defeq \etardx^{-1}(x_0)$.

We first focus on the "drastic Shapley value" of $\mu$ in $\AG[G_0]$ in order to define and compute two values we will need later. For this we consider the orderings $\sigma$ such that the arrival of $\mu$ doesn’t change the score function. This happens in the following two mutually exclusive cases:
\begin{enumerate}[(1)]
\item\label{cas1.1} $\etardx(\sigmaless{\mu})\cap Y = \emptyset$
\item\label{cas1.2} $\exists x,y\in \etardx(\sigmaless{\mu}).\; (x,y)\in E$
\end{enumerate}
The first case is when there is no $Y$ vertex to match $x_0$, and the second is when there already is a matching between two vertices meaning those added by $x_0$ are redundant. The number of orderings that satisfy \ref{cas1.1} is:%
\[P_0^1\defeq\frac{(|V|+1)!}{|Y|+1}\]

This is because every element of $\etardx^{-1}(Y)\cup\{\mu\}$ has the same $\nicefrac{1}{|V+1|}$ probability of arriving first among these, and the total number of orderings is $(|V|+1)!$.
To compute the number of orderings that satisfy \ref{cas1.2}, we denote by $\mathsf{NIS}(G,j)$ the set of all subsets of $V$ of size $j$ that are \textit{not} independent in $G$. We can then express the number of orderings that satisfy \ref{cas1.2} as:%
\[P_0^2\defeq\sum_{j=2}^{|V|} |\mathsf{NIS}(G,j)|\cdot j! \cdot (|V|-j)!\]

This is because, if $|\sigmaless{\mu}|=j$, we must have $\etardx(\sigmaless{\mu})\in\mathsf{NIS}(G,j)$, then we can have any ordering of the $j$ elements of $\sigmaless{\mu}$ and of the $|V|-j$ elements of $\sigma_{>\mu}$ (\ie\ the elements that appear strictly after $\mu$ in $\sigma$). We can now compute $P_0^2$ from $P_0^1$ and the drastic Shapley oracle using Eq.~\eqref{formul:sh1}:%
\[\mathrm{Sh}_{Q}(\AG[G_0],\mu) = 1-\frac{P_0^1+P_0^2}{(|V|+1)!}\]
\[P_0^2 = (1-\mathrm{Sh}_{Q}(\AG[G_0],\mu))\cdot(n+1)! - P_0^1\]

We similarly express the "drastic Shapley value" of $\mu$ in $\AG[G_i]$. In this context, the arrival of $\mu$ doesn’t change the score function in the following two mutually exclusive cases:
\begin{enumerate}[(1),resume]
   \item\label{cas2.1} $\exists x,y\in \etardx(\sigmaless{\mu}).\; (x,y)\in E$
\item\label{cas2.2} $\forall x,y\in \etardx(\sigmaless{\mu}).\; (x,y)\notin E ~~\land~~\etardx(\sigmaless{\mu})\cap \{y_1,\dots,y_i\} = \emptyset$
\end{enumerate}

Case \ref{cas2.1} is when there already is a matching that makes the arrival of $x_0$ redundant, and case \ref{cas2.2} is when, although there isn’t any matching yet, there is no $y_j$ to match $x_0$ with.

To choose an ordering that satisfies \ref{cas2.1}, we can first choose the order of the elements in $\etardx^{-1}(V)\cup\{\mu\}$, which are the only ones that matter in satisfying the property. These must satisfy \ref{cas1.2}, and there also are $P_0^2$ of them. Then we need to choose a way to place the various $\etardx^{-1}(y_j)$ in any position, and there are $m_i\defeq \binom{|V|+i+1}{i}\cdot i!$ ways of doing it: $\binom{|V|+i+1}{i}$ ways of choosing where to place the $i$ new elements, and $i!$ possible orderings. The number of orderings that satisfy \ref{cas2.1} is therefore $m_i\cdot P_0^2$, which we can compute in polynomial time with the help of a drastic Shapley oracle.

Next we count the number $P_i$ of orderings that satisfy \ref{cas2.2}. Since $\etardx(\sigmaless{\mu})$ shouldn’t contain any of the $y_i$ we added, it must be an independent set of $G$, which leaves $|\mathsf{IS}(G,j)|$ choices for each size $j$, then $j!$ orderings of $\sigmaless{\mu}$ and $(|V|+i-j)!$ orderings of $\sigma_{>\mu}$. Overall we get:
\[P_i = \sum_{j=0}^{|V|}|\mathsf{IS}(G,j)|\cdot j!(|V|+i-j)!\]

We can also express the "drastic Shapley value" of $\mu$ in terms of $P_0^2$ and $P_i$:
\[\mathrm{Sh}_{Q}(\AG[G_i],\mu) = 1-\frac{m_i\cdot P_0^2+P_i}{(|V|+i+1)!}\]

We can now rearrange the equation and substitute $P_i$ out to obtain a linear equations of unknowns $\mathsf{IS}(G,0),\dots,\mathsf{IS}(G,|V|)$ for $i\in[|V|+1]$:
\[
   \sum_{j=0}^{|V|}j!(|V|+i-j)!|\mathsf{IS}(G,j)|%
   =(n+i+1)!(1-\mathrm{Sh}_{Q}(\AG[G_i],\mu))-m_i\cdot P_0^2
\]

The right-hand side has been shown to be computable in polynomial time with a drastic Shapley oracle, so it only remains to prove that the system is solvable, which we do as follows: we take its underlying matrix whose general term is \mbox{$j!(|V|+i-j)!$}, we divide each column by $j!$ then reverse the column order to obtain the matrix of general term $(i+j+1)!$, which is known to be invertible \cite[proof of Theorem 1.1]{bacherDeterminantsMatricesRelated2002}.
\end{proof}

The main objective of this section will be proving \Cref{th:non-rw}. We start in \Cref{ssec:path} by explaining the core idea for the "reduction" with a restricted case and showing the existence of the necessary structures to build the "reduction". Then, in \Cref{ssec:general}, we delve into the proof details and show how to apply the idea from the restricted case in general. Note that while our focus here is on the Shapley value computation task, \Cref{th:non-rw} also establishes hardness for the related $\PQEPhalfOne_{Q}$ problem. In \Cref{ssec:prob}, we exploit \Cref{th:non-rw} to improve an existing complexity dichotomy for probabilistic OMQA. 

\subsection{Proof Idea via a Restricted Setting} \label{ssec:path}%
The idea behind the proof is the fact that any "non-FO-rewritable" "OMQ" $Q=\withT{q}\in(\horndl,\UCQ)$ must have "minimal supports" that contain arbitrarily deep tree-like structures (see \Cref{clm:Ak}). 
\AP These must contain some arbitrarily long "path", where we define a ""path"" between two "constants" $a_0$ and $a_k$ to be any non-empty set of "assertions" $\P = \set{R_1(a_0,a_1), \dotsc, R_k(a_{k-1},a_k)}$, where $a_0, \dotsc, a_k$ are pairwise distinct "constants" and each $R_i(a_{i-1},a_i)$ stands for a "role assertion" of the form $r(a_{i-1},a_i)$ or $r(a_i,a_{i-1})$.
To give the intuition, we shall first consider the restricted case where $\P$ is a "minimal support" for $Q$ on its own, and where $Q$ is "connected@@q" and "hom-closed".

Similarly to what has been done to prove \cite[Proposition 4.6]{livshitsShapleyValueTuples2021}, we wish to reduce from the $\AP\intro*\numBipIndep$ problem of counting the number of independent sets in a bipartite graph $G=(X,Y,E)$, which is known to be $\sP$-hard \cite[Problem 2]{provanComplexityCountingCuts1983}. For a sufficiently large $\P$, we can take two consecutive internal individuals $(a_\chi,a_{\chi+1})$, and duplicate them to encode an arbitrary bipartite graph as an "ABox" $\PG$, as shown in \Cref{fig:idee-construction}. 
We set as "endogenous" every "role assertion" linking $a_{\chi-1}$ to a copy of $a_{\chi}$, or linking a copy of $a_{\chi+1}$ to $a_{\chi+2}$, the remaining "assertions" are "exogenous".

\begin{figure}[tb]
\centering
\resizebox{.8\columnwidth}{!}{\begin{tikzpicture}[xscale=.6]
	\scriptsize
	\begin{pgfonlayer}{nodelayer}
		\node [draw, circle,color=blue] (0) at (3.5, 0) {};
		\node [draw, circle,color=blue] (1) at (4.5, 0) {};
		\node [draw, circle,color=blue] (2) at (3.5, 0.4) {};
		\node [draw, circle,color=blue] (3) at (4.5, 0.4) {};
		\node [draw, circle,color=blue] (4) at (3.5, -0.4) {};
		\node [draw, circle,color=blue] (5) at (4.5, -0.4) {};
		\node [draw, circle] (6) at (5.5, 0) {};
		\node [draw, circle] (7) at (2.5, 0) {};
		\node [] (8) at (6.5, 0) {$\dots$};
		\node [] (9) at (1.5, 0) {$\dots$};
		\node [draw, circle,color=blue] (10) at (-5.5, 0) {};
		\node [above=.1,color=blue] at (10) {\small $\alpha_\chi$};
		\node [draw, circle,color=blue] (11) at (-4.5, 0) {};
		\node [above=.1,color=blue] at (11) {\small $\alpha_{\chi+1}$};
		\node [draw, circle] (16) at (-3.5, 0) {};
		\node [draw, circle] (17) at (-6.5, 0) {};
		\node [] (18) at (-2.5, 0) {$\dots$};
		\node [] (19) at (-7.5, 0) {$\dots$};
		\pic[transform shape,yscale=.3,xscale=0.5] at (-0.5,0) {graphbox-ew};
	\end{pgfonlayer}
	\begin{pgfonlayer}{edgelayer}
		\draw (9) to (7);
		\draw (6) to (8);
		\draw [very thick, color=blue] (7) to (2);
		\draw [very thick, color=blue] (7) to (0);
		\draw [very thick, color=blue] (7) to (4);
		\draw [very thick, color=blue] (3) to (6);
		\draw [very thick, color=blue] (1) to (6);
		\draw [very thick, color=blue] (5) to (6);
		\draw [color=blue] (2) to (3);
		\draw [color=blue] (0) to (5);
		\draw [color=blue] (4) to (1);
		\draw (19) to (17);
		\draw (16) to (18);
		\draw [color=blue] (17) to (10);
		\draw [color=blue] (10) to (11);
		\draw [color=blue] (11) to (16);
	\end{pgfonlayer}
\end{tikzpicture}}
\caption[$\PG$]{\AP Encoding $\intro*\PG$ of $G$ in $\P$. The "endogenous" "assertions" are indicated by thick lines.}
\label{fig:idee-construction}
\end{figure}

Any subset $\X\inc \PGn$ naturally maps to a set $V_\X \subseteq X \cup Y$ of vertices of $G$. If $V_\X$ is not independent, then there is a "path" "homomorphic" to $\P$ in $\X\cup\PGx$, hence $\X\cup\PGx\models Q$. Otherwise, if $V_\X$ is independent, then any "connected@@db" subset of $\X\cup\PGx$ will "homomorphically" map to a \emph{proper} subset of $\P$. Since $Q$ is "connected@@q" and $\P$ is a "minimal support" for it, this implies that $\X\cup\PGx\nvDash Q$. From this we can use the same technique as \cite[Proposition 4.6]{livshitsShapleyValueTuples2021} to count the independent sets of $G$, as we shall see later.

\subsection{The General Case}\label{ssec:general}

Now that we have seen the underlying idea for the restricted case, we formally show the existence of the "path" $\P$ upon which we shall build the "reduction".
To fix the notations, let $Q=\withT{q}\in(\horndl,\UCQ)$ be the "non-FO-rewritable" ("w.r.t. consistent ABoxes") "OMQ" we consider. We write $q=\bigvee_{i\in I}\bigwedge_{j\in J_i} q_{i,j}$ where the $q_{i,j}$ are
\AP ""connected components@@q"", that is, maximal "connected@@q" subqueries.\footnote{Observe that, for example, the "CQ" $r(x,c) \land s(c,y)$, where $c$ is a constant, has only one "connected component@@q".}
\AP For every $q_{i,j}$ we write $\intro*\aCij\defeq\const(q_{i,j})$, and $\reintro*\aCij[]\defeq \bigcup_{i\in I}\bigdcup_{j\in J_i}\aCij$. Note that the last union is indeed disjoint because the "connected components@@q" of a "CQ" cannot share any "constant" or variable.

The following is a simple consequence of $Q$ being "non-FO-rewritable" "w.r.t. consistent ABoxes".

\begin{claim}\label{clm:SNij}
There exist $i,j$ such that for all $N\in\lN$, there are $\T$-"consistent" "ABoxes" $\APintrorep{\SNij},\reintro*{\SNij[i,\lnot j]}$ with $|\SNij|\ge N$ such that:
\begin{enumerate}[(1)]
   \item\label{clm:SNij-1} $\const(\SNij)\cap \const(\SNij[i,\lnot j])=\emptyset$ and  $\SNij \dcup \SNij[i,\lnot j]$ is a "minimal support" for $Q$;
\item\label{clm:SNij-2} $\SNij$ and $\SNij[i,\lnot j]$ are "minimal supports" for $(\T,q_{i,j})$ and $(\T,q_{i,\lnot j})$ respectively, where $q_{i,\lnot j}\defeq \bigwedge_{j'\in J_i\setminus\{j\}} q_{i,j}$.
\end{enumerate}
\end{claim}

\begin{proof}
Since $Q\defeq \withT{q}$ is "non-FO-rewritable" "w.r.t. consistent ABoxes" and is only concerned by the finite set of predicates that appear in $(q,\T)$, it must contain arbitrarily large $\T$-"consistent" "minimal supports". Thus, for every $N\in\lN$, let $S^N$ be a $\T$-"consistent" minimal support with $|S^N| \geq N $.

Observe that every "minimal support" for $Q$ is a "minimal support" for the OMQ $(\T, \bigwedge_{j\in J_i} q_{i,j})$ for some disjunct of $q$, meaning it can be written as the union of "minimal supports" of the OMQs  $\withT{q_{i,j}}$ for $j\in J_i$. In particular, for the minimal support $S^N$, 
we can find $i_N$ such that $S^N$ is the union of "minimal supports" of the OMQs  $\withT{q_{i_N,j}}$. Since $|S^N| \geq N$, 
there must furthermore exist $j_N \in J_{i_N}$ such that the "minimal support" for $\withT{q_{i_N,j_N}}$ contained within $S^N$ has size at least $\frac{|S^N|}{|J_{i_N}|}$. 
We can thus write $S^N = S^N_{i_N,j_N} \cup S^N_{i_N,\lnot j_N}$, where $S^N_{i_N,j_N}$ is the "minimal support" for $\withT{q_{i_N,j_N}}$
and $S^N_{i_N,\lnot j_N}$ the union of the "minimal supports" of $\withT{q_{i_N,j}}$ for all $j \in J_{i_N}$ with $j \neq J_N$. 
Note that we could potentially have $\const(S^N_{i_N,j_N}) \cap \const(S^N_{i_N,\lnot j_N})\neq\emptyset$. 
However, since the $q_{i_N,j}$ are (maximal) "connected components@@q", they have disjoint sets of constants, meaning that any constant $c$ that appears in both $ S^N_{i_N,j_N}$ and $ S^N_{i_N,\lnot j_N}$ must either not belong to $C_{i_N,j_N}$ (in which case occurrences of $c$
in $S^N_{i_N,j_N}$ can be replaced by a fresh constant) 
or does not belong to $C_{i_N,\lnot j_N}$ (in which case occurrences of $c$ within $S^N_{i_N,\lnot j_N}$ can be replaced by a fresh constant). After performing such renamings, we will have
$\const(S^N_{i_N,j_N}) \cap \const(S^N_{i_N,\lnot j_N})=\emptyset$, as required. 

Now let us consider the infinite sequence $(i_N,j_N)_{N\in\lN}$. Necessarily, there is some pair $(i^*, j^*)$ that occurs infinitely often in this sequence.
We fix some such $(i^*, j^*)$ %
argue that this pair satisfies the conditions of the claim. 
Indeed, given any $N\in\lN$, we can find $K > N \cdot |J_{i^*}|$ such that $(i_K,j_K)= (i^*,j^*)$. 
Fixing some such $K$, we then let $M_{i^*,j^*}^N= S^K_{i_K,j_K}= S^K_{i^*,j^*}$ 
and $M_{i^*, \lnot j^*}^N=  S^K_{i_K,\lnot j_K}= S^K_{i^*,\lnot j^*}$. The fact that $M_{i^*,j^*}^N$ 
and  $M_{i^*, \lnot j^*}^N$ satisfy conditions
\ref{clm:SNij-1} and \ref{clm:SNij-2} has been shown in the previous paragraph. 
Finally, having chosen $K > N \cdot |J_{i^*}|$, we know that 
$|M_{i^*,j^*}^N| \geq \frac{|S^K|}{|J_{i^*}|} \geq \frac{N \cdot |J_{i^*}|}{|J_{i^*}|}  \geq N$. 
\end{proof}

Let us henceforth fix $q_{\overline{\imath},\overline{\jmath}}$ as having the indices $\overline{\imath},\overline{\jmath}$ satisfying the statement of \Cref{clm:SNij}.
The next claim is a straightforward application of \cite[Proposition~23]{bienvenuFirstOrderRewritabilityContainment2020}, which we restated as \Cref{prop:unraveling}.
\begin{claim}\label{clm:Ak}
For all $k\in \lN$, there exists a "minimal support" $\APintrorep\Ak$ for $\withT{q_{\overline{\imath},\overline{\jmath}}}$ that is "$\T$-consistent", individuals $a_0, \dotsc, a_k$, and $N\in\lN$ such that:
\begin{enumerate}[(1)]
\item\label{clm:Ak-1} $\Ak \homto[C-] \SNij[\overline{\imath},\overline{\jmath}]$;
\item\label{clm:Ak-2} $\Ak$ has a "path" $\P = \{R_1(a_0,a_1), \dotsc, R_k(a_{k-1},a_k)\}$, and further $\P$ is the only "path" between $a_0$ and $a_k$ in $\Ak$;
\item\label{clm:Ak-3} $\{a_0, \dots, a_k\}\cap\aCij[] = \emptyset$%
\end{enumerate}
\end{claim}
\begin{proof}
Let us apply \Cref{prop:unraveling}
to the "OMQ" $\withT{q_{\overline{\imath},\overline{\jmath}}}$ and the "ABox" $\SNij[\overline{\imath},\overline{\jmath}]$ provided by \Cref{clm:SNij}. We obtain some $\A_N^*$ that satisfies \Cref{prop:unraveling}\ref{prop:unraveling-a}
and \Cref{prop:unraveling}\ref{prop:unraveling-b}, and by removing some assertions if needed we can assume that it is a "minimal support" for $\withT{q_{\overline{\imath},\overline{\jmath}}}$, because a "connected@@db" subset of a "pseudo tree" is still a "pseudo tree", whose "width@@ptree" and "depth@@ptree" are no larger. In particular it satisfies condition \ref{clm:Ak-1}.

Observe that $|\A_N^*|\ge|\SNij[\overline{\imath},\overline{\jmath}]|$, otherwise $\A_N^*$ would "homomorphically" map into a proper subset of $\SNij[\overline{\imath},\overline{\jmath}]$, which would contradict its minimality. Now since the "width@@ptree" and "depth@@ptree" of $\A_N^*$ are bounded by $|q|$ and $|\T|$ respectively, the only way for it to be arbitrarily large is to have an arbitrary "depth@@ptree". This means that if we choose a large enough $N$, the resulting $\A_N^*$ (which we shall call $\Ak$) will contain "tree@@ptree" of "depth@@ptree" at least $k$, which implies the existence of the $\P$ desired to satisfy condition \ref{clm:Ak-2}. Finally, since $\P$ is contained within a "tree@@ptree" of $\SNij[\overline{\imath},\overline{\jmath}]$ and $\aC$ appear in its "core@@ptree", condition \ref{clm:Ak-3} is automatically satisfied.%
\end{proof}

\AP We now have the "minimal support" $\intro*\Ao\defeq\Ak\dcup\SNij[\overline{\imath},\lnot\overline{\jmath}]$ we need for the "reduction". However, it differs from the one in the restricted case of \Cref{ssec:path} in two respects:
\begin{enumerate}
\item the "path" $\P$ is merely a subset of the "minimal support";
\item the "OMQ" $Q$ is not necessarily "connected@@q".
\end{enumerate}

To deal with the first issue we need to introduce some further notions.
\AP 
For every index $\chi \in [k-1]$ along $\P$,
let $\intro*\Akchi{\chi} \defeq \Ak \setminus \set{R_\chi(a_{\chi-1},a_\chi), R_{\chi+1}(a_{\chi},a_{\chi+1})}$.
Given an "assertion" $\alpha\in\Ak$ and an "individual" $a_\chi$ on the "path" $\P$, 
we say that $\alpha$ is \AP ""below@@interf""  (resp. \AP ""left@@interf"" of, resp. ""right@@interf"" of) $a_\chi$ if $\alpha$ is in the "connected component@@db" of $a_\chi$ in $\Akchi{\chi}$ (resp. $a_{\chi-1}$, resp. $a_{\chi+1}$).
Further, we say that the "assertions" $R_\chi(a_{\chi-1},a_\chi)$ and $R_{\chi+1}(a_{\chi},a_{\chi+1})$ are respectively "left@@interf" and "right@@interf" of $a_\chi$.

\begin{claim}\label{clm:below-left-right}
\AP For every "assertion" $a_\chi$ on the "path" $\P$, the sets $\intro*\Below(a_\chi)$, $\intro*\Left(a_\chi)$ and $\intro*\Right(a_\chi)$ of "assertions" that are "below@@interf", "left@@interf" of and "right@@interf" of $a_\chi$ resp.\ are disjoint. Furthermore, if $\chi<\lambda$, then $\Right(a_\chi)$ contains $\Below(a_\lambda)$ and $\Right(a_\lambda)$, while $\Left(a_\lambda)$ contains $\Below(a_\chi)$ and $\Left(a_\chi)$ (see \Cref{fig:below-left-right}).
\end{claim}
\begin{proof}
We start by establishing disjointness. Assume for a contradiction that $\Below(a_\chi)\cap\Right(a_\chi)\neq\emptyset$, the other cases being strictly analogous. Take some $\alpha \in \Below(a_\chi)\cap\Right(a_\chi)$ and an "individual" $b$ of $\alpha$. From the definition 
of $\Below(a_\chi)$ and $\Right(a_\chi)$, we know that $b$ belongs to the same "connected component@@db" of $\Akchi{\chi}$ as $a_\chi$  
and the same "connected component@@db" of $\Akchi{\chi}$ as $a_{\chi+1}$. It follows that $a_\chi$ and $a_{\chi+1}$ occur in the same connected component of $\Akchi{\chi}$. Since by definition $R_{\chi+1}(a_{\chi},a_{\chi+1}) \not \in \Akchi{\chi}$, 
this implies the existence of a ``detour'' "path" between $a_\chi$ and $a_{\chi+1}$ that is distinct from the portion of $\P$ because it doesn’t contain the "assertion" $R_{\chi+1}(a_{\chi},a_{\chi+1})$.
This detour can then be extended in a "path" between $a_0$ and $a_k$ that is distinct from $\P$, which contradicts \Cref{clm:Ak}.\ref{clm:Ak-2}.\medskip

Now we show the inclusions, based on \Cref{fig:below-left-right}. Consider indices $\chi < \lambda$ and an "assertion" $\alpha$ in either $\Below(a_\lambda)$ or $\Right(a_\lambda)$. By definition there exists a "path" from $a_\lambda$ that contains $\alpha$ but doesn’t intersect any $a_\kappa$ with $\kappa<\lambda$. This "path" can therefore be extended by adding $R_{\chi+1}(a_\chi,a_{\chi+1}),\dots,R_\lambda(a_{\lambda-1},a_{\lambda})$ without creating any cycle, thereby forming a "path" from $a_\chi$ that contains $\alpha$ and intersects $a_\lambda$, which implies $\alpha\in\Right(a_\chi)$ since $\chi < \lambda$. We also obtain $\Below(a_\chi)\cup\Left(a_\chi)\subseteq\Right(a_\lambda)$ by symmetry.
\end{proof}

\begin{figure}[tb]
\centering
\begin{tikzpicture}[scale=1.5]
\begin{pgfonlayer}{nodelayer}
	\pic[transform shape,scale=.5] (0) at (-1, 0) {subtree};
	\pic[transform shape,scale=.5] (1) at (1, 0) {subtree};
	\node [] (2) at (-2, 0) {$\cdots$};
	\node [] (3) at (0, 0) {$\cdots$};
	\pic[transform shape,scale=.5] (4t) at (2, 0) {subtrinvisible};
	\node [] (4) at (2, 0) {$\cdots$};
	\node at (0-a) {$a_\chi$};
	\node at (1-a) {$a_\lambda$};
\end{pgfonlayer}
\begin{pgfonlayer}{edgelayer}
	\draw (2) to (0-west);
	\draw (0-east) to (3);
	\draw (3) to (1-west);
	\draw (1-east) to (4);
	\draw[color=red, dashed,scale=.5] (1-e) arc (-45:-135:0.5) -- (1-h) -- node[midway,above=.06] {$\mathbf{B}(a_{\lambda})$} (1-i) -- cycle;
	\draw[color=blue,dashed,scale=.5] (1-e) arc (-45:45:0.5) -- (4t-j) arc (45:-45:0.5) --  node[midway,below] {$\mathbf{R}(a_{\lambda})$} cycle;
	\draw[color=green,dashed,scale=.5] ($(0-j)+(.1,.1)$) arc (45:-45:0.65) -- node[midway,below] {$\mathbf{R}(a_{\chi})$} ($(1-d) + (-.2,-.1)$) -- ($(1-h) + (-.2,-.1)$) -- ($(1-i) + (.2,-.1)$) -- ($(1-e) + (.2,-.1)$) -- ($(4t-e) + (.1,-.1)$) arc (-45:45:0.65) -- ($(1-j)+(.1,.1)$) arc (45:135:0.65) -- cycle;
\end{pgfonlayer}
\end{tikzpicture}
\caption{Illustration of why $\Right(a_\chi)$ contains $\Below(a_\lambda)$ and $\Right(a_\lambda)$.}
\label{fig:below-left-right}
\end{figure}

\Cref{clm:below-left-right}
allows us to generalize the construction from the restricted case. The idea is simple: every time we copy an "assertion" in the restricted case, we also copy its $\Below$ set, %
with fresh "constants". Note that the construction still depends on a pair $(a_\chi,a_{\chi+1})$ of consecutive "individuals" of $\P$ that are \AP ""internal"" (\ie\ not in $\{a_0,a_k\}$). We call such pair an \AP ""interface"". We overlook the choice of this "interface" for now, but it will become relevant later.

\begin{figure}[tb]
\centering
\begin{tikzpicture}[scale=1.4]
\begin{pgfonlayer}{nodelayer}
	\coordinate (00) at (2, -2);
	\coordinate (01) at (3, -2);
	\coordinate (02) at (2, -1);
	\coordinate (03) at (3, -1);
	\coordinate (04) at (2, -3);
	\coordinate (05) at (3, -3);
	\coordinate (06) at (4, -2);
	\coordinate (07) at (1, -2);
	\coordinate (08) at (4.75, -2);
	\coordinate (09) at (0.25, -2);
	\coordinate (010) at (2, 1);
	\coordinate (011) at (3, 1);
	\coordinate (016) at (4, 1);
	\coordinate (017) at (1, 1);
	\coordinate (018) at (4.75, 1);
	\coordinate (019) at (0.25, 1);
	\coordinate (020) at (-0.5, 1);
	\coordinate (021) at (5.5, 1);
	\coordinate (022) at (-0.5, -2);
	\coordinate (023) at (5.5, -2);
	\coordinate (024) at (0.25, -0.75);
	\coordinate (025) at (-1, 1.25);
	\coordinate (026) at (-2, 0.75);
	\coordinate (027) at (-1.5, 1);
	\coordinate (028) at (-1, 0);
	\coordinate (029) at (-2, -0.5);
	\coordinate (030) at (-1.5, -0.25);
	\coordinate (031) at (-2.25, 0.25);
	\coordinate (032) at (-2.25, 1.5);
	\coordinate (033) at (-2, -0.75);
	\coordinate (034) at (5.5, -0.75);
	\coordinate (035) at (-1, -1.25);
	\coordinate (036) at (-2, -1.75);
	\coordinate (037) at (-1.5, -1.5);
	\coordinate (038) at (-1, -2.25);
	\coordinate (039) at (-2, -2.75);
	\coordinate (040) at (-1.5, -2.5);
	\coordinate (041) at (-2.25, -3.75);
	\coordinate (042) at (-2.25, -0.75);
	\coordinate (east) at (6,0);
	\node at (east) {\phantom{.}};
	
	\pic[transform shape,scale=.3, color=blue] (0) at (00) {subtree};
	\pic[transform shape,scale=.3, color=blue] (1) at (01) {subtree};
	\pic[transform shape,scale=.3, color=blue] (2) at (02) {subtree};
	\pic[transform shape,scale=.3, color=blue] (3) at (03) {subtree};
	\pic[transform shape,scale=.3, color=blue] (4) at (04) {subtree};
	\pic[transform shape,scale=.3, color=blue] (5) at (05) {subtree};
	\pic[transform shape,scale=.3] (6) at (06) {subtree};
	\pic[transform shape,scale=.3] (7) at (07) {subtree};
	\node [] (8) at (08) {$\dots$};
	\node [] (9) at (09) {$\dots$};
	\pic[transform shape,scale=.3, color=blue] (10) at (010) {subtree};
	\node[draw=blue,circle,fill=white,minimum height=1.5em] at (010) {};
	\node at (010) {$a_{\chi}$};
	\pic[transform shape,scale=.3, color=blue] (11) at (011) {subtree};
	\node[draw=blue,circle,fill=white,minimum height=1.5em] at (011) {};
	\pic[transform shape,scale=.3] (16) at (016) {subtree};
	\node[draw,circle,fill=white,minimum height=1.5em] at (016) {};
	\pic[transform shape,scale=.3] (17) at (017) {subtree};
	\node[draw,circle,fill=white,minimum height=1.5em] at (017) {};
	\node [] (18) at (018) {$\dots$};
	\node [] (19) at (019) {$\dots$};
	\pic[transform shape,scale=.3] (20) at (020) {subtree};
	\node[draw,circle,fill=white,minimum height=1.5em] at (020) {};
	\node at (020) {$a_0$};
	\pic[transform shape,scale=.3] (21) at (021) {subtree};
	\node[draw,circle,fill=white,minimum height=1.5em] at (021) {};
	\node at (021) {$a_k$};
	\pic[transform shape,scale=.3] (22) at (022) {subtree};
	\pic[transform shape,scale=.3] (23) at (023) {subtree};
	\pic[transform shape,scale=.3] (24) at ($(024)+(0,0.1)$) {graphbox-ns};
	\node [] (25) at ($(025)+(0.1,0.1)$) {};
	\node [] (26) at ($(026)+(-0.1,-0.1)$) {};
	\node [] (27) at (027) {\scriptsize $\Ak\setminus\Reach{\P}$};
	\node [] (28) at ($(028)+(0.1,0.1)$) {};
	\node [] (29) at ($(029)+(-0.1,-0.1)$) {};
	\node [] (30) at (030) {$\SNij[\overline{\imath},\lnot\overline{\jmath}]$};
	\node [] (31) at (031) {};
	\node [] (32) at (032) {};
	\node [] (33) at (033) {};
	\node [] (34) at (034) {};
	\node [] (35) at ($(035)+(0.1,0.1)$) {};
	\node [] (36) at ($(036)+(-0.1,-0.1)$) {};
	\node [] (37) at (037) {\scriptsize $\Ak\setminus\Reach{\P}$};
	\node [] (38) at ($(038)+(0.1,0.1)$) {};
	\node [] (39) at ($(039)+(-0.1,-0.1)$) {};
	\node [] (40) at (040) {$\SNij[\overline{\imath},\lnot\overline{\jmath}]$};
	\node [] (41) at (041) {};
	\node [] (42) at (042) {};
	\node [color=green,below] at ($(018)+(0,-3mm)$) {$\P$};
\end{pgfonlayer}
\begin{pgfonlayer}{edgelayer}
	\draw (9) to (07);
	\draw (06) to (8);
	\draw [thick, color=blue] (07) to (02);
	\draw [thick, color=blue] (07) to (00);
	\draw [thick, color=blue] (07) to (04);
	\draw [thick, color=blue] (03) to (06);
	\draw [thick, color=blue] (01) to (06);
	\draw [thick, color=blue] (05) to (06);
	\draw [color=blue] (02) to (03);
	\draw [color=blue] (00) to (05);
	\draw [color=blue] (04) to (01);
	\draw (19) to (017);
	\draw (016) to (18);
	\draw [color=blue] (017) to (010);
	\draw [color=blue] (010) to (011);
	\draw [color=blue] (011) to (016);
	\draw (020) to (19);
	\draw (18) to (021);
	\draw (022) to (9);
	\draw (8) to (023);
	\draw [decorate,decoration={brace,amplitude=3pt,raise=-2pt}] (31) to node[midway,left] {\normalsize$\Ak$} (32);
	\draw (26) rectangle (25);
	\draw (29) rectangle (28);
	\draw [decorate,decoration={brace,amplitude=3pt,raise=-2pt}] (41) to node[midway,left] {\normalsize$\AG$} (42);
	\draw (36) rectangle (35);
	\draw (39) rectangle (38);
	\draw [dashed, color=gray] (33) to (34);
	\draw [dashed, color=green] ($(020)+(0,3mm)$) arc (90:270:3mm) -- ($(021)+(0,-3mm)$) arc (-90:90:3mm) -- cycle;
\end{pgfonlayer}
\end{tikzpicture}
\caption[Construction of $\AG$]{Construction of $\AG$ from $\Ao$ and $G$. Thick lines represent "endogenous" "assertions". The boxes are sets with no structure, and the triangle below some $a$ is the "connected@@db" $\Below(a)$. \AP $\intro*\Reach{\P}$ denotes the set of "assertions" in the "connected component@@db" of $\P$.}
\label{fig:construction}
\end{figure}

Starting with a bipartite graph $G=(X,Y,E)$ whose independent sets we wish to count to solve $\numBipIndep$, we build the "partitioned ABox" $\AG$ illustrated in \Cref{fig:construction}
(compare with the restricted case of \Cref{fig:idee-construction}), which we formally define as follows. Starting from $\Ao$, we focus on the three "assertions" $R_{\chi}(a_{\chi-1},a_{\chi})$, $R_{\chi+1}(a_{\chi},a_{\chi+1})$, $R_{\chi+2}(a_{\chi+1},a_{\chi+2})$ of $\P$ that intersect the chosen $(a_\chi,a_{\chi+1})$. We replace $a_{\chi}$ (resp. $a_{\chi+1}$) with a family $(b_x)_{x\in X}$ (resp. $(c_y)_{y\in Y}$) of copies with fresh "constants". We also copy every "assertion" "below@@interf" $a_{\chi}$ (resp.\ "below@@interf" $a_{\chi+1}$), again with fresh "constants". Finally, we add the "assertions" $\set{R_{\chi}(a_{\chi-1},b_x)}_{x\in X}$, $\set{R_{\chi+1}(b_x,c_y)}_{(x,y)\in E}$, $\set{R_{\chi+2}(c_y,a_{\chi+2})}_{y\in Y}$ to obtain an "ABox", which we denote by $\intro*\AG$. As for the associated "partition@@db", we set the "endogenous" "assertions" $\AGn\defeq \{R_{\chi}(a_{\chi-1},b_x)\mid x\in X\}\cup\{R_{\chi+2}(c_y,a_{\chi+2})\mid y\in Y\}$ and the rest as "exogenous".

We can now try to apply to $\AG$ the same arguments that we used on $\PG$ for the restricted case. If we consider a subset $\X\inc\PGn$, it once again naturally maps to a set of vertices of $G$, via a bijection we denote by \AP$\intro*\etardx:\AGn\to X\cup Y$. 
If $\etardx(\X)$ is not independent, then there is a "path" "$\aC$-homomorphic" to $\P$ in $\X\cup\AGx$, and since every "assertion" in $\Ao\setminus\P$ is present in $\AGx$ (or at least has the necessary "$\aC$-isomorphic" copies in the case of the "assertions" "below@@interf" $a_{\chi}$ and $a_{\chi+1}$), this implies $\X\cup\AGx\models Q$.

If $\etardx(\X)$ is independent, however, we cannot directly use the argument above, because $Q$ is no longer assumed to be "connected@@q". In fact $\X\cup\AGx$ could conceivably contain a "disconnected@@db" "minimal support" $ M$ for $Q$ such as depicted in \Cref{fig:construction2}.\ref{fig:construction2-a}. Nevertheless, if we denote by $\rho:\AG\to \Ao$ the $\aC$-"homomorphism" that maps every fresh "constant" back to its original counterpart (essentially, $\rho$ reverses the construction of $\AG$), then we must have $\rho( M)=\Ao$. Otherwise, this would contradict the "minimality@minimal support" of $\Ao$. Moreover, by leveraging the fact that the "OMQ" $\withT{q_{i,j}}$ built from any "connected component@@q" $q_{i,j}$ is "connected@@q" by \Cref{lem:cnx}; we can further show that such an $ M$ always admits a partition $ M= M_\la\dcup M_\to$ with specific properties. This motivates the following definition of "splittable" "interfaces".

\begin{figure}[tb]
\centering
\begin{tikzpicture}[scale=1.4]
\begin{pgfonlayer}{nodelayer}
	\coordinate (01) at (2, 2);
	\coordinate (02) at (1, 3);
	\coordinate (04) at (1, 1);
	\coordinate (05) at (2, 3);
	\coordinate (06) at (3, 2);
	\coordinate (07) at (0, 2);
	\coordinate (08) at (3.75, 2);
	\coordinate (09) at (-0.75, 2);
	\coordinate (022) at (-1.5, 2);
	\coordinate (023) at (4.5, 2);
	\coordinate (035) at (-2, 2.75);
	\coordinate (036) at (-3, 2.25);
	\coordinate (037) at (-2.5, 2.5);
	\coordinate (038) at (-2, 1.75);
	\coordinate (039) at (-3, 1.25);
	\coordinate (040) at (-2.5, 1.5);
	\coordinate (041) at (1.5, 0.5);
	\coordinate (042) at (-1.5, -1);
	\coordinate (043) at (-0.75, -1);
	\coordinate (044) at (0, -1);
	\coordinate (045) at (1, 0);
	\coordinate (046) at (2, 0);
	\coordinate (047) at (1, -2);
	\coordinate (048) at (2, -1);
	\coordinate (049) at (3, -1);
	\coordinate (050) at (3.75, -1);
	\coordinate (051) at (4.5, -1);
	\coordinate (052) at (-2, -0.25);
	\coordinate (053) at (-3, -0.75);
	\coordinate (054) at (-2.5, -0.5);
	\coordinate (055) at (-2, -1.25);
	\coordinate (056) at (-3, -1.75);
	\coordinate (057) at (-2.5, -1.5);
	\coordinate (058) at (1.5, -2.5);
	
	\node[draw,circle,color=blue] (1) at (01) {};
	\pic[transform shape,scale=.3, color=blue] (2) at (02) {subtree};
	\node[draw,circle,color=blue] (4) at (04) {};
	\pic[transform shape,scale=.3, color=blue] (5) at (05) {subtree};
	\pic[transform shape,scale=.4] (6) at (06) {subtree};
	\node[draw,circle,fill=white,minimum height=2.3em] at (06) {};
	\node[] at (06) {$a_{\chi+2}$};
	\pic[transform shape,scale=.4] (7) at (07) {subtree};
	\node[draw,circle,fill=white,minimum height=2.3em] at (07) {};
	\node[] at (07) {$a_{\chi-1}$};
	\node [] (8) at (08) {$\dots$};
	\node [] (9) at (09) {$\dots$};
	\pic[transform shape,scale=.3] (22) at (022) {subtree};
	\pic[transform shape,scale=.3] (23) at (023) {subtree};
	\node [] (35) at ($(035)+(0.1,0.1)$) {};
	\node [] (36) at ($(036)+(-0.1,-0.1)$) {};
	\node [] (37) at (037) {\scriptsize $\Ak\setminus\Reach{\P}$};
	\node [] (38) at ($(038)+(0.1,0.1)$) {};
	\node [] (39) at ($(039)+(-0.1,-0.1)$) {};
	\node [] (40) at (040) {$\SNij[\overline{\imath},\lnot\overline{\jmath}]$};
	\node [] (41) at ($(041)+(0,0.2)$) {(a)};
	\pic[transform shape,scale=.3, color=red] (42) at (042) {subtree};
	\node [] (43) at (043) {\dots};
	\node [color=red] at ($(043)+(0,-0.35)$) {\normalsize $M_{\la}$}; %
	\pic[transform shape,scale=.4, color=red] (44) at (044) {subtree};
	\node[draw=red,circle,fill=white,minimum height=2.3em] at (044) {};
	\node[] at (044) {$a_{\chi-1}$};
	\pic[transform shape,scale=.3, color=red] (45) at (045) {subtree};
	\pic[transform shape,scale=.3, color=red] (46) at (046) {subtree};
	\node [draw, circle, color=green] (47) at (047) {};
	\node [draw, circle, color=green] (48) at (048) {};
	\pic[transform shape,scale=.4,color=green] (49) at (049) {subtree};
	\node[draw=green,circle,fill=white,minimum height=2.3em] at (049) {};
	\node[] at (049) {$a_{\chi+2}$};
	\node [] (50) at (050) {\dots};
	\node [color=green] at ($(050)+(0,-0.35)$) {\normalsize $M_{\to}$}; %
	\pic[transform shape,scale=.3,color=green] (51) at (051) {subtree};
	\node [] (52) at ($(052)+(0.1,0.1)$) {};
	\node [] (53) at ($(053)+(-0.1,-0.1)$) {};
	\node [] (54) at (054) {\scriptsize $\Ak\setminus\Reach{\P}$};
	\node [] (55) at ($(055)+(0.1,0.1)$) {};
	\node [] (56) at ($(056)+(-0.1,-0.1)$) {};
	\node [] (57) at (057) {$\SNij[\overline{\imath},\lnot\overline{\jmath}]$};
	\node [] (58) at ($(058)+(0,0.2)$) {(b)};
\end{pgfonlayer}
\begin{pgfonlayer}{edgelayer}
	\draw (9) to (07);
	\draw (06) to (8);
	\draw [color=blue] (4) to (1);
	\draw (022) to (9);
	\draw (8) to (023);
	\draw [] (36) rectangle (35);
	\draw [] (39) rectangle (38);
	\draw [color=blue] (07) to (02);
	\draw [color=blue] (06) to (1);
	\draw [style=red] (042) to (43);
	\draw [style=red] (43) to (044);
	\draw [style=red] (044) to (045);
	\draw [color=green] (47) to (48);
	\draw [color=green] (48) to (049);
	\draw [color=green] (049) to (50);
	\draw [color=green] (50) to (051);
	\draw [color=red] (53) rectangle (52);
	\draw [color=red] (56) rectangle (55);
\end{pgfonlayer}
\end{tikzpicture}
\captionsetup{singlelinecheck=off}
\caption[Justification for introducing splittable interfaces]{%
\label{fig:construction2}%
\begin{enumerate*}[(a)]
\item\label{fig:construction2-a} Example of what $ M$ could look like. There is no path between $a_{\chi-1}$ and $a_{\chi+2}$, but the "ABox" can still collapse back to the whole $\Ao$ by $\rho$. This implies in particular that the whole black regions are present.
\item\label{fig:construction2-b} Partition of $ M$ into $ M_\la\dcup M_\to$ that makes the "interface" "splittable".
\end{enumerate*}
}
\end{figure}
\AP An "interface" $(a_\chi,a_{\chi+1})$ is said to be ""splittable"" if there exists some $i \in I$, two sets $J_{\la},J_{\to}\subseteq J_i$ of "connected components@@q" of the $i$\textsuperscript{th} disjunct of $q$, and a "minimal support" $ M_\la$ for $\bigwedge_{j\in J_{\la}}q_{i,j}$ which contains every "assertion" "left@@interf" of $a_\chi$ and no "assertion" "right@@interf" of $a_{\chi+1}$, and the symmetrical property for $J_{\to}$, with the added condition that $ M_\to$ is "connected@@db". These two "minimal supports" correspond to the red and green regions in \Cref{fig:construction2}.\ref{fig:construction2-b}, and they are the only way the construction can fail:

\begin{claim}\label{clm:unsplittable}
$\AG$ is "$\T$-consistent" and, if it is built from an "unsplittable" "interface" $(a_\chi,a_{\chi+1})$, then for every subset $\X\inc\AGn$, $\X\dcup\AGx\vDash Q$ iff $\etardx(\X)$ isn’t independent in $G$.
\end{claim}
\begin{proof}
We first show that $\AG$ is "$\T$-consistent". Denote by $\rho:\AG\to \Ao$ the $\aC$-"homomorphism" that maps every fresh constant back to its original counterpart. Essentially, $\rho$ reverses the construction of $\AG$. It then suffices to remark that the property of being $\T$-"inconsistent" %
is "$C$-hom-closed" (this is a straightforward consequence of Lemma \ref{incons-bot}) %
and that $\rho(\AG)=\Ao$. Indeed, this means that if $\AG$ were "inconsistent" with $\T$, then $\Ao$ would be as well, which would contradict \Cref{clm:SNij} or \ref{clm:Ak}.%

Now we assume $\AG$ is built from an "unsplittable" "interface" $(a_\chi,a_{\chi+1})$ and show the equivalence, one direction at a time.

\proofcase{$\Ra$} Let us assume for a contradiction that $\X\subseteq \AGn$ is such that %
$\X\dcup\AGx\vDash Q$ but $\etardx(\X)$ is an independent set of $G$. Then, necessarily, there is a "minimal support" $M\subseteq \X\cup\AGx$ for $Q$. With the same $\rho$ as above, since $Q$ is $\aC$-"hom-closed", we obtain $\rho( M)\omqsat Q$. Moreover, $\rho( M)=\Ao$ since the latter is a "minimal support" for $Q$.

The fact that $ M$ maps to the whole of $\Ao$ means that it necessarily contains the set $\AG\cap\Ao$ of facts that were untouched by the construction (the black regions in \Cref{fig:construction,fig:construction2}).

The fact that $ M$ maps to the whole of $\Ao$ further implies that it contains some $R_{\chi}(a_{\chi-1},b_x)$ and some $R_{\chi+2}(c_y,a_{\chi+2})$, but there cannot be any "assertion" connecting them together otherwise $\X$ wouldn’t be an independent set of $G$. This means that $ M$ is "disconnected@@db" since the respective sets $\Reach{a_{\chi-1}}$ and $\Reach{a_{\chi+2}}$ of "assertions" reachable from $a_{\chi-1}$ and $a_{\chi+2}$ are disjoint. If we then partition $ M$ into the "disconnected@@db" sets $ M_{\to}\defeq\Reach{a_{\chi+2}}$ and $ M_{\la}\defeq M\setminus M_{\to}$ (see \Cref{fig:construction2}.\ref{fig:construction2-b}), then since $ M$ is a "minimal support" for $Q$, $ M_{\la}$ and $ M_{\to}$ must necessarily be "minimal supports" for respective $\withT{\bigwedge_{j\in J_{\la}}q_{i,j}}$ and $\withT{\bigwedge_{j\in J_{\to}}q_{i,j}}$, where $J_{\la}$ and $J_{\to}$ are sets of "connected components@@q" of some $i$\textsuperscript{th} disjunct of $q$, which together form the whole disjunct. This is because the $\withT{q_{i,j}}$ are "connected@@q" by \Cref{lem:cnx}, hence any of their "minimal supports" in $ M$ must be fully contained in either $ M_{\la}$ or $ M_{\to}$. Finally, since $ M_{\la}$ contains every "assertion" "left@@interf" of $a_\chi$ and no "assertion" "right@@interf" of $a_{\chi+1}$, and vice versa for $ M_{\to}$, this contradicts the fact that the chosen interface is "unsplittable". %

\proofcase{$\La$} By contrapositive, if $\etardx(\X)$ isn't an independent set of $G$, then there is a path $a_{\chi-1} \xrightarrow{R_{\chi}} b_x \xrightarrow{R_{\chi+1}} c_y \xrightarrow{R_{\chi+2}} a_{\chi+2}$ in $\X\cup\AGx$, which was the only missing piece to have a $\aC$-"homomorphism" from the "minimal support" for $Q$. Since $Q$ is $\aC$-"hom-closed", this means $\X\cup\AGx\omqsat Q$.
\end{proof}

The final ingredient is the fact that we can always find an "unsplittable" "interface" in $\P$ if it is long enough.

\begin{claim}\label{clm:splittable}
The number of "splittable" "interfaces" in $\P$ is bounded by a function of $q$. 
\end{claim}
\begin{proof}
We show that two disjoint "interfaces" cannot have the same $J_{\to}\subseteq J_i$ for some $i\in I$, which implies the total number of "splittable" "interfaces" is bounded by 2 times the number of sets of "connected components@@q" of $q$. By contradiction, assume $(a_\chi,a_{\chi+1})$ and $(a_\lambda,a_{\lambda+1})$, with $\chi+1 < \lambda$ are two "splittable" "interfaces" with the same $J_{\to}$. By definition, there must be a "minimal support" $ M_{\chi}$ (resp. $ M_{\lambda}$) for $\bigwedge_{j\in J_{\la}}q_{i,j}$ such that $\Right(a_{\chi+1})\inc M_{\chi}$ (resp. $\Right(a_{\lambda+1})\inc M_{\lambda}$) and $\Left(a_\chi)\cap M_{\chi}=\emptyset$ (resp. $\Left(a_\lambda)\cap M_{\lambda}=\emptyset$). $\Left(a_\lambda)\cap M_{\lambda}=\emptyset$ implies in particular that $ M_\lambda\inc \Below(a_\lambda)\cup\Right(a_\lambda)$ since $ M_{\lambda}$ is "connected@@db" (it is a $ M_{\to}$ set in the definition of "splittable" "interfaces") and intersects $a_\lambda$. From there we get:%
\[
\begin{array}{rll}
M_\lambda &\inc \Below(a_\lambda)\cup\Right(a_\lambda) &\\
&\inc \Right(a_{\chi+1}) &\text{(\Cref{clm:below-left-right} and $\chi+1 < \lambda$)}\\
&\inc  M_\chi &\text{(above paragraph)}
\end{array}
\]
Finally, $ M_{\chi}$ intersects $a_{\chi+1}$ while $ M_{\lambda}$ doesn’t, which means that $ M_{\lambda}$ is a proper subset of $ M_{\chi}$, which contradicts the fact they are both "minimal supports" for $Q$.	
\end{proof}

Once we have this, \Cref{clm:unsplittable} gives a bijection between the $\X$ such that $\X\cup\AGx\nvDash Q$ and the independent sets of $G$, which yields the desired reduction $\numBipIndep \polyrx \PQEPhalfOne_Q$ by the direct equivalence between counting satisfying subsets and probabilistic query evaluation. We can then obtain $\numBipIndep \polyrx \dShapley_Q$ by 
exactly reproducing the end of the proof for \cite[Proposition 4.6]{livshitsShapleyValueTuples2021}, which consists in building several variants of the graph $G$ to obtain a solvable linear system for the number of independent sets of $G$ of each size. Details are provided in \ref{app:non-rw}.

\subsection{Improved Dichotomy for Probabilistic OMQA}\label{ssec:prob}
As a consequence of \Cref{th:non-rw}, we obtain the following dichotomy for the "probabilistic evaluation" of "OMQ"s, which extends  \cite[Theorem 7]{jungOntologyBasedAccessProbabilistic2012} by generalizing from the class of constant-free connected OMQs from $(\mathcal{ELI},\CQ)$ to \emph{all} OMQs in $(\horndl,\UCQ)$.

\begin{theorem}\label{th:pqe-dichotomy}
Let $Q$ be a $(\horndl,\UCQ)$ "OMQ". Then $\PQEPhalfOne_{Q}$, $\SPPQE_{Q}$ and $\PQE_{Q}$ are all in $\FP$ if $Q$ is "FO-rewritable" into a "safe@@q" "UCQ", and $\sP$-hard otherwise. Further, this dichotomy is "effective" if $Q\in(\horndl,\CQ)$.
\end{theorem}

\begin{proof}
If %
$Q$ is "FO-rewritable" (w.r.t.\ arbitrary ABoxes), then it admits a "UCQ"-"rewriting", and all problems between $\PQEPhalfOne_{Q}$ and $\PQE_{Q}$ enjoy the same $\FP$/$\sP$-hard dichotomy as for "UCQ"s \cite{kenigDichotomyGeneralizedModel2021,dalviDichotomyProbabilisticInference2012}.
If instead $Q$ is "non-FO-rewritable" "w.r.t. consistent ABoxes", then we can apply Theorem \ref{th:non-rw}. 

It therefore remains to consider the case where $Q$ is "non-FO-rewritable" (w.r.t.\ arbitrary ABoxes) but "FO-rewritable" "w.r.t.\ consistent ABoxes", i.e.\ when 
the $\T$-consistency check is "non-FO-rewritable".  Thus, let $Q=(\T,q)$ be a $(\horndl,\UCQ)$ "OMQ" with these properties. 
As explained in Section \ref{prelims-dl}, we may suppose w.l.o.g. that $\T$ is in "$\bot$-normal form", i.e.\ $\bot$ occurs in $\T$ uniquely in concept inclusions of the form $C \sqsubseteq \bot$. 

Consider the "OMQ" $Q' =(\T',q')$ where $\T'$ 
is obtained from $\T$ by replacing every occurrence of $\bot$ 
with a fresh concept name $A_\bot$, and letting 
$$ q' = q \vee \exists x. A_\bot(x) $$ %
It has been shown in Lemma \ref{incons-bot} that for every ABox $\A$ that does not mention $A_\bot$, we have the following:
\begin{equation}\label{nobot-equiv}
   (\A, \T) \models q \text { iff } (\A, \T') \models q'  
\end{equation}

We wish to show that there is no "FO-rewriting" of $Q'$ "w.r.t.\ consistent ABoxes". Suppose for a contradiction that such a rewriting $\varphi$ exists. Note that since $\T'$ does not contain $\bot$, every ABox is trivially $\T'$-consistent, hence $\varphi$ is an "FO-rewriting" of $Q'$ w.r.t.\ arbitrary ABoxes, hence also w.r.t.\ the class of ABoxes without $A_\bot$.  
It follows that from Eq.\ \eqref{nobot-equiv} that $\varphi$ is an "FO-rewriting" of the original OMQ $Q$ w.r.t.\ the class of ABoxes not mentioning $A_\bot$. Finally, since $Q$ does not make any mention of $A_\bot$, %
we can obtain an FO-rewriting of $Q$ w.r.t.\ arbitrary ABoxes by replacing every occurrence of an atom $A_\bot(x)$  by some unsatisfiable query (e.g.\ $D(x) \wedge \neg D(x)$ for some concept name $D$).  We thus obtain an "FO-rewriting" of $Q$ "w.r.t. consistent ABoxes", contradicting our assumption that no such rewriting exists.

As we have shown that $Q'$ is "non-FO-rewritable" "w.r.t. consistent ABoxes", 
we can now apply Theorem \ref{th:non-rw} to $Q'$, which tells us that 
$\PQEPhalfOne_{Q'}$ %
is $\sP$-hard. We observe that for any probabilistic database $\D=(X, \pi)$, with $X$ an ABox and $\pi: X \to \{\nicefrac{1}{2},1\}$, 
we have:
\[
   \Prob(\D \models Q') = \Prob(\D \models \exists x. A_\bot(x))
   + (1- \Prob(\D \models \exists x. A_\bot(x))) \cdot \Prob(\D \models Q)
\]
We then observe that $\Prob(\D \models \exists x. A_\bot(x))$ can be easily computed:
$$\Prob(\D \models \exists x. A_\bot(x)) = 1 - \prod_{A_\bot \! (c)\in X} (1- \pi(A_\bot(c)) $$ 
It follows that it is $\sP$-hard to compute $\Prob(\D \models Q)$. 

To show the effectiveness of the dichotomy, we use the fact that  "FO-rewritability" (w.r.t.\ arbitrary ABoxes and consistent ABoxes) is effective over $(\horndl,\CQ)$ \cite[Theorem 5]{bienvenuFirstOrderrewritabilityContainment2016a}, and the problem of whether or not a given "UCQ" is "safe@@q" is decidable as well \cite{dalviDichotomyProbabilisticInference2012}.
\end{proof}

\section{Discussion and Future Work}\label{sec:discuss}
While 
 the "drastic Shapley value" had previously been suggested for ontology debugging \cite{dengMeasuringInconsistenciesOntologies2007}, 
its application to ontology-mediated query answering had not yet been considered. Moreover, a precise analysis of the complexity of Shapley value computation in the ontology setting was missing. The present paper proposes and studies 
different formulations of the "drastic Shapley value" computation task for description logic knowledge bases, 
by varying what is to be explained (entailment of a TBox axiom, ABox assertion, or query answer),
which parts of the KB are assigned responsibility values (axioms and/or assertions), and how complexity is measured. 

Our initial complexity results, presented in Section \ref{sec:reach}, established the 
intractability of various "drastic Shapley value" computation tasks related to axiom entailment and instance queries,
covering a wide range of description logics. We did however identify one setting in which
"drastic Shapley values" can be tractably computed via a straightforward algorithm (Proposition \ref{prop:dllite-atomic}): 
 instance query entailment in common DL-Lite dialects. Moreover, these mostly %
negative results spurred us to perform a more fine-grained complexity analysis 
at the level of individual ontology-mediated queries. 

By building upon very recent results on drastic Shapley value computation in the database setting, 
we were able to establish complexity dichotomies for "OMQ"s whose "TBoxes" are formulated 
in the Horn description logic $\horndl$, 
thereby covering prominent DLs of the DL-Lite and $\mathcal{EL}$ families. 
Theorem~\ref{th:cnx-hom-closed} identifies classes of "OMQ"s for which the "drastic Shapley value" can be computed in $\FP$, 
while Theorem \ref{th:non-rw} provides a more general $\sP$-hardness result for "non-FO-rewritable" "OMQ"s. 
As noted in Section \ref{sec:rw}, it can be effectively decided, given an OMQ in  $(\T,q) \in (\horndl,\CQ)$, 
whether $\dShapley_{(\T,q)}$ is in $\FP$ or $\sP$-hard. We leave it as an open problem to determine the exact complexity of 
this latter task. 

From a practical perspective, the logical next step is to design and implement algorithms for computing 
"drastic Shapley values" in OMQA. In the database setting, "drastic Shapley value" computation has been implemented 
via a reduction to "probabilistic query evaluation", thereby enabling the reuse of existing provenance and knowledge compilation methods \cite{deutchComputingShapleyValue2022a}. 
As Theorem \ref{th:cnx-hom-closed} extends the correspondence between "drastic Shapley value" computation and 
"probabilistic query evaluation" to a large class of OMQs, a natural direction would be to develop algorithms for 
"drastic Shapley value" computation in OMQA by combining existing query rewriting procedures 
with "probabilistic query evaluation" techniques. 

The tight connections between "drastic Shapley value computation" and 
"probabilistic query evaluation" also enabled us to significantly generalize an existing dichotomy result for probabilistic OMQA (Theorem \ref{th:pqe-dichotomy}). 
Comparing Theorems \ref{th:cnx-hom-closed} and \ref{th:pqe-dichotomy}, one may observe that the dichotomy result we obtain 
for probabilistic OMQA is stronger than the one we obtain for $\dShapley$. 
This is to be expected as existing results on $\dShapley$ in the pure database setting do not cover "UCQ"s that are "disconnected@@q" or with "constants",
and it is an important open question whether there is a dichotomy for $\dShapley_q$ for all "Boolean" "UCQ"s~$q$. 
We remark however that any progress on this question can be immediately transferred to the OMQA setting. 
Let us further add that, in the context of "probabilistic databases", existing works have established  that $\PQEPhalfOne_{q} \polyeq \PQE_q$ for any $q$ that is either a 
"UCQ" or a ("constant-free") "hom-closed" query. Our \Cref{th:pqe-dichotomy} shows that this equivalence also holds for every OMQ 
$q\in(\horndl,\UCQ)$ (with "constants"), suggesting that it might hold for the full class of "$C$-hom-closed" queries.

Our focus in this paper was on the "drastic Shapley value", which is obtained by interpreting the initial entailment or querying task as a 0/1 function. 
Recently, however, another class of Shapley-based responsibility measures has been defined based upon weighted sums of minimal supports (WSMS) \cite{ourpods25}. 
Like the "drastic Shapley value", the WSMS measures satisfy the same desirable axioms, but they have been shown to enjoy more favourable computational properties. 
In particular, the WSMS measures can be tractably computed for all UCQs. A preliminary study of WSMS measures in the ontology setting established various tractability results in data and combined complexity for "FO-rewritable" OMQs \cite{ourKR25}. It would be worthwhile to pursue this investigation to identify further tractable cases and also compare experimentally the scores obtained by applying the "drastic Shapley value" and WSMS measures. 

Another interesting but challenging direction for future work is to
study $\dShapley$ (as well as the analogous task for the WSMS measures) for ontologies which are formulated using tuple-generating dependencies \cite{alice} (also referred to as existential rules \cite{mugnierIntroductionOntologyBasedQuery2014} or Datalog +/- \cite{DBLP:journals/ws/CaliGL12} in the context of ontologies).
We expect that the extension of our results to such ontologies will be non-trivial, due both to the presence of higher-arity predicates and the lack of forest-shaped models. 
Indeed, in \Cref{sec:dichotomy-OMQ}, the dichotomy for $\SPPQE$ that we exploited to obtain \Cref{th:cnx-hom-closed} is currently only known for arity 2 relations (\ie\ "graph databases"). 
The techniques in \Cref{sec:strengthening}, on the other hand, %
rely on the existence of a unique path between anonymous elements, which was ensured by %
the existence of forest-shaped canonical models. By contrast, the bounded treewidth canonical models obtained for many classes of existential rules do not satisfy this unique-path condition.

\bibliographystyle{alpha} 
\bibliography{long,biblio}

\newcommand{\etalchar}[1]{$^{#1}$}
\begin{thebibliography}{MCGH{\etalchar{+}}12}

\bibitem[ABKK22]{DBLP:conf/ruleml/AlrabbaaBKK22}
Christian Alrabbaa, Stefan Borgwardt, Patrick Koopmann, and Alisa Kovtunova.
\newblock Explaining ontology-mediated query answers using proofs over universal models.
\newblock In {\em Proceedings of the International Joint Conference on Rules and Reasoning (RuleML+RR)}, pages 167--182, 2022.

\bibitem[ACKZ09]{DBLP:journals/jair/ArtaleCKZ09}
Alessandro Artale, Diego Calvanese, Roman Kontchakov, and Michael Zakharyaschev.
\newblock The dl-lite family and relations.
\newblock {\em J. Artif. Intell. Res.}, 36:1--69, 2009.

\bibitem[ADF{\etalchar{+}}24]{DBLP:journals/pacmmod/AbramovichDF0O24}
Omer Abramovich, Daniel Deutch, Nave Frost, Ahmet Kara, and Dan Olteanu.
\newblock Banzhaf values for facts in query answering.
\newblock {\em Proceedings of the ACM on Management of Data}, 2(3):123, 2024.

\bibitem[AHV94]{alice}
Serge Abiteboul, Richard Hull, and Victor Vianu.
\newblock {\em Foundations of Databases}.
\newblock Addison Wesley, 1994.

\bibitem[Ama23]{Amarilli23}
Antoine Amarilli.
\newblock Uniform reliability for unbounded homomorphism-closed graph queries.
\newblock In {\em International Conference on Database Theory (ICDT)}, volume 255 of {\em Leibniz International Proceedings in Informatics (LIPIcs)}, pages 14:1--14:17. Leibniz-Zentrum f{\"u}r Informatik, 2023.

\bibitem[Bac02]{bacherDeterminantsMatricesRelated2002}
Roland Bacher.
\newblock Determinants of matrices related to the {{Pascal}} triangle.
\newblock {\em Journal de th{\'e}orie des nombres de Bordeaux}, 14(1):19--41, 2002.

\bibitem[BBG19]{DBLP:journals/jair/BienvenuBG19}
Meghyn Bienvenu, Camille Bourgaux, and Fran{\c{c}}ois Goasdou{\'{e}}.
\newblock Computing and explaining query answers over inconsistent {DL-Lite} knowledge bases.
\newblock {\em Journal of Artificial Intelligence Research (JAIR)}, 64:563--644, 2019.

\bibitem[BBL05]{DBLP:conf/ijcai/BaaderBL05}
Franz Baader, Sebastian Brandt, and Carsten Lutz.
\newblock Pushing the {EL} envelope.
\newblock In {\em International Joint Conference on Artificial Intelligence (IJCAI)}, pages 364--369, 2005.

\bibitem[BCR08]{DBLP:conf/otm/BorgidaCR08}
Alexander Borgida, Diego Calvanese, and Mariano Rodriguez{-}Muro.
\newblock Explanation in the {DL-Lite} family of description logics.
\newblock In {\em Proceedings of the International Conference: On the Move to Meaningful Internet Systems ({OTM})}, pages 1440--1457, 2008.

\bibitem[BFL24a]{BienvenuFL24}
Meghyn Bienvenu, Diego Figueira, and Pierre Lafourcade.
\newblock Shapley value computation in ontology-mediated query answering.
\newblock In {\em Principles of Knowledge Representation and Reasoning (KR)}, 2024.

\bibitem[BFL24b]{ourpods24}
Meghyn Bienvenu, Diego Figueira, and Pierre Lafourcade.
\newblock When is {S}hapley value computation a matter of counting?
\newblock In {\em ACM Symposium on Principles of Database Systems (PODS)}, 2024.

\bibitem[BFL25a]{ourpods25}
Meghyn Bienvenu, Diego Figueira, and Pierre Lafourcade.
\newblock Shapley revisited: Tractable responsibility measures for query answers.
\newblock {\em Proceedings of the {ACM} on Management of Data (PACMMOD)}, 3(2:PODS):1--26, 2025.

\bibitem[BFL25b]{ourpods25arxiv}
Meghyn Bienvenu, Diego Figueira, and Pierre Lafourcade.
\newblock Shapley revisited: Tractable responsibility measures for query answers.
\newblock {\em CoRR}, abs/2503.22358v2, 2025.
\newblock Long version of \cite{ourpods25}.

\bibitem[BFL25c]{ourKR25}
Meghyn Bienvenu, Diego Figueira, and Pierre Lafourcade.
\newblock Tractable responsibility measures for ontology-mediated query answering.
\newblock In {\em Principles of Knowledge Representation and Reasoning (KR)}, 2025.

\bibitem[BFL25d]{ourKR25arxiv}
Meghyn Bienvenu, Diego Figueira, and Pierre Lafourcade.
\newblock Tractable responsibility measures for ontology-mediated query answering.
\newblock {\em CoRR}, abs/2507.23191v1, 2025.
\newblock Long version of \cite{ourKR25}.

\bibitem[BHLS17]{DBLP:books/daglib/0041477}
Franz Baader, Ian Horrocks, Carsten Lutz, and Ulrike Sattler.
\newblock {\em An Introduction to Description Logic}.
\newblock Cambridge University Press, 2017.

\bibitem[BHLW16]{bienvenuFirstOrderrewritabilityContainment2016a}
Meghyn Bienvenu, Peter Hansen, Carsten Lutz, and Frank Wolter.
\newblock First order-rewritability and containment of conjunctive queries in {Horn} description logics.
\newblock In {\em International Joint Conference on Artificial Intelligence (IJCAI)}, pages 965--971. AAAI Press, 2016.

\bibitem[BHLW20]{bienvenuFirstOrderRewritabilityContainment2020}
Meghyn Bienvenu, Peter Hansen, Carsten Lutz, and Frank Wolter.
\newblock First order-rewritability and containment of conjunctive queries in {H}orn description logics ({L}ong version with appendix).
\newblock {\em CoRR}, abs/2011.09836v1, 2020.
\newblock Long version of \cite{bienvenuFirstOrderrewritabilityContainment2016a}.

\bibitem[BKLM23]{DBLP:journals/sigmod/BertossiKLM23}
Leopoldo~E. Bertossi, Benny Kimelfeld, Ester Livshits, and Mika{\"{e}}l Monet.
\newblock The {S}hapley value in database management.
\newblock {\em {SIGMOD} Record}, 52(2):6--17, 2023.

\bibitem[BO15]{bienvenuOntologyMediatedQueryAnswering2015}
Meghyn Bienvenu and Magdalena Ortiz.
\newblock Ontology-{{Mediated Query Answering}} with {{Data-Tractable Description Logics}}.
\newblock In {\em Tutorial Lectures of the Reasoning Web (RW) Summer School}, volume 9203 of {\em Lecture Notes in Computer Science}, pages 218--307. 2015.

\bibitem[CGL{\etalchar{+}}07]{DBLP:journals/jar/CalvaneseGLLR07}
Diego Calvanese, Giuseppe~De Giacomo, Domenico Lembo, Maurizio Lenzerini, and Riccardo Rosati.
\newblock Tractable reasoning and efficient query answering in description logics: The \emph{DL-Lite} family.
\newblock {\em J. Autom. Reason.}, 39(3):385--429, 2007.

\bibitem[CGL12]{DBLP:journals/ws/CaliGL12}
Andrea Cal{\`{\i}}, Georg Gottlob, and Thomas Lukasiewicz.
\newblock A general datalog-based framework for tractable query answering over ontologies.
\newblock {\em J. Web Semant.}, 14:57--83, 2012.

\bibitem[CGLP11]{DBLP:journals/sigmod/CaliGLP11}
Andrea Cal{\`{\i}}, Georg Gottlob, Thomas Lukasiewicz, and Andreas Pieris.
\newblock A logical toolbox for ontological reasoning.
\newblock {\em {SIGMOD} Record}, 40(3):5--14, 2011.

\bibitem[CLMV19]{DBLP:conf/ijcai/CeylanLMV19}
{\.I}smail~{\.I}lkan Ceylan, Thomas Lukasiewicz, Enrico Malizia, and Andrius Vaicenavicius.
\newblock Explanations for query answers under existential rules.
\newblock In {\em International Joint Conference on Artificial Intelligence (IJCAI)}, pages 1639--1646, 2019.

\bibitem[CLMV20]{DBLP:conf/ecai/CeylanLMV20}
{\.I}smail~{\.I}lkan Ceylan, Thomas Lukasiewicz, Enrico Malizia, and Andrius Vaicenavicius.
\newblock Explanations for ontology-mediated query answering in description logics.
\newblock In {\em Proceedings of the European Conference on Artificial Intelligence (ECAI)}, pages 672--679, 2020.

\bibitem[DFKM22]{deutchComputingShapleyValue2022a}
Daniel Deutch, Nave Frost, Benny Kimelfeld, and Mika{\"{e}}l Monet.
\newblock Computing the {S}hapley value of facts in query answering.
\newblock In {\em ACM SIGMOD International Conference on Management of Data (SIGMOD)}, pages 1570--1583, 2022.

\bibitem[DHS07]{dengMeasuringInconsistenciesOntologies2007}
Xi~Deng, Volker Haarslev, and Nematollaah Shiri.
\newblock Measuring {{Inconsistencies}} in {{Ontologies}}.
\newblock In {\em Proceedings of the European Semantic Web Conference}, pages 326--340, 2007.

\bibitem[DS04]{dalviEfficientQueryEvaluation2004}
Nilesh~N. Dalvi and Dan Suciu.
\newblock Efficient query evaluation on probabilistic databases.
\newblock In {\em International Conference on Very Large Data Bases ({VLDB})}, pages 864--875, 2004.

\bibitem[DS12]{dalviDichotomyProbabilisticInference2012}
Nilesh Dalvi and Dan Suciu.
\newblock The dichotomy of probabilistic inference for unions of conjunctive queries.
\newblock {\em Journal of the ACM}, 59(6):1--87, December 2012.

\bibitem[EOS{\etalchar{+}}12]{DBLP:conf/aaai/EiterOSTX12}
Thomas Eiter, Magdalena Ortiz, Mantas Simkus, Trung{-}Kien Tran, and Guohui Xiao.
\newblock Query rewriting for {H}orn-{$\mathcal{SHIQ}$} plus rules.
\newblock In {\em Proceedings of the {AAAI} Conference on Artificial Intelligence (AAAI)}, pages 726--733, 2012.

\bibitem[Esc03]{escoffierGuideCulinaireAidememoire1903}
Auguste Escoffier.
\newblock {\em {Le guide culinaire : aide-m{\'e}moire de cuisine pratique}}.
\newblock Biblioth{\`e}que Professionnelle, Paris, 1903.

\bibitem[FHW80]{fortuneDirectedSubgraphHomeomorphism1980}
Steven Fortune, John Hopcroft, and James Wyllie.
\newblock The directed subgraph homeomorphism problem.
\newblock {\em Theoretical Computer Science}, 10(2):111--121, 1980.

\bibitem[GH06]{grantMeasuringInconsistencyKnowledgebases2006}
John Grant and Anthony Hunter.
\newblock Measuring inconsistency in knowledge bases.
\newblock {\em Journal of Intelligent Information Systems}, 27(2):159--184, 2006.

\bibitem[HK10]{hunterMeasureConflictsShapley2010}
Anthony Hunter and S{\'e}bastien Konieczny.
\newblock On the measure of conflicts: {S}hapley inconsistency values.
\newblock {\em Artificial Intelligence (AIJ)}, 174(14):1007--1026, 2010.

\bibitem[HLW17]{DBLP:conf/aaai/HernichLW17}
Andr{\'{e}} Hernich, Julio Lemos, and Frank Wolter.
\newblock Query answering in dl-lite with datatypes: {A} non-uniform approach.
\newblock In Satinder Singh and Shaul Markovitch, editors, {\em Proceedings of the {AAAI} Conference on Artificial Intelligence (AAAI)}, pages 1142--1148, 2017.

\bibitem[Hoe63]{hoeffdingProbabilityInequalitiesSums1963}
Wassily Hoeffding.
\newblock Probability {{Inequalities}} for {{Sums}} of {{Bounded Random Variables}}.
\newblock {\em Journal of the American Statistical Association}, 58(301):13--30, March 1963.

\bibitem[JL12]{jungOntologyBasedAccessProbabilistic2012}
Jean~Christoph Jung and Carsten Lutz.
\newblock Ontology-based access to probabilistic data with {OWL} {QL}.
\newblock In {\em Proceedings of the International Semantic Web Conference ({ISWC})}, pages 182--197, 2012.

\bibitem[KK23]{khalilComplexityShapleyValue2023}
Majd Khalil and Benny Kimelfeld.
\newblock The complexity of the {S}hapley value for regular path queries.
\newblock In {\em International Conference on Database Theory (ICDT)}, pages 11:1--11:19, 2023.

\bibitem[KOS24]{karaShapleyValueModel2023}
Ahmet Kara, Dan Olteanu, and Dan Suciu.
\newblock From {S}hapley value to model counting and back.
\newblock In {\em ACM Symposium on Principles of Database Systems (PODS)}, 2024.

\bibitem[KS21]{kenigDichotomyGeneralizedModel2021}
Batya Kenig and Dan Suciu.
\newblock A dichotomy for the generalized model counting problem for unions of conjunctive queries.
\newblock In {\em ACM Symposium on Principles of Database Systems (PODS)}, pages 312--324, 2021.

\bibitem[LBKS21]{livshitsShapleyValueTuples2021}
Ester Livshits, Leopoldo Bertossi, Benny Kimelfeld, and Moshe Sebag.
\newblock The {S}hapley value of tuples in query answering.
\newblock {\em Logical Methods in Computer Science (LMCS)}, Volume 17, Issue 3:6942, 2021.

\bibitem[LK22]{livshitsShapleyValueInconsistency2022}
Ester Livshits and Benny Kimelfeld.
\newblock The {Shapley} value of inconsistency measures for functional dependencies.
\newblock {\em Logical Methods in Computer Science (LMCS)}, 18(2), 2022.

\bibitem[LL17]{lundbergUnifiedApproachInterpreting2017}
Scott~M Lundberg and Su-In Lee.
\newblock A unified approach to interpreting model predictions.
\newblock In {\em Proceedings of Advances in Neural Information Processing Systems ({NeurIPS})}, 2017.

\bibitem[LS22]{DBLP:journals/ai/LutzS22}
Carsten Lutz and Leif Sabellek.
\newblock A complete classification of the complexity and rewritability of ontology-mediated queries based on the description logic {EL}.
\newblock {\em Artif. Intell.}, 308:103709, 2022.

\bibitem[LSW19]{DBLP:journals/lmcs/LutzSW19}
Carsten Lutz, Inan{\c{c}} Seylan, and Frank Wolter.
\newblock The data complexity of ontology-mediated queries with closed predicates.
\newblock {\em Log. Methods Comput. Sci.}, 15(3), 2019.

\bibitem[MCGH{\etalchar{+}}12]{profiles}
Boris Motik, Bernardo Cuenca~Grau, Ian Horrocks, Zhe Wu, Achille Fokoue, and Carsten Lutz.
\newblock {\em {OWL} 2 {W}eb {O}ntology {L}anguage Profiles}.
\newblock {W3C} {R}ecommendation, 2012.
\newblock Available at \url{http://www.w3.org/TR/owl2-profiles/}.

\bibitem[MGMS10]{meliouComplexityCausalityResponsibility2010}
Alexandra Meliou, Wolfgang Gatterbauer, Katherine~F. Moore, and Dan Suciu.
\newblock The {{Complexity}} of {{Causality}} and {{Responsibility}} for {{Query Answers}} and non-{{Answers}}.
\newblock {\em Proc. VLDB Endow.}, 4(1):34--45, 2010.

\bibitem[MT14]{mugnierIntroductionOntologyBasedQuery2014}
Marie-Laure Mugnier and Micha{\"e}l Thomazo.
\newblock An introduction to ontology-based query answering with existential rules.
\newblock In {\em Tutorial Lectures of the International Reasoning Web ({RW}) Summer School}, volume 8714 of {\em Lecture Notes in Computer Science}, pages 245--278. 2014.

\bibitem[MW95]{DBLP:journals/siamcomp/MendelzonW95}
Alberto~O. Mendelzon and Peter~T. Wood.
\newblock Finding regular simple paths in graph databases.
\newblock {\em {SIAM} Journal on computing}, 24(6):1235--1258, 1995.

\bibitem[PB83]{provanComplexityCountingCuts1983}
J.~Scott Provan and Michael~O. Ball.
\newblock The {{Complexity}} of {{Counting Cuts}} and of {{Computing}} the {{Probability}} that a {{Graph}} is {{Connected}}.
\newblock {\em SIAM Journal on Computing}, 12(4):777--788, November 1983.

\bibitem[PLC{\etalchar{+}}08]{poggiLinkingDataOntologies2008}
Antonella Poggi, Domenico Lembo, Diego Calvanese, Giuseppe De~Giacomo, Maurizio Lenzerini, and Riccardo Rosati.
\newblock Linking data to ontologies.
\newblock In {\em Journal on {{Data Semantics X}}}, pages 133--173, 2008.

\bibitem[PS10]{DBLP:conf/kr/PenalozaS10}
Rafael Pe{\~{n}}aloza and Baris Sertkaya.
\newblock On the complexity of axiom pinpointing in the {$\mathcal{EL}$} family of description logics.
\newblock In Fangzhen Lin, Ulrike Sattler, and Miroslaw Truszczynski, editors, {\em Principles of Knowledge Representation and Reasoning (KR)}, 2010.

\bibitem[RKL20]{reshefImpactNegationComplexity2020}
Alon Reshef, Benny Kimelfeld, and Ester Livshits.
\newblock The impact of negation on the complexity of the {Shapley} value in conjunctive queries.
\newblock In {\em ACM Symposium on Principles of Database Systems (PODS)}, pages 285--297, 2020.

\bibitem[SBSdB16]{salimiQuantifyingCausalEffects2016}
Babak Salimi, Leopoldo Bertossi, Dan Suciu, and Guy~Van den Broeck.
\newblock Quantifying {{Causal Effects}} on {{Query Answering}} in {{Databases}}.
\newblock In {\em 8th {{USENIX Workshop}} on the {{Theory}} and {{Practice}} of {{Provenance}} ({{TaPP}})}. 2016.

\bibitem[Sha53]{shapley:book1952}
Lloyd~S Shapley.
\newblock A value for n-person games.
\newblock In {\em Contributions to the Theory of Games {{II}}}, pages 307--317. {Princeton University Press}, 1953.

\bibitem[Val79]{valiantComplexityEnumerationReliability1979}
Leslie~G. Valiant.
\newblock The complexity of enumeration and reliability problems.
\newblock {\em {SIAM} Journal on computing}, 8(3):410--421, 1979.

\bibitem[XCK{\etalchar{+}}18]{xiaoOntologyBasedDataAccess2018}
Guohui Xiao, Diego Calvanese, Roman Kontchakov, Domenico Lembo, Antonella Poggi, Riccardo Rosati, and Michael Zakharyaschev.
\newblock Ontology-based data access: {{A}} survey.
\newblock In {\em International Joint Conference on Artificial Intelligence (IJCAI)}, pages 5511--5519, 2018.

\end{thebibliography}

\end{document}